\documentclass{article}
\usepackage{microtype}
\usepackage{graphicx}
\usepackage{subcaption}
\usepackage{booktabs} 
\usepackage{hyperref}

\usepackage[accepted]{icml2026}

\usepackage{amsmath}
\usepackage{amssymb}
\usepackage{mathtools}
\usepackage{amsthm}

\usepackage[capitalize,noabbrev]{cleveref}
\usepackage[shortlabels]{enumitem}
\setlist{
  noitemsep,
  topsep=0pt,
  parsep=0pt,
  partopsep=0pt
}

\usepackage[dvipsnames]{xcolor}
\definecolor{verylightblue}{RGB}{235, 245, 255}

\usepackage{tikz}
\usetikzlibrary{fit, positioning, arrows.meta}
\usetikzlibrary{decorations.pathreplacing}
\usetikzlibrary{calc} 

\theoremstyle{plain}
\newtheorem{thm}{\protect\theoremname}
\theoremstyle{definition}
\newtheorem{problem}[thm]{\protect\problemname}
\theoremstyle{plain}
\newtheorem{lem}[thm]{\protect\lemmaname}
\theoremstyle{plain}
\newtheorem{cor}[thm]{\protect\corollaryname}
\theoremstyle{plain}
\newtheorem{prop}[thm]{\protect\propositionname}
\theoremstyle{remark}
\newtheorem{rem}[thm]{\protect\remarkname}
\newtheorem{defn}[thm]{\protect\definitionname}
\newtheorem{conjecture}[thm]{\protect\conjecturename}

\usepackage{babel}
\providecommand{\corollaryname}{Corollary}
\providecommand{\lemmaname}{Lemma}
\providecommand{\problemname}{Problem}
\providecommand{\propositionname}{Proposition}
\providecommand{\remarkname}{Remark}
\providecommand{\theoremname}{Theorem}
\providecommand{\definitionname}{Definition}
\providecommand{\conjecturename}{Conjecture}

\usepackage[textsize=tiny]{todonotes}

\icmltitlerunning{\tournamentsort: Principled Zero-shot Ranking Agents with Tournament Graphs}

\newcommand{\tournamentsort}{\textsc{BlitzRank}}
\newcommand{\algshortname}{\textsc{Blitz}}

\global\long\def\indeg{\deg^{-}}%
 
\global\long\def\outdeg{\deg^{+}}%
 
\global\long\def\inreach{R^{-}}%
 
\global\long\def\outreach{R^{+}}%
 
\global\long\def\path{\rightsquigarrow}%

\global\long\def\perm#1#2{{#1}^{\underline{#2}}}%

 \global\long\def\condense#1{\left[#1\right]}%

\begin{document}

\twocolumn[
  \icmltitle{\tournamentsort: Principled Zero-shot Ranking Agents with Tournament Graphs}



  \icmlsetsymbol{equal}{*}

  \begin{icmlauthorlist}
    \icmlauthor{Sheshansh Agrawal}{equal,ctxl}
    \icmlauthor{Thien Hang Nguyen}{equal,ctxl}
    \icmlauthor{Douwe Kiela}{ctxl}
  \end{icmlauthorlist}

  \icmlaffiliation{ctxl}{Contextual AI}

  \icmlcorrespondingauthor{Sheshansh Agrawal and Thien Hang Nguyen}{blitzrank.icml@gmail.com}

  \icmlkeywords{Reranking, Tournament Graphs, k-wise sorting, k-sort, Zero-shot reranking, Information retrieval}

  \vskip 0.3in
]


\printAffiliationsAndNotice{\icmlEqualContribution}

\begin{abstract}
Selecting the top $m$ from $n$ items via expensive $k$-wise comparisons is central to settings ranging from LLM-based document reranking to crowdsourced evaluation and tournament design.
Existing methods either rely on heuristics that discard comparison information, or exploit it at prohibitive cost.
We introduce a \textit{tournament graph} framework that provides a principled foundation for $k$-wise ranking.
Our key observation is that each $k$-item comparison reveals an induced tournament of $\binom{k}{2}$ pairwise preferences; aggregating these into a global preference graph and computing its transitive closure yields many additional orderings without further oracle calls.
We formalize when the current top-$m$ output is \textit{certifiably determined} and design a greedy query schedule that maximizes information gain towards identifying the top-$m$ items.
The framework also gracefully handles non-transitive preferences -- cycles induced by real-world oracles -- by collapsing them into equivalence classes that yield principled \textit{tiered rankings}.
Applied to LLM reranking across 14 benchmarks and 5 models, \tournamentsort{} achieves Pareto dominance over existing approaches: matching or exceeding accuracy while requiring 25--40\% fewer tokens than comparable methods; against pairwise reranking, it achieves near-identical quality with 7$\times$ fewer tokens. Code available at
\url{https://github.com/ContextualAI/BlitzRank}.

\end{abstract}

\section{Introduction}
\label{sec:intro}

Consider the classic \textit{25 horses' race} puzzle: \textit{given 25 horses where 5 can race at a time, what is the minimum number of races needed to identify the 3 fastest?}

The answer -- 7 races -- is well known~\citep{zhou2008practical}, but the reasoning reveals a deeper principle.
A naive tournament bracket uses each race only to identify a winner.
The optimal strategy extracts \emph{all} the information each race provides: a race among 5 horses reveals a complete ordering -- 10 pairwise preferences -- and accumulating these transitively certifies the top 3 with far fewer races than a bracket requires.

This paper develops a principled framework for this class of problems: \emph{top-$m$ selection from $n$ items via $k$-wise comparison queries}.
Each query to an oracle reveals the induced tournament on up to $k$ items -- $\binom{k}{2}$ pairwise preferences -- and the goal is to identify the top-$m$ items using as few queries as possible.
The problem arises naturally whenever comparisons are expensive: LLM-based document reranking, crowdsourced preference judgments, human-in-the-loop evaluation, and tournament design.
In all these settings, each query incurs cost in time, money, or human attention, making query efficiency paramount.

Existing methods leave information on the table.
Classical sorting algorithms adapted for $k$-wise comparisons -- heapsort~\citep{zhuang2024setwise,qin2024large}, tournament brackets~\citep{chen2025tourrank}, sliding windows~\citep{sun2023chatgpt} -- typically focus on selection or identifying a winner at each stage, underutilizing the $\binom{k}{2}$ pairwise relationships revealed.
To our knowledge, no existing approach provides a  framework for accumulating comparison outcomes and a criterion to certify the top-$m$.

\begin{figure}[t]
    \centering
    \begin{tikzpicture}[
        vertex/.style={circle, draw, minimum size=0.45cm, inner sep=0pt, font=\footnotesize},
        selected/.style={vertex, fill=blue!12, draw=blue!60!black},
        edge/.style={->, >=stealth, semithick},
        newedge/.style={edge, blue!70!black},
        transedge/.style={edge, orange!80!black, densely dashed},
        oldedge/.style={edge, black!40},
        queryset/.style={draw=blue!60!black, dashed, rounded corners=6pt, thick},
    ]
    
    \begin{scope}[local bounding box=left]
        \node[selected] (1) at (0, 0.7) {1};
        \node[selected] (2) at (-0.4, 0) {2};
        \node[selected] (3) at (0.3, 0) {3};
        \node[vertex] (4) at (1.6, 0.7) {4};
        \node[vertex] (5) at (1.0, 0.7) {5};
        \node[vertex] (6) at (1.3, 0) {6};
        
        \draw[oldedge] (3) -- (5);
        \draw[oldedge] (5) -- (6);
        \draw[oldedge] (4) -- (6);
        \draw[transedge] (3) -- (6);

        \node[queryset, fit=(1)(2)(3), inner sep=3.5pt] (qbox) {};
    \end{scope}
    
    \draw[->, >=stealth, thick, black!70] (2.4, 0.28) -- (3.25, 0.28);
    \node[font=\scriptsize, above] at (2.85, 0.32) {$\mathcal{O}(S)$};
    
    \begin{scope}[xshift=4.4cm, local bounding box=right]
        \node[vertex] (1r) at (0, 0.7) {1};
        \node[vertex] (2r) at (-0.4, 0) {2};
        \node[vertex] (3r) at (0.3, 0) {3};
        \node[vertex] (4r) at (1.6, 0.7) {4};
        \node[vertex] (5r) at (1.0, 0.7) {5};
        \node[vertex] (6r) at (1.3, 0) {6};
        
        \draw[oldedge] (3r) -- (5r);
        \draw[oldedge] (5r) -- (6r);
        \draw[oldedge] (4r) -- (6r);

        \draw[newedge] (1r) -- (2r);
        \draw[newedge] (1r) -- (3r);
        \draw[newedge] (3r) -- (2r);
        \draw[transedge] (3r) -- (6r);
        \draw[transedge] (1r) -- (5r);
        \draw[transedge] (1r) -- (6r);
    \end{scope}
    
    \node[font=\scriptsize, anchor=north] at ([yshift=-0.17cm]left.south) {Query set $S$ selected};
    \node[font=\scriptsize, anchor=north] at ([yshift=-0.3cm]right.south) {Oracle reveals $3$ new edges};
    
    \end{tikzpicture}
    \vspace{-0.5em}
    \caption{A $k$-wise oracle query on $n{=}6$ candidates. \textbf{Left:} A query set $S$ of $k{=}3$ candidates (shaded) is selected. \textbf{Right:} The oracle returns a tournament on $S$, revealing $\binom{3}{2}{=}3$ new edges (\textcolor{blue!70!black}{blue}). Combined with prior edges (\textcolor{black!40}{gray}), additional preferences are inferred transitively (\textcolor{orange!80!black}{orange dashed}).}
    
    \label{fig:oracle-query}
\end{figure}
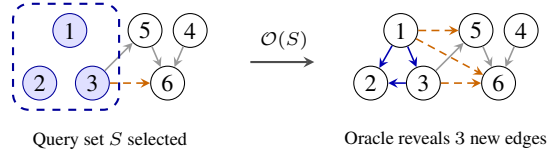
\definecolor{leftcolor}{named}{red}
\definecolor{rightcolor}{named}{ForestGreen}

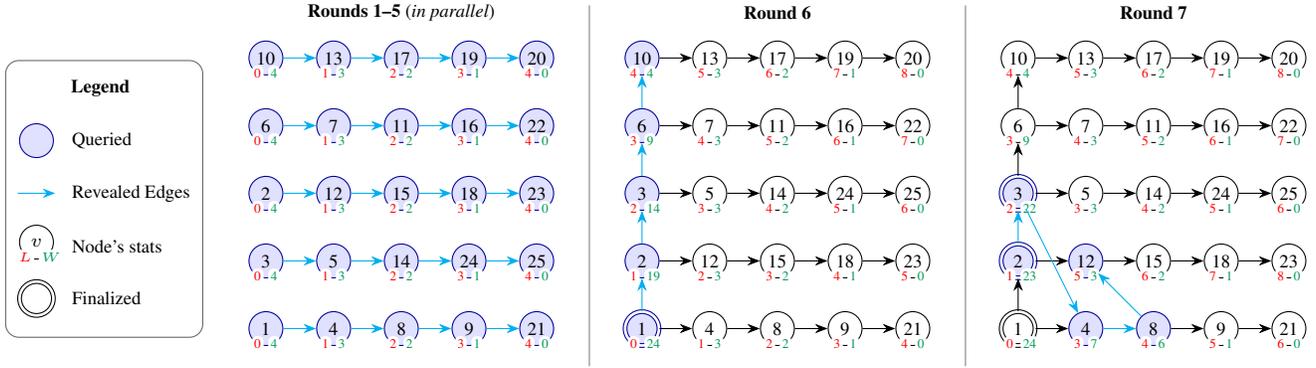
\begin{figure*}[t]
\centering
\begin{tikzpicture}[
    node distance=0.9cm,
    every node/.style={font=\scriptsize},
    graph node/.style={circle, draw, minimum size=0.45cm, inner sep=1pt},
    blue node/.style={fill=blue!12, draw=blue!60!black},
    double node/.style={double, double distance=1pt},
    arrow/.style={-{Stealth[scale=1]}},
    blue arrow/.style={-{Stealth[scale=1]}, cyan},
    title/.style={font=\scriptsize},
    score/.style={font=\fontsize{5}{6}\selectfont, inner sep=1pt, fill=white, rounded corners=2pt},
    scoreL/.style={score, text=leftcolor},
    scoreW/.style={score, text=rightcolor},
]

\newcommand{\horsenode}[6]{
    \node[graph node, #2] (#1) at (#3) {#4};
    \node[scoreL, below left=-2pt and -4pt of #1] {#5};
    \node[scoreW, below right=-2pt and -4pt of #1] {#6};
}

\begin{scope}[local bounding box=legend, xshift=-0.5cm, yshift=-1cm]
    \node[title] at (0.5, 4.2) {\textbf{Legend}};
    
    \node[graph node, blue node] (leg3) at (-0.35, 3.5) {};
    \node[anchor=west, font=\scriptsize] at (0, 3.5) {Queried};
    
    \draw[blue arrow] (-0.6, 2.8) -- (-0.1, 2.8);
    \node[anchor=west, font=\scriptsize] at (0, 2.8) {Revealed Edges};

    \node[graph node] (leg1) at (-0.35, 2.15) {$v$};
    \node[scoreL, below left=-2pt and -4pt of leg1] {$L$};
    \node[scoreW, below right=-2pt and -4pt of leg1] {$W$};
    \node[anchor=west, font=\scriptsize] at (0, 2.1) {Node's stats};

    \node[graph node, double node] (leg2) at (-0.35, 1.4) {};
    \node[anchor=west, font=\scriptsize] at (0, 1.4) {Finalized};

    \draw[rounded corners=5pt, draw=black!60] 
    ([xshift=-7pt, yshift=-10pt]leg2.south west) 
    rectangle 
    ([xshift=38pt, yshift=4pt]legend.north);

\end{scope}

\begin{scope}[xshift=2.2cm]
    \node[title] at (1.8, 4.2) {\textbf{Rounds 1--5} (\textit{in parallel})};
    
    \horsenode{A1}{blue node}{0, 0}{1}{0}{4}
    \horsenode{A4}{blue node}{0.9, 0}{4}{1}{3}
    \horsenode{A8}{blue node}{1.8, 0}{8}{2}{2}
    \horsenode{A9}{blue node}{2.7, 0}{9}{3}{1}
    \horsenode{A21}{blue node}{3.6, 0}{21}{4}{0}
    
    \horsenode{A3}{blue node}{0, 0.9}{3}{0}{4}
    \horsenode{A5}{blue node}{0.9, 0.9}{5}{1}{3}
    \horsenode{A14}{blue node}{1.8, 0.9}{14}{2}{2}
    \horsenode{A24}{blue node}{2.7, 0.9}{24}{3}{1}
    \horsenode{A25}{blue node}{3.6, 0.9}{25}{4}{0}
    
    \horsenode{A2}{blue node}{0, 1.8}{2}{0}{4}
    \horsenode{A12}{blue node}{0.9, 1.8}{12}{1}{3}
    \horsenode{A15}{blue node}{1.8, 1.8}{15}{2}{2}
    \horsenode{A18}{blue node}{2.7, 1.8}{18}{3}{1}
    \horsenode{A23}{blue node}{3.6, 1.8}{23}{4}{0}
    
    \horsenode{A6}{blue node}{0, 2.7}{6}{0}{4}
    \horsenode{A7}{blue node}{0.9, 2.7}{7}{1}{3}
    \horsenode{A11}{blue node}{1.8, 2.7}{11}{2}{2}
    \horsenode{A16}{blue node}{2.7, 2.7}{16}{3}{1}
    \horsenode{A22}{blue node}{3.6, 2.7}{22}{4}{0}
    
    \horsenode{A10}{blue node}{0, 3.6}{10}{0}{4}
    \horsenode{A13}{blue node}{0.9, 3.6}{13}{1}{3}
    \horsenode{A17}{blue node}{1.8, 3.6}{17}{2}{2}
    \horsenode{A19}{blue node}{2.7, 3.6}{19}{3}{1}
    \horsenode{A20}{blue node}{3.6, 3.6}{20}{4}{0}
    
    \draw[blue arrow] (A10) -- (A13); \draw[blue arrow] (A13) -- (A17); \draw[blue arrow] (A17) -- (A19); \draw[blue arrow] (A19) -- (A20);
    \draw[blue arrow] (A6) -- (A7); \draw[blue arrow] (A7) -- (A11); \draw[blue arrow] (A11) -- (A16); \draw[blue arrow] (A16) -- (A22);
    \draw[blue arrow] (A2) -- (A12); \draw[blue arrow] (A12) -- (A15); \draw[blue arrow] (A15) -- (A18); \draw[blue arrow] (A18) -- (A23);
    \draw[blue arrow] (A3) -- (A5); \draw[blue arrow] (A5) -- (A14); \draw[blue arrow] (A14) -- (A24); \draw[blue arrow] (A24) -- (A25);
    \draw[blue arrow] (A1) -- (A4); \draw[blue arrow] (A4) -- (A8); \draw[blue arrow] (A8) -- (A9); \draw[blue arrow] (A9) -- (A21);
\end{scope}

\begin{scope}[xshift=7.2cm]
    \node[title] at (1.8, 4.2) {\textbf{Round 6}};
    
    \horsenode{B1}{blue node, double node}{0, 0}{1}{0}{24}
    \horsenode{B4}{}{0.9, 0}{4}{1}{3}
    \horsenode{B8}{}{1.8, 0}{8}{2}{2}
    \horsenode{B9}{}{2.7, 0}{9}{3}{1}
    \horsenode{B21}{}{3.6, 0}{21}{4}{0}
    
    \horsenode{B2}{blue node}{0, 0.9}{2}{1}{19}
    \horsenode{B12}{}{0.9, 0.9}{12}{2}{3}
    \horsenode{B15}{}{1.8, 0.9}{15}{3}{2}
    \horsenode{B18}{}{2.7, 0.9}{18}{4}{1}
    \horsenode{B23}{}{3.6, 0.9}{23}{5}{0}
    
    \horsenode{B3}{blue node}{0, 1.8}{3}{2}{14}
    \horsenode{B5}{}{0.9, 1.8}{5}{3}{3}
    \horsenode{B14}{}{1.8, 1.8}{14}{4}{2}
    \horsenode{B24}{}{2.7, 1.8}{24}{5}{1}
    \horsenode{B25}{}{3.6, 1.8}{25}{6}{0}
    
    \horsenode{B6}{blue node}{0, 2.7}{6}{3}{9}
    \horsenode{B7}{}{0.9, 2.7}{7}{4}{3}
    \horsenode{B11}{}{1.8, 2.7}{11}{5}{2}
    \horsenode{B16}{}{2.7, 2.7}{16}{6}{1}
    \horsenode{B22}{}{3.6, 2.7}{22}{7}{0}
    
    \horsenode{B10}{blue node}{0, 3.6}{10}{4}{4}
    \horsenode{B13}{}{0.9, 3.6}{13}{5}{3}
    \horsenode{B17}{}{1.8, 3.6}{17}{6}{2}
    \horsenode{B19}{}{2.7, 3.6}{19}{7}{1}
    \horsenode{B20}{}{3.6, 3.6}{20}{8}{0}
    
    \draw[arrow] (B10) -- (B13); \draw[arrow] (B13) -- (B17); \draw[arrow] (B17) -- (B19); \draw[arrow] (B19) -- (B20);
    \draw[arrow] (B6) -- (B7); \draw[arrow] (B7) -- (B11); \draw[arrow] (B11) -- (B16); \draw[arrow] (B16) -- (B22);
    \draw[arrow] (B3) -- (B5); \draw[arrow] (B5) -- (B14); \draw[arrow] (B14) -- (B24); \draw[arrow] (B24) -- (B25);
    \draw[arrow] (B2) -- (B12); \draw[arrow] (B12) -- (B15); \draw[arrow] (B15) -- (B18); \draw[arrow] (B18) -- (B23);
    \draw[arrow] (B1) -- (B4); \draw[arrow] (B4) -- (B8); \draw[arrow] (B8) -- (B9); \draw[arrow] (B9) -- (B21);
    \draw[blue arrow] (B1) -- (B2); \draw[blue arrow] (B2) -- (B3); \draw[blue arrow] (B3) -- (B6); \draw[blue arrow] (B6) -- (B10);
\end{scope}

\begin{scope}[xshift=12.2cm]
    \node[title] at (1.8, 4.2) {\textbf{Round 7}};
    
    \horsenode{C1}{double node}{0, 0}{1}{0}{24}
    \horsenode{C4}{blue node}{0.9, 0}{4}{3}{7}
    \horsenode{C8}{blue node}{1.8, 0}{8}{4}{6}
    \horsenode{C9}{}{2.7, 0}{9}{5}{1}
    \horsenode{C21}{}{3.6, 0}{21}{6}{0}
    
    \horsenode{C2}{blue node, double node}{0, 0.9}{2}{1}{23}
    \horsenode{C12}{blue node}{0.9, 0.9}{12}{5}{3}
    \horsenode{C15}{}{1.8, 0.9}{15}{6}{2}
    \horsenode{C18}{}{2.7, 0.9}{18}{7}{1}
    \horsenode{C23}{}{3.6, 0.9}{23}{8}{0}
    
    \horsenode{C3}{blue node, double node}{0, 1.8}{3}{2}{22}
    \horsenode{C5}{}{0.9, 1.8}{5}{3}{3}
    \horsenode{C14}{}{1.8, 1.8}{14}{4}{2}
    \horsenode{C24}{}{2.7, 1.8}{24}{5}{1}
    \horsenode{C25}{}{3.6, 1.8}{25}{6}{0}
    
    \horsenode{C6}{}{0, 2.7}{6}{3}{9}
    \horsenode{C7}{}{0.9, 2.7}{7}{4}{3}
    \horsenode{C11}{}{1.8, 2.7}{11}{5}{2}
    \horsenode{C16}{}{2.7, 2.7}{16}{6}{1}
    \horsenode{C22}{}{3.6, 2.7}{22}{7}{0}
    
    \horsenode{C10}{}{0, 3.6}{10}{4}{4}
    \horsenode{C13}{}{0.9, 3.6}{13}{5}{3}
    \horsenode{C17}{}{1.8, 3.6}{17}{6}{2}
    \horsenode{C19}{}{2.7, 3.6}{19}{7}{1}
    \horsenode{C20}{}{3.6, 3.6}{20}{8}{0}
    
    \draw[arrow] (C10) -- (C13); \draw[arrow] (C13) -- (C17); \draw[arrow] (C17) -- (C19); \draw[arrow] (C19) -- (C20);
    \draw[arrow] (C6) -- (C7); \draw[arrow] (C7) -- (C11); \draw[arrow] (C11) -- (C16); \draw[arrow] (C16) -- (C22);
    \draw[arrow] (C3) -- (C5); \draw[arrow] (C5) -- (C14); \draw[arrow] (C14) -- (C24); \draw[arrow] (C24) -- (C25);
    \draw[arrow] (C2) -- (C12); \draw[arrow] (C12) -- (C15); \draw[arrow] (C15) -- (C18); \draw[arrow] (C18) -- (C23);
    \draw[arrow] (C1) -- (C4); \draw[blue arrow] (C4) -- (C8); \draw[arrow] (C8) -- (C9); \draw[arrow] (C9) -- (C21);
    \draw[arrow] (C1) -- (C2); \draw[blue arrow] (C2) -- (C3); \draw[arrow] (C3) -- (C6); \draw[arrow] (C6) -- (C10);
   \draw[blue arrow] (C8) -- (C12); \draw[blue arrow] (C3) -- (C4);
\end{scope}
\draw[black!50] (6.5, -0.5) -- (6.5, 4.3);
\draw[black!50] (11.5, -0.5) -- (11.5, 4.3);

\end{tikzpicture}
\caption{Illustration of \tournamentsort{} (Algorithm \ref{alg:tournament-sort-main}) achieving the optimal 7 queries on the classic \textit{25 horses puzzle}, where $(n,k,m)=(25,5,3)$. Each node shows the horse ID with \textcolor{leftcolor}{$L(u):=|\inreach_G(u)|$} (cf. \eqref{eqdef:inreach}) bottom left and \textcolor{rightcolor}{$W(u):=|\outreach_G(u)|$} (cf. \eqref{eqdef:outreachdef}) at bottom right. Blue nodes indicate horses queried in that round. In this transitive instance, $\kappa_G(u)=L(u)+W(u)$, (cf. \eqref{eqdef:known-relationship}) and double circles indicate resolved horses (where $\kappa_G(u)=24$). 
\\
{\scriptsize Note: For reproducibility, the initial ordering was generated with \texttt{Python}'s \texttt{random.shuffle} on $[1,2,\dots,25]$ with \texttt{seed=42}. The first five rounds are grouped as shown.}}
\label{fig:horses}
\end{figure*}

We propose a \emph{tournament graph} framework and an algorithm \tournamentsort{}\footnote{The name \tournamentsort{}  evokes blitz chess, where rapid tournament play efficiently determines  rankings.} that addresses these limitations.
By aggregating these local tournaments into a global preference graph, the transitive closure yields additional orderings without further queries.
We formalize when a node is \emph{resolved} -- its relationship to all $n-1$ others determined -- and design an algorithm that terminates once the current top-$m$ are certified.
Applied to the 25 horses puzzle, the algorithm discovers the optimal 7-race strategy without being given any problem-specific knowledge (Figure~\ref{fig:horses}).

The framework also addresses a phenomenon that prior work largely overlooks: \emph{non-transitive preferences}.
Real-world oracles -- LLMs, crowdworkers, domain experts -- sometimes produce cyclic judgments: $A \succ B \succ C \succ A$.
Rather than treating cycles as noise, we treat them as structure indicating that the oracle cannot consistently separate those items.
Our framework captures this via \emph{strongly connected components} (SCCs) of the preference graph, which collapse into equivalence classes to yield a DAG of ``relevance tiers.''
The \emph{same algorithm} handles both settings -- returning a total ordering when preferences are consistent, and a principled tiered ranking otherwise.

\textbf{Contributions.}
\begin{enumerate}[leftmargin=*,itemsep=2pt]
    \item \textbf{A unifying theoretical framework.}
    We formalize top-$m$ selection via $k$-wise comparison oracles using tournament graphs (\S\ref{sec:tournament}).
    The framework captures how transitive closure amplifies each query's information yield, defines when items are \emph{resolved}, and unifies transitive and non-transitive preferences.

    \item \textbf{A provably correct algorithm.}
    We present \tournamentsort{}, a greedy algorithm that schedules queries among minimally-resolved SCCs to guarantee progress (\S\ref{sec:algorithm}).
    We prove correctness and termination in both transitive and non-transitive settings.

    \item \textbf{Empirical validation.}
    Across 14 benchmarks and 5 LLM oracles, our approach achieves Pareto dominance: matching or exceeding baseline accuracy while requiring 25--40\% fewer tokens than methods with comparable structure, and 7$\times$ fewer than pairwise reranking at near-identical quality (\S\ref{sec:experiments}).
    Convergence is predictable, and our analysis confirms that cycles capture genuinely ambiguous documents rather than noise.
\end{enumerate}

\paragraph{Overview.}
Section~\ref{sec:tournament} presents the tournament graph framework, Section~\ref{sec:related} discusses related work, and Section~\ref{sec:experiments} evaluates \tournamentsort{} empirically on document reranking benchmarks.
Full proofs, theoretical analysis, and additional experiments appear in the appendices.

\section{Tournament Graph Framework}\label{sec:tournament}

\subsection{Problem Setup}\label{sec:setup}

Let $V$ be a set of $n$ items and $G^* = (V, E^*)$ be a fixed but unknown 
tournament, representing the oracle's latent preferences: for every
pair $u \neq v$, exactly one of $(u,v)$ or $(v,u)$ belongs to $E^*$.
We access $G^*$ through a $k$-wise comparison oracle $O_{G^*}$:
given any queried set $S \subseteq V$ with $|S|\le k$, returns all
directed edges of $G^*$ induced by $S$.

Our ranking target is reachability-based. For a directed graph $G$,
write $u \leadsto_G v$ if there is a directed path from $u$ to $v$,
and define the \emph{in-reach} of $v$ by
\begin{align}
    \inreach_G(v) := \{u \in V \setminus \{v\} : u \leadsto_G v\}. \label{eqdef:inreach}
\end{align}

The task is to identify $m$ vertices with smallest in-reach
using as few oracle queries as possible. When $G^*$ is transitive,
in-reach coincides with in-degree, so this reduces to the usual loss-count
ordering; in general, the reachability view extends naturally to cyclic
preferences. Formal definitions and discussion appear in Appendix~\ref{sec:Problem-Statement}.

\subsection{Revealed Graph and Transitive Inference}\label{sec:revealed}
We maintain a \emph{revealed graph} $G = (V, E)$ that accumulates edges
returned by the oracle. Initially $E = \emptyset$; after querying
a set $S$, we add the revealed edges $O_{G^*}(S)$ to $E$. By construction,
$G$ is always a subgraph of $G^*$.

The key point is that $G$ contains more information than its explicit
edges. If $u \leadsto_G v$, then the transitive preference $u \succ v$ is already
certified in the unknown tournament, even if the edge $(u,v)$ was never
queried directly. In this sense, each query is amplified by reachability:
a newly revealed edge can combine with existing paths to certify many
additional pairwise relations.

For a vertex $v$, we also consider its out-reach
\begin{align}
    \outreach_G(v) := \{u \in V \setminus \{v\} : v \leadsto_G u\}. \label{eqdef:outreachdef}
\end{align}

The vertices whose relation to $v$ is already determined form the known-relationship
set
\begin{align}
   K_G(v) := \inreach_G(v) \cup \outreach_G(v),
\qquad
\kappa_G(v) := |K_G(v)|. \label{eqdef:known-relationship}
\end{align}

When $\kappa_G(v)=n-1$, the comparison between $v$ and every other
vertex is already implied by the revealed graph; we call such a vertex
\emph{resolved}. At that point, its discovered in-reach is already its
true in-reach in $G^*$, so further queries cannot change its rank. 
We use this as a practical stopping rule for algorithm design: we terminate once the current
top-$m$ vertices by discovered in-reach are resolved.
The full finalization framework is presented in Appendix~\ref{sec:graph-prelims}.

\subsection{Non-Transitive Preferences}\label{sec:nontransitive}
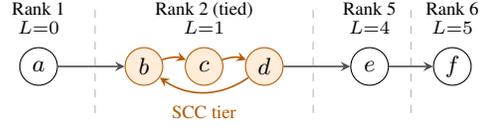
\begin{figure}[t]
    \centering
    \begin{tikzpicture}[
        vertex/.style={circle, draw, minimum size=0.5cm, inner sep=0pt, font=\footnotesize},
        sccvertex/.style={vertex, fill=orange!15, draw=orange!70!black},
        edge/.style={->, >=stealth, semithick, black!70},
        sccedge/.style={edge, orange!70!black, bend left=20},
        tierline/.style={dashed, black!30},
        label/.style={font=\scriptsize, align=center},
    ]
    
    \node[font=\footnotesize\scshape, anchor=west] at (-0.8, 1.6) {Transitive};
    
    \begin{scope}[local bounding box=transitive]
        \foreach \x in {0.52, 1.65, 2.75, 3.85, 4.95} {
            \draw[tierline] (\x, 1.1) -- (\x, -0.3);
        }
        
        \foreach \i/\x/\r/\l in {1/0/1/0, 2/1.1/2/1, 3/2.2/3/2, 4/3.3/4/3, 5/4.4/5/4, 6/5.5/6/5} {
            \node[label] at (\x, 0.9) {Rank \r\\[-1pt] $L{=}\l$};
        }
        
        \node[vertex] (a) at (0, 0.35) {$a$};
        \node[vertex] (b) at (1.1, 0.35) {$b$};
        \node[vertex] (c) at (2.2, 0.35) {$c$};
        \node[vertex] (d) at (3.3, 0.35) {$d$};
        \node[vertex] (e) at (4.4, 0.35) {$e$};
        \node[vertex] (f) at (5.5, 0.35) {$f$};
        
        \draw[edge] (a) -- (b);
        \draw[edge] (b) -- (c);
        \draw[edge] (c) -- (d);
        \draw[edge] (d) -- (e);
        \draw[edge] (e) -- (f);
    \end{scope}
    
    \node[font=\footnotesize\scshape, anchor=west] at (-0.8, -0.5) {Non-transitive};
    
    \begin{scope}[yshift=-2.1cm, local bounding box=nontransitive]
        \foreach \x in {0.75, 3.65, 4.95} {
            \draw[tierline] (\x, 1.1) -- (\x, -0.3);
        }
        
        \node[label] at (0, 1.0) {Rank 1\\[-1pt] $L{=}0$};
        \node[label] at (2.2, 1.0) {Rank 2 (tied)\\[-1pt] $L{=}1$};
        \node[label] at (4.4, 1.0) {Rank 5\\[-1pt] $L{=}4$};
        \node[label] at (5.5, 1.0) {Rank 6\\[-1pt] $L{=}5$};
        
        \node[vertex] (a2) at (0, 0.35) {$a$};
        \node[sccvertex] (b2) at (1.4, 0.35) {$b$};
        \node[sccvertex] (c2) at (2.2, 0.35) {$c$};
        \node[sccvertex] (d2) at (3.0, 0.35) {$d$};
        \node[vertex] (e2) at (4.4, 0.35) {$e$};
        \node[vertex] (f2) at (5.5, 0.35) {$f$};
        
        \draw[edge] (a2) -- (b2);
        \draw[edge] (d2) -- (e2);
        \draw[edge] (e2) -- (f2);
        
        \draw[sccedge] (b2) to (c2);
        \draw[sccedge] (c2) to (d2);
        \draw[sccedge, bend left=35] (d2) to (b2);
        
        \node[font=\scriptsize, orange!70!black] at (2.2, -0.25) {SCC tier};
    \end{scope}
    
    \end{tikzpicture}
    \caption{Tournament graphs with $n{=}6$ candidates. \textbf{Top (transitive):} Consistent preferences yield a total ordering; each node has a unique rank determined by $L(u):=|\inreach_G(u)|$, the number of vertices that reach it. \textbf{Bottom (non-transitive):} A cycle $b \succ c \succ d \succ b$ forms a strongly connected component (orange). Nodes in the SCC share the same tier since no consistent ordering exists among them, but the partial order $a \succ \{b,c,d\} \succ e \succ f$ is still recovered.}
    \label{fig:transitive-comparison}
\end{figure}

When the oracle produces cycles, a global total order may not exist. 
Our framework captures these intransitivities via \emph{strongly connected components} (SCCs). 
Vertices in the same SCC reach one another
in both directions, so they behave as a tied tier under reachability-based
ranking. Collapsing each SCC to a supernode gives the \emph{condensation}
$\condense G$, which is always a DAG; for tournaments, the condensation
is in fact a transitive tournament. This induces a total order on SCC
tiers even when the original graph contains cycles.

Accordingly, our output in the non-transitive case is a \emph{tiered
ranking}: SCCs are ordered from best to worst, while vertices inside
the same SCC are tied (see Figure~\ref{fig:transitive-comparison}). When all SCCs are singletons, this reduces to
the transitive case and yields a total ordering. For top-$m$ selection,
if the boundary tier is larger than the remaining quota, any subset
of the required size from that tier is a valid output.\footnote{If
 a total order is required, ties inside a tier
can be broken by a secondary signal such as the original retrieval score.}

At a high level, the algorithm operates on the discovered SCC structure
while maintaining the same vertex-level resolution criterion from the
previous subsection. Part~\ref{part:theory} of the appendix develops the full theory, including
why discovered SCCs refine the true ones and why reasoning on the condensation
is enough to certify correctness.

\subsection{Algorithm: \tournamentsort{}}\label{sec:algorithm}

\begin{algorithm}[tb]
\caption{\tournamentsort}
\label{alg:tournament-sort-main}
\begin{algorithmic}[1]
\STATE \textbf{Input:} vertex set $V$ with $|V|=n$, oracle $O_{G^*}$, maximum query size $k$, target count $m$
\STATE \textbf{Output:} a valid top-$m$ output for the underlying tournament $G^*$
\STATE Initialize $E \leftarrow \emptyset$, $G \leftarrow (V, E)$
\LOOP
    \STATE \textit{// Compute vertex metrics}
    \FOR{each $u \in V$}
        \STATE Compute $|\inreach_G(u)|$ and $\kappa_G(u)$
    \ENDFOR
\STATE \textit{// Identify resolved vertices and current top-$m$ candidates}
    \STATE $F \leftarrow \{u \in V : \kappa_G(u) = n - 1\}$ \COMMENT{resolved vertices}
    \STATE $T \leftarrow$ the $m$ vertices with smallest $|\inreach_G(\cdot)|$
\STATE \textit{// Termination condition: current top candidates are resolved}
    \IF{$T \subseteq F$}
        \STATE \textbf{return} $T$ sorted by ascending $|\inreach_G(\cdot)|$
    \ENDIF
    \STATE \textit{// Greedy Schedule}
    \STATE Compute condensation $\condense{G}$ of the revealed graph $G$
    \STATE $\mathcal{C} \leftarrow$ SCCs containing $\geq 1$ vertex with $\kappa_G(u) < n{-}1$, ordered by ascending in-reach in $\condense{G}$
    \STATE $Q \leftarrow \{\mathrm{rep}(C) : C \in \mathcal{C}[1:k']\}$ where $k' = \min(k, |\mathcal{C}|)$ \COMMENT{one representative per SCC}
    \STATE \textit{// Query and update}
    \STATE $E \leftarrow E \cup O_{G^*}(Q)$, \quad $G \leftarrow (V, E)$
\ENDLOOP
\end{algorithmic}
\end{algorithm}

\tournamentsort{} (Algorithm~\ref{alg:tournament-sort-main}) realizes the tournament graph framework.
 At each iteration, the algorithm recomputes
the discovered in-reach $|\inreach_G(v)|$ and known-relationship count
$\kappa_G(v)$ for every vertex, forms the current top-$m$ set $T$ by
smallest discovered in-reach, and stops once every vertex in $T$ is
resolved. If not, it computes the condensation $\condense G$ and greedily queries
representatives from the unresolved SCCs with the smallest in-reaches in the condensation graph.

\textbf{Why greedy scheduling works.}
The greedy step operates on SCCs ordered by ascending in-reach in the
condensation. The progress argument has two parts. First, at any non-terminal
iteration, the earliest unresolved SCCs are tied in condensation in-reach;
this follows from the general structural lemma on tied candidates after
the finalization threshold (Lemma~\ref{lem:tied-candidates-after-finalization}),
applied to $\condense G$, with the fact that these tied SCCs
indeed contain unresolved vertices (Lemma~\ref{lem:tied-sccs-unresolved}).
Second, once two SCCs are tied, the forced-tie property (Lemma~\ref{lem:tied-unknown-edge-general})
implies that there is no edge between them in the condensation and thus no revealed edge between any of their members in $G$. 
Querying representatives from such SCCs therefore uncovers at least one new edge and progress is guaranteed at every round.
This is the key behind the termination proof for \tournamentsort{};
see Theorem~\ref{thm:practical-termination} in Appendix~\ref{sec:stronger-finalization-general}.

\textbf{Parallelization.}
When $k$ is small relative to $n$, multiple disjoint groups can be queried
in parallel. The algorithm naturally supports this because unresolved
SCCs are handled independently until they become connected by newly
revealed edges. 

\textbf{Variable window size.}
The framework places no constraint on $k$ being fixed across rounds.
In each iteration, $k$ can be chosen adaptively based on the candidate documents' lengths, the oracle's context window, or model-specific capabilities. This is a practical advantage for heterogeneous document collections where a fixed window size either truncates long documents or underutilizes the context on short ones.

\subsection{Proof Sketch}\label{sec:guarantees}

\tournamentsort{} stops when the current top-$m$ vertices by discovered
in-reach are all resolved.
We sketch why this stopping rule yields a correct output and why it is
always reached; the full proofs appear in
Appendix~\ref{sec:stronger-finalization-general}.

\paragraph{Output correctness.}
The correctness argument rests on the following property:
if $u$ is resolved (i.e. $\kappa_G(u) = n{-}1$), every
$w \neq u$ is either in $\inreach_G(u)$ or $\outreach_G(u)$.
We show in Lemma~\ref{lem:resolved-vertex-ordering} that whenever $u$ is resolved and has discovered in-reach no
larger than that of some vertex~$v$, then $u$ truly ranks at least as
well as $v$ in the unknown tournament~$G^*$ -- that is,
$|\inreach_{G^*}(u)| \le |\inreach_{G^*}(v)|$.


At termination, every vertex $u$ in the returned set~$T$ is resolved
and has discovered in-reach at most that of every vertex outside~$T$.
Applying the lemma to each such pair certifies both the internal
ordering within~$T$ and its rank dominance over all remaining
vertices
(Corollary~\ref{cor:resolved-inreaches-are-true},
Theorem~\ref{thm:practical-output-correctness}).

\paragraph{Termination.}
Every non-terminal round discovers at least one new edge, so the
algorithm halts in at most $\binom{n}{2}$ rounds.
The argument works on the condensation $\condense{G}$, which is
always a DAG.
In any DAG, the $i$-th vertex by in-reach has in-reach at most
$i{-}1$; the longest prefix where this bound is tight defines
the \emph{finalization threshold}~$\widetilde{m}$.
Non-terminality forces $\widetilde{m} \le n'{-}2$, and the two
positions just past the threshold must share the same in-reach
(otherwise finalization would extend further).
These ``tied'' SCCs have no revealed edge between them, so the
greedy schedule -- which prioritizes them
(\S\ref{sec:algorithm}) -- is guaranteed to discover a new edge
(Theorem~\ref{thm:practical-termination}).

\subsection{Tie-breaking and Query Complexity}\label{subsec:tie-breaking-and-query-complexity}

Algorithm~\ref{alg:tournament-sort-main} leaves two choices open:
the ordering among SCCs with equal condensation in-reach, and the
representative selected from each SCC.
The correctness and termination guarantees hold for any choice, but
query efficiency depends on them.
A natural heuristic for both is to \emph{prioritize the least-resolved
entity}.
Among SCCs with equal condensation in-reach, we query those with lower
out-reach (highest positional uncertainty) and select the representative vertex with
the smallest~$\kappa_G$ (fewest established relationships), concentrating
effort where information gain is highest.
We denote this fully specified schedule $\tournamentsort^{\dagger}$
(defined in Appendix~\ref{sec:query-complexity}) and use it in all
experiments (Section~\ref{sec:experiments}).

The termination proof gives a worst-case bound of $\binom{n}{2}$
queries, but in practice \tournamentsort{} terminates much sooner:
Figure~\ref{fig:horses} shows it achieves the optimal $7$ queries on the 25-horses instance.
For top-1 selection ($m=1$), each query eliminates at least $k-1$ candidates, giving at most $\lceil(n-1)/(k-1)\rceil$ queries (Proposition~\ref{thm:top-1-complexity}).
For general $m$, we conjecture that $\tournamentsort^{\dagger}$ achieves $O((n-1)/(k-1) + (m-1)/(k-1)\cdot\log_k m)$, decomposing into a \emph{candidate reduction} term and a \emph{frontier refinement} term (Conjecture~\ref{conj:query-complexity}).
Empirically, observed query counts remain within a factor of $1.25$ of this form across $n$ from $100$ to $800$ and $k$ from $5$ to $50$ (Figure~\ref{fig:q-vs-m-grid}), suggesting the bound is tight up to lower-order terms.
A formal proof for $m>1$ remains open; see Appendix~\ref{sec:query-complexity} for details.

\section{Related Work}
\label{sec:related}

\textbf{LLM-based document reranking.}
LLM reranking methods fall into three paradigms.
\emph{Pointwise} methods score documents independently, enabling parallelism but discarding relative information; zero-shot LLM approaches include query likelihood~\citep{sachan2022improving} and relevance classification~\citep{zhuang2024beyond}, while trained cross-encoder models~\citep{contextualai2025rerank} dominate in efficiency.
\emph{Pairwise} methods recover relative preferences: \citet{qin2024large} introduced Pairwise Ranking Prompting with heapsort aggregation at $O(n \log n)$ comparisons, noting that ``pairwise comparisons are not guaranteed to be transitive.''
\emph{Listwise} methods compare multiple documents per call: RankGPT~\citep{sun2023chatgpt} and LRL~\citep{ma2023zero} concurrently established the sliding-window paradigm, processing windows of 20 documents with stride 10, while subsequent work distilled this into open-source models~\citep{pradeep2023rankvicuna,pradeep2023rankzephyr}.
AcuRank~\citep{yoon2025acurank} maintains Gaussian distributions over document relevance and performs Bayesian updates via TrueSkill, selectively reranking uncertain documents until confidence criteria are met.
All of these methods use a fixed window size; our framework accommodates variable $k$ per round (\S\ref{sec:algorithm}), adapting to heterogeneous document lengths and model-specific context limits.

The \emph{setwise} approach~\citep{zhuang2024setwise} bridges pairwise and listwise by prompting the LLM to select the most relevant from $k$ candidates, using this primitive within heapsort.
While more efficient than pairwise, setwise extracts only the \emph{winner}, discarding the remaining $\binom{k}{2} - (k-1)$ pairwise relationships.
Our framework differs by extracting the \emph{complete tournament} from each $k$-wise comparison and accumulating edges where transitive closure amplifies each query's information yield.

\textbf{Handling inconsistency in LLM rankings.}
LLM judgments frequently violate transitivity.
\citet{zeng2024llm} address this via LLM-RankFusion, measuring inconsistency patterns and resolving them through rank aggregation.
\citet{tang2024found} show that shuffling and Kemeny-optimal aggregation mitigate positional biases; similarly, TourRank~\citep{chen2025tourrank} runs multiple parallel tournaments with different random seeds and aggregates scores.
Most related to our approach, ELSPR~\citep{yu2025elspr} uses SCC analysis to quantify non-transitivity via structural entropy, but applies it to filter training examples for evaluator LLMs rather than for query-efficient selection.
REALM~\citep{wang2025realm} takes a probabilistic approach, modeling relevance as Gaussian distributions updated via Bayesian inference.
In contrast, our framework is deterministic and treats cycles as structure (tiered rankings) rather than a defect to be aggregated away.

\textbf{Tournament theory and multi-wise comparisons.}
Our framework builds on tournament graph
theory~\citep{brandt2016handbook,laslier1997tournament,landau1953dominance}.
The query complexity of tournament solutions under a pairwise oracle
was studied by \citet{dey2017query}, who proved that finding the
Copeland set, Slater set, and most other solutions requires
$\Omega(n^{2})$ queries in the worst case. When the top cycle has
bounded size~$c$, all solutions can be found in
$O(nc + n\log n/\log(1-1/c))$ queries; however, this identifies only
the first SCC (the top cycle) and does not recover the full
condensation ordering over all components.

On the multi-wise comparison side, \citet{ren2021sample} establish
that full-ranking feedback from $k$-wise comparisons improves sample
complexity by a factor of $k\log k$ over winner feedback, and
\citet{saha2019pac} show similar gains under the Plackett--Luce
model. These results, along with near-optimal top-$m$ algorithms
under parametric models~\citep{jang2017optimal,chen2018nearly},
assume stochastic comparison models and target total rankings rather
than structural objectives like SCC recovery.

Our setting differs from all of the above in three respects: (i)~we
assume a deterministic tournament rather than a stochastic comparison
model, (ii)~we use a $k$-wise oracle ($k \geq 2$) that reveals
complete induced tournaments rather than restricting to pairwise
queries or winner-only feedback, and (iii)~we target the SCC
decomposition ordering -- a tiered ranking that declares cyclic items
tied -- rather than a total ranking or a single tournament solution
set; no existing algorithm addresses this combination.
A more detailed discussion appears in Appendix~\ref{sec:theoretical-related-work}.

\section{Experiments}\label{sec:experiments}
\begin{figure*}[t]
    \centering
    \begin{subfigure}[b]{0.48\textwidth}
        \centering
        \includegraphics[width=\textwidth]{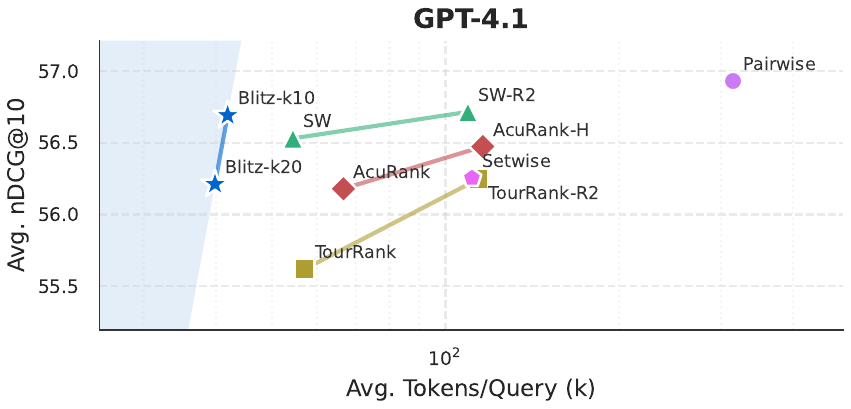}
    \end{subfigure}
    \hfill
    \begin{subfigure}[b]{0.48\textwidth}
        \centering
        \includegraphics[width=\textwidth]{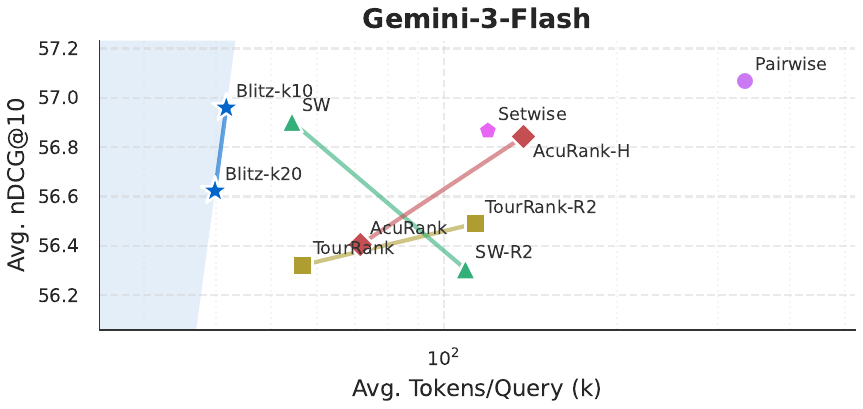}
    \end{subfigure}
        
    \begin{subfigure}[b]{0.48\textwidth}
        \centering
        \includegraphics[width=\textwidth]{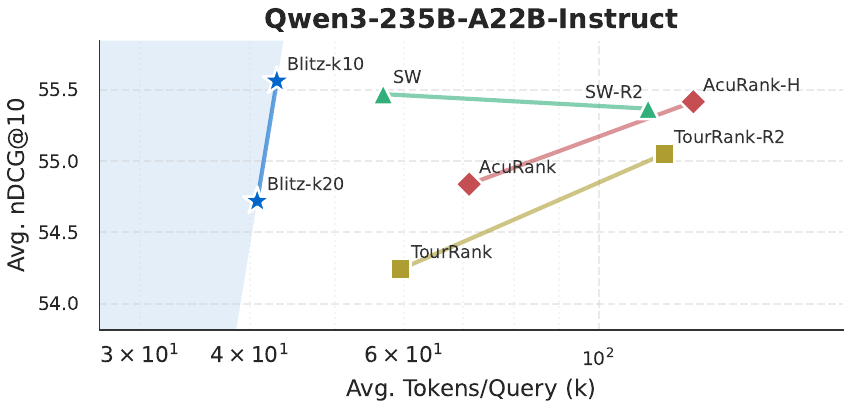}
    \end{subfigure}
    \hfill
    \begin{subfigure}[b]{0.48\textwidth}
        \centering
        \includegraphics[width=\textwidth]{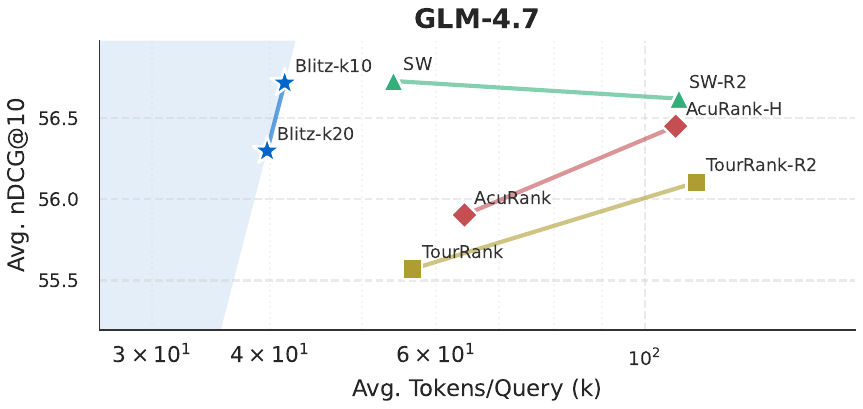}
    \end{subfigure}
    \caption{Pareto frontiers showing the accuracy-efficiency trade-off across LLM oracles. \tournamentsort{} (Algorithm~\ref{alg:tournament-sort-main}) consistently occupies the upper-left region, achieving competitive accuracy with 25--40\% fewer tokens than methods with comparable structure (\S\ref{sec:efficiency-details}).}
    \label{fig:pareto-comparison}
\end{figure*}

We evaluate \tournamentsort{} (Algorithm~\ref{alg:tournament-sort-main}) on standard document reranking benchmarks within the established retrieve-then-rerank pipeline, using an LLM as the zero-shot comparison oracle.
Our experiments address two primary questions:
(1)~Does extracting complete tournaments and propagating information via transitive closure translate to practical efficiency gains?
(2)~Do these efficiency gains compromise ranking quality, or can our framework match existing approaches with fewer oracle calls?
We compare against baseline reranking strategies across 14 datasets and 5 LLM oracles.

\subsection{Setup}
\label{sec:eval-protocol}

\textbf{Datasets.}
We evaluate on 14 datasets: six TREC Deep Learning tracks (DL19--DL23, DL-Hard) \citep{craswell2020overviewtrec2019deep, craswell2021overviewtrec2020deep, craswell2025overviewtrec2021deep, craswell2025overviewtrec2022deep, craswell2025overviewtrec2023deep, mackie2021dlhard} and eight BEIR datasets (TREC-COVID, NFCorpus, Signal-1M, News, Robust04, Touch\'e, DBPedia, SciFact) \citep{thakur2021beir}.
For each dataset, BM25~\citep{robertson2009probabilistic} retrieves the top-100 candidates per query.

\textbf{Metrics.} We measure ranking quality via nDCG@$k$~\cite{jarvelin2002cumulated} and efficiency via input tokens per query -- isolating algorithmic efficiency from implementation details. 

\textbf{Baselines.}
We compare against five LLM-based reranking methods: Sliding Windows~\citep{sun2023chatgpt} (single-pass and two-pass), TourRank~\citep{chen2025tourrank} (single and two-round), AcuRank~\citep{yoon2025acurank} (standard and high-precision), Setwise~\citep{zhuang2024setwise}, and Pairwise~\citep{qin2024large}.
See Appendix~\ref{sec:baseline-details} for detailed comparison rationale.

\textbf{Oracles.}
All methods use the same underlying LLM with the RankGPT prompt format~\citep{sun2023chatgpt}.
We evaluate five diverse LLMs: GPT-4.1, Gemini-3-Flash, GLM-4.7~\citep{5team2025glm45agenticreasoningcoding}, DeepSeek-V3.2~\citep{deepseekai2025deepseekv32pushingfrontieropen}, and Qwen3-235B-A22B-Instruct~\citep{qwen3technicalreport}, allowing us to assess whether our efficiency gains generalize across models of varying capabilities.

\textbf{Our method.}
We evaluate \tournamentsort{}, or simply \algshortname{}, with two window sizes: $k=10$ and $k=20$.
We set target $m = 10$, so \tournamentsort{} terminates once the current top-10 documents are resolved.

\subsection{Main Results: Accuracy-Efficiency Frontier}
\label{sec:main-results}

\begin{table}[t]
\centering
\caption{Summary of reranking quality and efficiency, macro-averaged across 14 datasets and 2 LLM oracles (GPT-4.1 \& Gemini-3-Flash). Relative cost is measured against \algshortname{}-k20.}
\label{tab:summary-results}
\small
\begin{tabular}{lcrc}
\toprule
Method & nDCG@10 & Tokens & Rel.\ Cost \\
\midrule
BM25 (no rerank) & 41.1 & 0 & --- \\
\midrule
Pairwise & \textbf{57.0} & 324k & 8.1$\times$ \\
Setwise & 56.6 & 115k & 2.9$\times$ \\
\midrule
TourRank & 56.0 & 57k & 1.4$\times$ \\
TourRank-R2 & 56.4 & 114k & 2.8$\times$ \\
SW & 56.7 & 54k & 1.4$\times$ \\
SW-R2 & 56.5 & 109k & 2.7$\times$ \\
AcuRank & 56.3 & 69k & 1.7$\times$ \\
AcuRank-H & 56.6 & 127k & 3.2$\times$ \\
\midrule
\algshortname{}-k20 & 56.4 & \textbf{40k} & \textbf{1.0}$\times$ \\
\algshortname{}-k10 & \underline{56.9} & \underline{42k} & \underline{1.1}$\times$ \\
\bottomrule
\end{tabular}
\end{table}

\begin{table}[ht]
\centering
\caption{Macro-averaged nDCG@$k$ across 14 datasets $\times$ 5 models.}
\label{tab:ndcg-cutoffs}
\begin{tabular}{lccc}
\toprule
Method & nDCG@1 & nDCG@5 & nDCG@10 \\
\midrule
BM25 retrieval & 46.5 & 43.0 & 41.1 \\
TourRank & 63.3 & 58.8 & 55.5 \\
TourRank-R2 & 63.5 & 59.3 & 56.1 \\
SW & \textbf{66.2} & \underline{59.8} & \textbf{56.6} \\
SW-R2 & 65.7 & 59.6 & 56.4 \\
AcuRank & 64.5 & 59.5 & 55.9 \\
AcuRank-H & 65.6 & \underline{59.8} & 56.4 \\
\midrule
\textbf{\algshortname{}-k20} & \underline{66.1} & 59.7 & 56.0 \\
\textbf{\algshortname{}-k10} & 66.0 & \textbf{60.1} & \underline{56.5} \\
\bottomrule
\end{tabular}
\end{table}

Our main results (summarized in Table~\ref{tab:summary-results} and Figure~\ref{fig:pareto-comparison}; full per-dataset breakdown in Table~\ref{tab:main-results}) show that the tournament graph framework achieves comparable or superior ranking quality while consuming substantially fewer tokens, validating the hypothesis that transitive closure yields practical efficiency gains.

\textbf{Efficiency gains.}
\tournamentsort{} requires significantly fewer tokens than all baselines across every oracle.
With GPT-4.1, \algshortname{}-k10 and \algshortname{}-k20 consume 42k and 40k tokens per query, respectively -- a 22--26\% reduction compared to SW (54k) and TourRank (57k), and 37--40\% fewer than AcuRank (67k).
The gap widens dramatically against comparison-based methods: Pairwise requires 315k tokens (7.5$\times$ more than \algshortname{}-k10) and Setwise requires 111k tokens (2.6$\times$ more).
Higher-computation variants exacerbate these differences: SW-R2 (109k), TourRank-R2 (114k), and AcuRank-H (116k) all consume 2.6--2.9$\times$ more tokens than \algshortname{}-k10.

\textbf{Ranking quality and Model Generalization.}
The efficiency-accuracy trade-off generalizes robustly across all five LLM oracles (Table~\ref{tab:summary-results}).
Despite using fewer oracle calls, \tournamentsort{} matches or exceeds baseline accuracy across almost all configurations.
With GPT-4.1, \algshortname{}-k10 achieves 56.7 average nDCG@10 -- matching the best sliding-window variant (SW-R2: 56.7) at less than 40\% of its token cost, and outperforming SW (56.5), TourRank (55.6), AcuRank (56.2), and Setwise (56.3). Only Pairwise is marginally higher (56.9) but at 7.5$\times$ the cost.
With Gemini-3-Flash, \algshortname{}-k10 achieves 57.0 nDCG@10, within 0.1 points of Pairwise (57.1) while using 8$\times$ fewer tokens (42k vs.\ 334k).
With GLM-4.7, \algshortname{}-k10 (56.7) matches SW (56.7) and exceeds TourRank (55.6) at 24\% lower token cost (41k vs.\ 54k).
With DeepSeek-V3.2, \algshortname{}-k10 (56.6) still outperforms TourRank (55.5) and AcuRank (56.0) at lower cost.
With Qwen3-235B -- the weakest oracle -- \algshortname{}-k10 (55.6) exceeds all baselines except SW (55.5, within noise), again at reduced token consumption (43k vs.\ 57k). Table~\ref{tab:ndcg-cutoffs} shows that trends for nDCG@1 and nDCG@5 are consistent.

\textbf{Pareto dominance.}
Across all oracles, \tournamentsort{} consistently achieves quality comparable to the most accurate methods at a fraction of their cost (Figure~\ref{fig:pareto-comparison}).
Methods that match \algshortname{}-k10's accuracy (e.g., SW-R2, AcuRank-H) require 2--3$\times$ more tokens; methods that match its efficiency achieve lower accuracy.
Appendix~\ref{sec:efficiency-details} provides further efficiency details.

\subsection{Effect of Window Size}
\label{sec:window-size}

\tournamentsort{}'s primary user-facing hyperparameter is the window size $k$.
Across all 5 oracles (Table~\ref{tab:main-results}), \algshortname{}-k10 consistently outperforms \algshortname{}-k20 in ranking quality: by 0.5 with GPT-4.1 (56.7 vs.\ 56.2), 0.4 with Gemini-3-Flash (57.0 vs.\ 56.6), and 0.9 with Qwen3-235B (55.6 vs.\ 54.7).
The token cost difference between configurations is modest (40--43k), suggesting that smaller windows, with more queries but finer-grained comparisons, provide better accuracy.
We note that $k{=}10$ outperforms $k{=}20$ despite comparing fewer documents per query; as analyzed in \S\ref{sec:scc-analysis}, larger windows cause LLMs to produce more cyclic judgments among similar documents due to positional attention limitations.
Additionally, convergence is highly predictable: for $k{=}10$, \tournamentsort{} terminates in 12--15 rounds (mean 13.6, std 0.58), reflecting the progress guarantee of the greedy SCC schedule (Theorem~\ref{thm:practical-termination}) and enabling reliable cost estimation.
See Appendix~\ref{sec:convergence-details} for detailed convergence analysis.

\begin{table}[t]
\centering
\caption{nDCG@10 at matched window sizes (GPT-4.1, TREC DL19 and DL20). Sliding Window's quality degrades sharply at $k{=}10$ while \tournamentsort{} maintains comparable performance.}
\label{tab:sw-k10}
\begin{tabular}{llrr}
\toprule
\textbf{Method} & \textbf{k} & \textbf{DL19} & \textbf{DL20} \\
\midrule
Sliding Window    & 20 & 74.0 & 70.8 \\
Sliding Window    & 10 & 56.4 & 53.2 \\
\midrule 
\tournamentsort   & 20 & 74.6 & 70.7 \\
\tournamentsort   & 10 & 73.6 & 72.4 \\
\bottomrule
\end{tabular}
\end{table}

\textbf{Comparison with Sliding Window at $k{=}10$.}
Table~\ref{tab:sw-k10} compares both methods at matched window sizes on DL19 and DL20.
Sliding Window's quality degrades sharply from $k{=}20$ to $k{=}10$ (74.0$\to$56.4 on DL19), because with stride 5 it propagates only the top-5 documents per pass which is insufficient to identify the top 10.
In contrast, \tournamentsort{} at $k{=}10$ maintains quality comparable to $k{=}20$ (73.6 vs.\ 74.6 on DL19), because correctness is certified by the resolution criterion rather than window coverage.

\subsection{Analysis of Strongly Connected Components}
\label{sec:scc-analysis}

We analyze SCCs formed by \tournamentsort{} on DL19 using GPT-4.1 to understand when and why cycles arise, and what they reveal about ranking difficulty.

\begin{table}[t]
\centering
\caption{SCC statistics at convergence on DL19 (43 queries) using GPT-4.1. Larger $k$ requires fewer rounds but produces more cycles. This reflects a pattern of increased cycle occurrences when comparing many  documents simultaneously.}
\label{tab:scc-evolution}
\small
\begin{tabular}{lcccc}
\toprule
Config & Queries & \#Rounds & Final \#SCCs & Avg SCC Size \\
\midrule
$k=10$ & 43 & 13.4 & 93.7 & 1.069 \\
$k=20$ & 43 & 6.6 & 85.3 & 1.180 \\
\bottomrule
\end{tabular}
\end{table}

\begin{figure}[t]
    \centering
    \begin{subfigure}[b]{0.48\columnwidth}
        \centering
        \includegraphics[width=\textwidth]{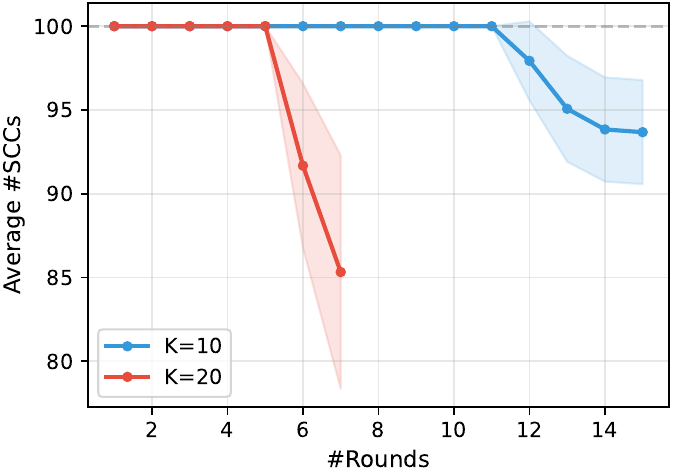}
        \caption{Number of SCCs}
        \label{fig:scc-evolution-count}
    \end{subfigure}
    \hfill
    \begin{subfigure}[b]{0.48\columnwidth}
        \centering
        \includegraphics[width=\textwidth]{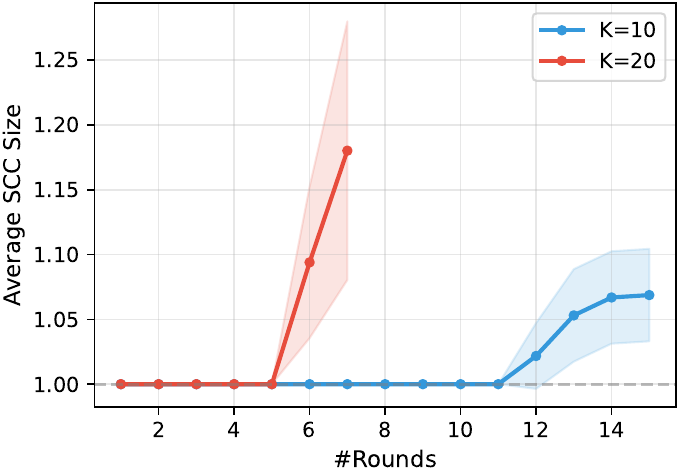}
        \caption{Average SCC size}
        \label{fig:scc-evolution-size}
    \end{subfigure}
    \caption{Evolution of SCCs on DL19 with GPT-4.1. Solid lines and shaded regions show means and variance across queries, respectively. \textbf{(a)} Both $k$'s begin with 100 singleton SCCs. $k{=}20$ forms cycles earlier (rounds 5--7), ending with $\sim$85 SCCs. $k{=}10$ forms fewer cycles and also later. \textbf{(b)} Average SCC size follows a similar pattern: $k{=}20$ reaches 1.18 average size by round 7, while $k{=}10$ reaches only 1.07 by round 15. 
    }
    \label{fig:scc-evolution}
\end{figure}

\textbf{Evolution of Strongly Connected Components.}
Table~\ref{tab:scc-evolution} summarizes SCC statistics at convergence, and Figure~\ref{fig:scc-evolution} traces their evolution across rounds.
Both configurations begin with 100 singleton SCCs (one per document).
With $k{=}20$, merging begins earlier -- around rounds 5--7 -- as \tournamentsort{} compares larger document sets where cyclic judgments are more likely.
By convergence, $k{=}20$ produces 85.3 SCCs with average size 1.18, while $k{=}10$ produces 93.7 SCCs with average size 1.07.

The shaded regions in Figure~\ref{fig:scc-evolution} indicate variance across queries.
The $k{=}20$ configuration exhibits higher variance in both SCC count and size, reflecting greater sensitivity to query-specific document similarity.
This explains the accuracy gap between configurations (Table~\ref{tab:main-results}): larger windows induce more ties, and ties at the top-$m$ boundary directly impact nDCG@10.

\begin{figure}[t]
    \centering
    \begin{subfigure}[b]{\columnwidth}
        \centering
        \includegraphics[width=\textwidth]{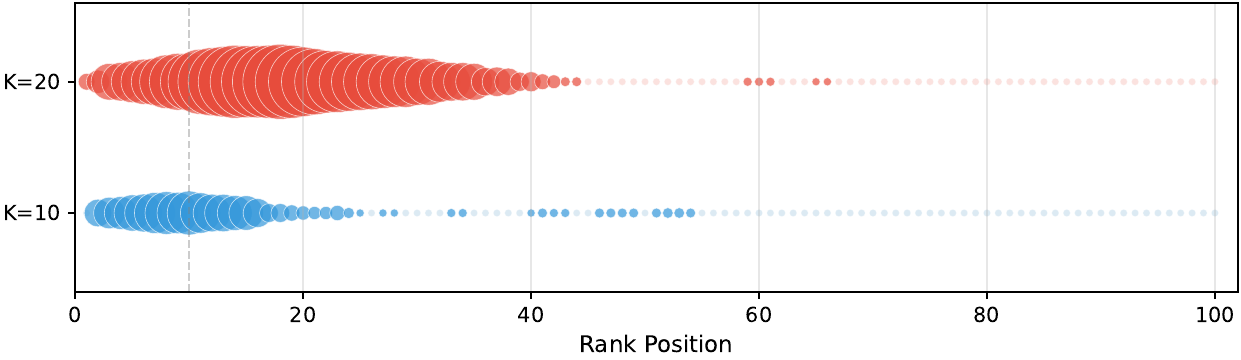}
        \caption{SCC locations by rank position (bubble size $\propto$ SCC size)}
        \label{fig:scc-locations}
    \end{subfigure}    
    \begin{subfigure}[b]{\columnwidth}
        \centering
        \includegraphics[width=\textwidth]{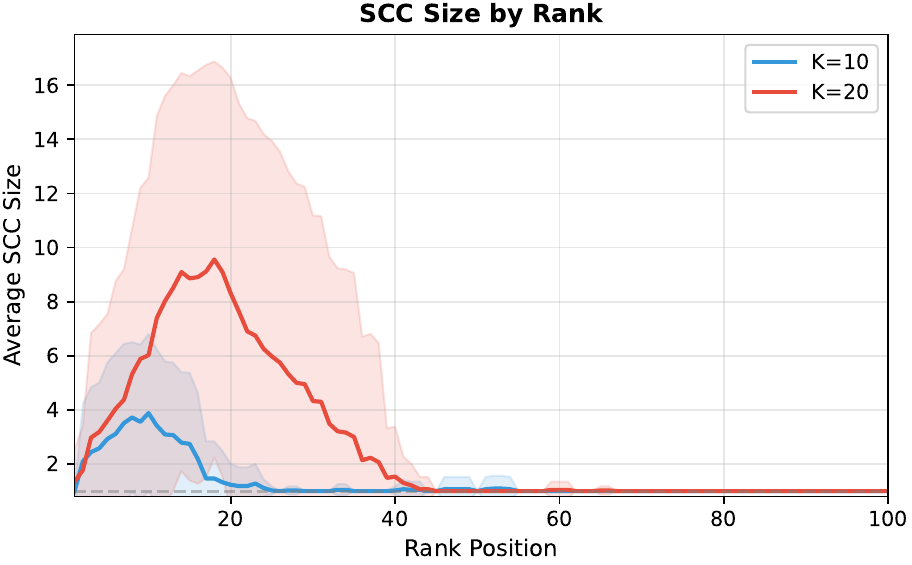}
        \caption{Average SCC size by rank position}
        \label{fig:scc-size-by-rank}
    \end{subfigure}
    \caption{Spatial distribution of strongly connected components after convergence on DL19 using GPT-4.1. \\ \textbf{(a)} Each bubble represents an SCC; size indicates number of documents. The dashed vertical line marks rank 10 (top-$m$ boundary).  $k{=}20$ produces larger SCCs concentrated in ranks 10--40, while $k{=}10$ produces smaller SCCs focused in ranks 1--20. \\ \textbf{(b)} Average SCC size exhibits a wave pattern: both configurations show peaks in mid-ranks rather than at the tail. 
    }
    \label{fig:scc-position}
\end{figure}


\textbf{Larger windows produce larger SCCs.}
At every rank position, $k{=}20$ produces larger SCCs than $k{=}10$.
The effect is most pronounced in mid-ranks: at rank 15, $k{=}20$ averages 8.86 documents per SCC compared to 2.74 for $k{=}10$.
This likely reflects the ``lost in the middle'' phenomenon~\citep{liu2024lost}: when comparing 20 documents simultaneously, LLMs struggle to attend to items in the middle of the list, producing inconsistent judgments that manifest as cycles.
Smaller windows ($k{=}10$) reduce this burden, yielding more consistent orderings.

\begin{table}[t]
\centering
\caption{BM25 score variance within SCCs vs.\ neighboring documents on DL19 using GPT-4.1. Documents within SCCs exhibit ${\sim}40\%$ lower BM25 variance than equal-sized neighbor groups (ratio ${\sim}0.6$), confirming that cycles capture genuinely similar documents rather than arbitrary ties.}
\label{tab:scc-bm25-variance}
\tiny
\begin{tabular}{lccccc}
\toprule
Config & SCCs ($\geq$2) & Avg Size & Within-SCC Std & Neighbor Std & Ratio \\
\midrule
$K=10$ & $100$ & $3.72$ & $0.605$ & $1.032$ & $0.59$ \\
$K=20$ & $126$ & $6.01$ & $0.695$ & $1.125$ & $0.62$ \\
\bottomrule
\end{tabular}
\end{table}

\textbf{SCCs Capture Genuinely Similar Documents.}
We use BM25 score variance as a proxy for document similarity: documents with similar lexical relevance to the query should have similar BM25 scores.
For each SCC of size $\geq$2, we compute the standard deviation of BM25 scores within the SCC and among equal-sized groups of neighboring (non-tied) documents.
Table~\ref{tab:scc-bm25-variance} shows that documents within SCCs have substantially lower BM25 variance than their neighbors.
For $k{=}10$, within-SCC standard deviation is 0.605 compared to 1.032 for neighbors -- a ratio of 0.59.
For $k{=}20$, the ratio is 0.62 (0.695 vs.\ 1.125).
This ${\sim}40\%$ reduction in variance confirms that cycles capture documents that are genuinely similar in lexical relevance, not arbitrary ties.
Notably, $k{=}10$ SCCs have lower within-SCC variance (0.605 vs.\ 0.695) despite smaller average SCC size (3.72 vs.\ 6.01). This suggests that $k{=}10$ successfully resolves easier ambiguities through finer comparisons, leaving only the most difficult cases as unresolved cycles -- explaining its consistent accuracy advantage.

\textbf{Oracle noise.}
Our framework assumes a deterministic oracle (\S\ref{sec:setup}), but LLM oracles are inherently noisy.
The SCC analysis above provides indirect evidence that this noise is well-behaved: cycles capture genuinely similar documents (${\sim}40\%$ lower BM25 variance), suggesting that noisy edges are concentrated among hard-to-distinguish items rather than corrupting long transitive chains.
The consistent performance across five oracles of varying capability further indicates that the efficiency gains are robust to oracle-specific noise characteristics.

\section{Conclusion}

We presented a tournament graph framework for top-$m$ selection via $k$-wise comparison oracles, alongside \tournamentsort{}, a greedy algorithm with provable correctness and termination guarantees.
Empirically, \tournamentsort{} achieves Pareto dominance across 14 benchmarks and 5 LLM oracles, matching or exceeding accuracy while reducing token consumption by 25--40\% compared to comparable methods, and up to 7$\times$ compared to pairwise approaches.
Moreover, convergence is highly predictable, enabling reliable cost estimation.
Finally, our analysis of strongly connected components confirms that cycles capture genuine lexical similarity rather than arbitrary noise, validating the interpretation of non-transitivity as structure.

\subsection{Future Work}\label{sec:future-work}
\textbf{Noisy oracle.} Our framework assumes a deterministic oracle: each pairwise comparison has a fixed ground-truth outcome.
Real oracles -- LLMs, crowdworkers, human experts -- are noisy, and this noise interacts asymmetrically with transitive inference: in principle, a single erroneous edge can collapse a long chain into an SCC.
While our empirical analysis suggests such catastrophic collapse is rare with capable models, a principled treatment of noise remains a key challenge.
We also leave the explicit incorporation of prior retrieval scores into query selection to future work.

\textbf{Query complexity.}
We showed $\lceil(n-1)/(k-1)\rceil$ query complexity for top-1 selection and conjectured a tight bound for general $m$ (Conjecture~\ref{conj:query-complexity}).
A formal proof or matching lower bound remains open.

\textbf{Noisy and probabilistic oracles.}
Extending the framework to handle oracle noise is a significant open problem.
One approach models each edge as a random variable whose confidence accumulates across repeated or corroborating observations.
A key challenge is that edges differ in \emph{structural importance}: an edge along a long chain discriminates many pairs transitively, so its corruption is catastrophic, while a leaf edge affects only one comparison.
Models that account for this asymmetry -- weighting edges by transitive reach, or using soft SCCs that resist collapse from isolated errors -- could yield algorithms that degrade gracefully under noise.

\textbf{Incorporating priors.} First-stage retrieval scores provide a natural prior over the ranking.
Algorithms that initialize edge beliefs with priors, or prioritize queries where the prior is uncertain, could improve both efficiency and robustness. This connects to active learning formulations where queries maximize information gain relative to prior belief.

We hope the deterministic framework developed here provides a foundation for these extensions, much as noiseless sorting algorithms underpin the study of noisy comparison models.
\section*{Impact Statement}

This paper develops algorithmic methods for query-efficient ranking using expensive comparison oracles. By reducing the number of oracle calls required to identify top candidates, our framework lowers computational costs when LLMs serve as oracles -- contributing to more sustainable use of large-scale models. The techniques are domain-agnostic and apply broadly to ranking problems with costly comparisons (crowdsourcing, human evaluation, tournament design). We do not foresee specific negative societal consequences beyond those common to advances in information retrieval and ranking systems.
\bibliographystyle{icml2026}
\bibliography{bibliography}

@inproceedings{dey2017query,
  title     = {Query complexity of tournament solutions},
  author    = {Dey, Palash},
  booktitle = {Proceedings of the AAAI Conference on Artificial Intelligence},
  volume    = {31},
  year      = {2017}
}

@misc{qwen3technicalreport,
  title         = {Qwen3 Technical Report},
  author        = {Qwen Team},
  year          = {2025},
  eprint        = {2505.09388},
  archiveprefix = {arXiv},
  primaryclass  = {cs.CL},
  url           = {https://arxiv.org/abs/2505.09388}
}

@article{ma2023zero,
  title   = {Zero-shot listwise document reranking with a large language model},
  author  = {Ma, Xueguang and Zhang, Xinyu and Pradeep, Ronak and Lin, Jimmy},
  journal = {arXiv preprint arXiv:2305.02156},
  year    = {2023}
}

@misc{deepseekai2025deepseekv32pushingfrontieropen,
  title         = {DeepSeek-V3.2: Pushing the Frontier of Open Large Language Models},
  author        = {DeepSeek-AI and Aixin Liu and Aoxue Mei and Bangcai Lin and Bing Xue and Bingxuan Wang and Bingzheng Xu and Bochao Wu and Bowei Zhang and Chaofan Lin and Chen Dong and Chengda Lu and Chenggang Zhao and Chengqi Deng and Chenhao Xu and Chong Ruan and Damai Dai and Daya Guo and Dejian Yang and Deli Chen and Erhang Li and Fangqi Zhou and Fangyun Lin and Fucong Dai and Guangbo Hao and Guanting Chen and Guowei Li and H. Zhang and Hanwei Xu and Hao Li and Haofen Liang and Haoran Wei and Haowei Zhang and Haowen Luo and Haozhe Ji and Honghui Ding and Hongxuan Tang and Huanqi Cao and Huazuo Gao and Hui Qu and Hui Zeng and Jialiang Huang and Jiashi Li and Jiaxin Xu and Jiewen Hu and Jingchang Chen and Jingting Xiang and Jingyang Yuan and Jingyuan Cheng and Jinhua Zhu and Jun Ran and Junguang Jiang and Junjie Qiu and Junlong Li and Junxiao Song and Kai Dong and Kaige Gao and Kang Guan and Kexin Huang and Kexing Zhou and Kezhao Huang and Kuai Yu and Lean Wang and Lecong Zhang and Lei Wang and Liang Zhao and Liangsheng Yin and Lihua Guo and Lingxiao Luo and Linwang Ma and Litong Wang and Liyue Zhang and M. S. Di and M. Y Xu and Mingchuan Zhang and Minghua Zhang and Minghui Tang and Mingxu Zhou and Panpan Huang and Peixin Cong and Peiyi Wang and Qiancheng Wang and Qihao Zhu and Qingyang Li and Qinyu Chen and Qiushi Du and Ruiling Xu and Ruiqi Ge and Ruisong Zhang and Ruizhe Pan and Runji Wang and Runqiu Yin and Runxin Xu and Ruomeng Shen and Ruoyu Zhang and S. H. Liu and Shanghao Lu and Shangyan Zhou and Shanhuang Chen and Shaofei Cai and Shaoyuan Chen and Shengding Hu and Shengyu Liu and Shiqiang Hu and Shirong Ma and Shiyu Wang and Shuiping Yu and Shunfeng Zhou and Shuting Pan and Songyang Zhou and Tao Ni and Tao Yun and Tian Pei and Tian Ye and Tianyuan Yue and Wangding Zeng and Wen Liu and Wenfeng Liang and Wenjie Pang and Wenjing Luo and Wenjun Gao and Wentao Zhang and Xi Gao and Xiangwen Wang and Xiao Bi and Xiaodong Liu and Xiaohan Wang and Xiaokang Chen and Xiaokang Zhang and Xiaotao Nie and Xin Cheng and Xin Liu and Xin Xie and Xingchao Liu and Xingkai Yu and Xingyou Li and Xinyu Yang and Xinyuan Li and Xu Chen and Xuecheng Su and Xuehai Pan and Xuheng Lin and Xuwei Fu and Y. Q. Wang and Yang Zhang and Yanhong Xu and Yanru Ma and Yao Li and Yao Li and Yao Zhao and Yaofeng Sun and Yaohui Wang and Yi Qian and Yi Yu and Yichao Zhang and Yifan Ding and Yifan Shi and Yiliang Xiong and Ying He and Ying Zhou and Yinmin Zhong and Yishi Piao and Yisong Wang and Yixiao Chen and Yixuan Tan and Yixuan Wei and Yiyang Ma and Yiyuan Liu and Yonglun Yang and Yongqiang Guo and Yongtong Wu and Yu Wu and Yuan Cheng and Yuan Ou and Yuanfan Xu and Yuduan Wang and Yue Gong and Yuhan Wu and Yuheng Zou and Yukun Li and Yunfan Xiong and Yuxiang Luo and Yuxiang You and Yuxuan Liu and Yuyang Zhou and Z. F. Wu and Z. Z. Ren and Zehua Zhao and Zehui Ren and Zhangli Sha and Zhe Fu and Zhean Xu and Zhenda Xie and Zhengyan Zhang and Zhewen Hao and Zhibin Gou and Zhicheng Ma and Zhigang Yan and Zhihong Shao and Zhixian Huang and Zhiyu Wu and Zhuoshu Li and Zhuping Zhang and Zian Xu and Zihao Wang and Zihui Gu and Zijia Zhu and Zilin Li and Zipeng Zhang and Ziwei Xie and Ziyi Gao and Zizheng Pan and Zongqing Yao and Bei Feng and Hui Li and J. L. Cai and Jiaqi Ni and Lei Xu and Meng Li and Ning Tian and R. J. Chen and R. L. Jin and S. S. Li and Shuang Zhou and Tianyu Sun and X. Q. Li and Xiangyue Jin and Xiaojin Shen and Xiaosha Chen and Xinnan Song and Xinyi Zhou and Y. X. Zhu and Yanping Huang and Yaohui Li and Yi Zheng and Yuchen Zhu and Yunxian Ma and Zhen Huang and Zhipeng Xu and Zhongyu Zhang and Dongjie Ji and Jian Liang and Jianzhong Guo and Jin Chen and Leyi Xia and Miaojun Wang and Mingming Li and Peng Zhang and Ruyi Chen and Shangmian Sun and Shaoqing Wu and Shengfeng Ye and T. Wang and W. L. Xiao and Wei An and Xianzu Wang and Xiaowen Sun and Xiaoxiang Wang and Ying Tang and Yukun Zha and Zekai Zhang and Zhe Ju and Zhen Zhang and Zihua Qu},
  year          = {2025},
  eprint        = {2512.02556},
  archiveprefix = {arXiv},
  primaryclass  = {cs.CL},
  url           = {https://arxiv.org/abs/2512.02556}
}

@misc{5team2025glm45agenticreasoningcoding,
  title         = {GLM-4.5: Agentic, Reasoning, and Coding (ARC) Foundation Models},
  author        = {GLM Team and Aohan Zeng and Xin Lv and Qinkai Zheng and Zhenyu Hou and Bin Chen and Chengxing Xie and Cunxiang Wang and Da Yin and Hao Zeng and Jiajie Zhang and Kedong Wang and Lucen Zhong and Mingdao Liu and Rui Lu and Shulin Cao and Xiaohan Zhang and Xuancheng Huang and Yao Wei and Yean Cheng and Yifan An and Yilin Niu and Yuanhao Wen and Yushi Bai and Zhengxiao Du and Zihan Wang and Zilin Zhu and Bohan Zhang and Bosi Wen and Bowen Wu and Bowen Xu and Can Huang and Casey Zhao and Changpeng Cai and Chao Yu and Chen Li and Chendi Ge and Chenghua Huang and Chenhui Zhang and Chenxi Xu and Chenzheng Zhu and Chuang Li and Congfeng Yin and Daoyan Lin and Dayong Yang and Dazhi Jiang and Ding Ai and Erle Zhu and Fei Wang and Gengzheng Pan and Guo Wang and Hailong Sun and Haitao Li and Haiyang Li and Haiyi Hu and Hanyu Zhang and Hao Peng and Hao Tai and Haoke Zhang and Haoran Wang and Haoyu Yang and He Liu and He Zhao and Hongwei Liu and Hongxi Yan and Huan Liu and Huilong Chen and Ji Li and Jiajing Zhao and Jiamin Ren and Jian Jiao and Jiani Zhao and Jianyang Yan and Jiaqi Wang and Jiayi Gui and Jiayue Zhao and Jie Liu and Jijie Li and Jing Li and Jing Lu and Jingsen Wang and Jingwei Yuan and Jingxuan Li and Jingzhao Du and Jinhua Du and Jinxin Liu and Junkai Zhi and Junli Gao and Ke Wang and Lekang Yang and Liang Xu and Lin Fan and Lindong Wu and Lintao Ding and Lu Wang and Man Zhang and Minghao Li and Minghuan Xu and Mingming Zhao and Mingshu Zhai and Pengfan Du and Qian Dong and Shangde Lei and Shangqing Tu and Shangtong Yang and Shaoyou Lu and Shijie Li and Shuang Li and Shuang-Li and Shuxun Yang and Sibo Yi and Tianshu Yu and Wei Tian and Weihan Wang and Wenbo Yu and Weng Lam Tam and Wenjie Liang and Wentao Liu and Xiao Wang and Xiaohan Jia and Xiaotao Gu and Xiaoying Ling and Xin Wang and Xing Fan and Xingru Pan and Xinyuan Zhang and Xinze Zhang and Xiuqing Fu and Xunkai Zhang and Yabo Xu and Yandong Wu and Yida Lu and Yidong Wang and Yilin Zhou and Yiming Pan and Ying Zhang and Yingli Wang and Yingru Li and Yinpei Su and Yipeng Geng and Yitong Zhu and Yongkun Yang and Yuhang Li and Yuhao Wu and Yujiang Li and Yunan Liu and Yunqing Wang and Yuntao Li and Yuxuan Zhang and Zezhen Liu and Zhen Yang and Zhengda Zhou and Zhongpei Qiao and Zhuoer Feng and Zhuorui Liu and Zichen Zhang and Zihan Wang and Zijun Yao and Zikang Wang and Ziqiang Liu and Ziwei Chai and Zixuan Li and Zuodong Zhao and Wenguang Chen and Jidong Zhai and Bin Xu and Minlie Huang and Hongning Wang and Juanzi Li and Yuxiao Dong and Jie Tang},
  year          = {2025},
  eprint        = {2508.06471},
  archiveprefix = {arXiv},
  primaryclass  = {cs.CL},
  url           = {https://arxiv.org/abs/2508.06471}
}

@article{jarvelin2002cumulated,
  title     = {Cumulated gain-based evaluation of IR techniques},
  author    = {J{\"a}rvelin, Kalervo and Kek{\"a}l{\"a}inen, Jaana},
  journal   = {ACM Transactions on Information Systems (TOIS)},
  volume    = {20},
  number    = {4},
  pages     = {422--446},
  year      = {2002},
  publisher = {ACM New York, NY, USA}
}

@article{robertson2009probabilistic,
  title     = {The probabilistic relevance framework: BM25 and beyond},
  author    = {Robertson, Stephen and Zaragoza, Hugo and others},
  journal   = {Foundations and trends{\textregistered} in information retrieval},
  volume    = {3},
  number    = {4},
  pages     = {333--389},
  year      = {2009},
  publisher = {Now Publishers, Inc.}
}

@inproceedings{thakur2021beir,
  title     = {{BEIR}: A Heterogeneous Benchmark for Zero-shot Evaluation of Information Retrieval Models},
  author    = {Nandan Thakur and Nils Reimers and Andreas R{\"u}ckl{\'e} and Abhishek Srivastava and Iryna Gurevych},
  booktitle = {Thirty-fifth Conference on Neural Information Processing Systems Datasets and Benchmarks Track (Round 2)},
  year      = {2021},
  url       = {https://openreview.net/forum?id=wCu6T5xFjeJ}
}

@inproceedings{sachan2022improving,
  title     = {Improving passage retrieval with zero-shot question generation},
  author    = {Sachan, Devendra and Lewis, Mike and Joshi, Mandar and Aghajanyan, Armen and Yih, Wen-tau and Pineau, Joelle and Zettlemoyer, Luke},
  booktitle = {Proceedings of the 2022 Conference on Empirical Methods in Natural Language Processing},
  pages     = {3781--3797},
  year      = {2022}
}

@inproceedings{zhuang2024beyond,
  title     = {Beyond yes and no: Improving zero-shot llm rankers via scoring fine-grained relevance labels},
  author    = {Zhuang, Honglei and Qin, Zhen and Hui, Kai and Wu, Junru and Yan, Le and Wang, Xuanhui and Bendersky, Michael},
  booktitle = {Proceedings of the 2024 conference of the North American chapter of the Association for Computational Linguistics: Human language technologies (volume 2: short papers)},
  pages     = {358--370},
  year      = {2024}
}

@inproceedings{chen2025tourrank,
  title     = {Tourrank: Utilizing large language models for documents ranking with a tournament-inspired strategy},
  author    = {Chen, Yiqun and Liu, Qi and Zhang, Yi and Sun, Weiwei and Ma, Xinyu and Yang, Wei and Shi, Daiting and Mao, Jiaxin and Yin, Dawei},
  booktitle = {Proceedings of the ACM on Web Conference 2025},
  pages     = {1638--1652},
  year      = {2025}
}

@misc{craswell2025overviewtrec2023deep,
  title         = {Overview of the TREC 2023 deep learning track},
  author        = {Nick Craswell and Bhaskar Mitra and Emine Yilmaz and Hossein A. Rahmani and Daniel Campos and Jimmy Lin and Ellen M. Voorhees and Ian Soboroff},
  year          = {2025},
  eprint        = {2507.08890},
  archiveprefix = {arXiv},
  primaryclass  = {cs.IR},
  url           = {https://arxiv.org/abs/2507.08890}
}

@article{craswell2025overviewtrec2022deep,
  title   = {Overview of the TREC 2022 deep learning track},
  author  = {Craswell, Nick and Mitra, Bhaskar and Yilmaz, Emine and Campos, Daniel and Lin, Jimmy and Voorhees, Ellen M and Soboroff, Ian},
  journal = {arXiv preprint arXiv:2507.10865},
  year    = {2025}
}

@misc{craswell2025overviewtrec2021deep,
  title         = {Overview of the TREC 2021 deep learning track},
  author        = {Nick Craswell and Bhaskar Mitra and Emine Yilmaz and Daniel Campos and Jimmy Lin},
  year          = {2025},
  eprint        = {2507.08191},
  archiveprefix = {arXiv},
  primaryclass  = {cs.IR},
  url           = {https://arxiv.org/abs/2507.08191}
}

@misc{craswell2021overviewtrec2020deep,
  title         = {Overview of the TREC 2020 deep learning track},
  author        = {Nick Craswell and Bhaskar Mitra and Emine Yilmaz and Daniel Campos},
  year          = {2021},
  eprint        = {2102.07662},
  archiveprefix = {arXiv},
  primaryclass  = {cs.IR},
  url           = {https://arxiv.org/abs/2102.07662}
}

@misc{craswell2020overviewtrec2019deep,
  title         = {Overview of the TREC 2019 deep learning track},
  author        = {Nick Craswell and Bhaskar Mitra and Emine Yilmaz and Daniel Campos and Ellen M. Voorhees},
  year          = {2020},
  eprint        = {2003.07820},
  archiveprefix = {arXiv},
  primaryclass  = {cs.IR},
  url           = {https://arxiv.org/abs/2003.07820}
}

@inproceedings{mackie2021dlhard,
  title     = {How Deep is your Learning: the DL-HARD Annotated Deep Learning Dataset},
  author    = {Mackie, Iain and Dalton, Jeffrey and Yates, Andrew},
  booktitle = {Proceedings of the 44th International ACM SIGIR Conference on Research and Development in Information Retrieval},
  year      = {2021}
}

@inproceedings{tang2024found,
  title     = {Found in the middle: Permutation self-consistency improves listwise ranking in large language models},
  author    = {Tang, Raphael and Zhang, Crystina and Ma, Xueguang and Lin, Jimmy and T{\"u}re, Ferhan},
  booktitle = {Proceedings of the 2024 conference of the North American chapter of the Association for Computational Linguistics: human language technologies (volume 1: long papers)},
  pages     = {2327--2340},
  year      = {2024}
}

@inproceedings{wang2025realm,
  title     = {REALM: Recursive Relevance Modeling for LLM-based Document Re-Ranking},
  author    = {Wang, Pinhuan and Xia, Zhiqiu and Liao, Chunhua and Wang, Feiyi and Liu, Hang},
  booktitle = {Proceedings of the 2025 Conference on Empirical Methods in Natural Language Processing},
  pages     = {23875--23889},
  year      = {2025}
}

@book{brandt2016handbook,
  title     = {Handbook of computational social choice},
  author    = {Brandt, Felix and Conitzer, Vincent and Endriss, Ulle and Lang, J{\'e}r{\^o}me and Procaccia, Ariel D},
  year      = {2016},
  publisher = {Cambridge University Press}
}

@book{laslier1997tournament,
  title     = {Tournament solutions and majority voting},
  author    = {Laslier, Jean-Fran{\c{c}}ois},
  volume    = {7},
  year      = {1997},
  publisher = {Springer}
}

@article{alon2006ranking,
  title     = {Ranking tournaments},
  author    = {Alon, Noga},
  journal   = {SIAM Journal on Discrete Mathematics},
  volume    = {20},
  number    = {1},
  pages     = {137--142},
  year      = {2006},
  publisher = {SIAM}
}

@article{bartholdi1989voting,
  title     = {Voting schemes for which it can be difficult to tell who won the election},
  author    = {Bartholdi III, John and Tovey, Craig A and Trick, Michael A},
  journal   = {Social Choice and welfare},
  volume    = {6},
  number    = {2},
  pages     = {157--165},
  year      = {1989},
  publisher = {Springer}
}

@book{bang2008digraphs,
  title     = {Digraphs: theory, algorithms and applications},
  author    = {Bang-Jensen, J{\o}rgen and Gutin, Gregory Z},
  year      = {2008},
  publisher = {Springer Science \& Business Media}
}

@article{ren2021sample,
  title   = {Sample complexity bounds for active ranking from multi-wise comparisons},
  author  = {Ren, Wenbo and Liu, Jia and Shroff, Ness},
  journal = {Advances in Neural Information Processing Systems},
  volume  = {34},
  pages   = {4290--4300},
  year    = {2021}
}

@inproceedings{saha2019pac,
  title        = {Pac battling bandits in the plackett-luce model},
  author       = {Saha, Aadirupa and Gopalan, Aditya},
  booktitle    = {Algorithmic Learning Theory},
  pages        = {700--737},
  year         = {2019},
  organization = {PMLR}
}

@article{jang2017optimal,
  title   = {Optimal sample complexity of m-wise data for top-k ranking},
  author  = {Jang, Minje and Kim, Sunghyun and Suh, Changho and Oh, Sewoong},
  journal = {Advances in Neural Information Processing Systems},
  volume  = {30},
  year    = {2017}
}

@inproceedings{chen2018nearly,
  title        = {A nearly instance optimal algorithm for top-k ranking under the multinomial logit model},
  author       = {Chen, Xi and Li, Yuanzhi and Mao, Jieming},
  booktitle    = {Proceedings of the Twenty-Ninth Annual ACM-SIAM Symposium on Discrete Algorithms},
  pages        = {2504--2522},
  year         = {2018},
  organization = {SIAM}
}

@article{feige1994computing,
  title     = {Computing with noisy information},
  author    = {Feige, Uriel and Raghavan, Prabhakar and Peleg, David and Upfal, Eli},
  journal   = {SIAM Journal on Computing},
  volume    = {23},
  number    = {5},
  pages     = {1001--1018},
  year      = {1994},
  publisher = {SIAM}
}

@inproceedings{gu2023optimal,
  title     = {Optimal bounds for noisy sorting},
  author    = {Gu, Yuzhou and Xu, Yinzhan},
  booktitle = {Proceedings of the 55th Annual ACM Symposium on Theory of Computing},
  pages     = {1502--1515},
  year      = {2023}
}

@inproceedings{yoon2025acurank,
  title     = {AcuRank: Uncertainty-Aware Adaptive Computation for Listwise Reranking},
  author    = {Soyoung Yoon and Gyuwan Kim and GYU-HWUNG CHO and seung-won hwang},
  booktitle = {The Thirty-ninth Annual Conference on Neural Information Processing Systems},
  year      = {2025},
  url       = {https://openreview.net/forum?id=H918WyPf0s}
}

@inproceedings{sun2023chatgpt,
  title     = {Is {ChatGPT} Good at Search? Investigating Large Language Models as Re-Ranking Agents},
  author    = {Sun, Weiwei and Yan, Lingyong and Ma, Xinyu and Wang, Shuaiqiang and Ren, Pengjie and Chen, Zhumin and Yin, Dawei and Ren, Zhaochun},
  booktitle = {Proceedings of the 2023 Conference on Empirical Methods in Natural Language Processing},
  pages     = {14918--14937},
  year      = {2023},
  address   = {Singapore},
  publisher = {Association for Computational Linguistics},
  doi       = {10.18653/v1/2023.emnlp-main.923},
  url       = {https://aclanthology.org/2023.emnlp-main.923/}
}

@inproceedings{zhuang2024setwise,
  title     = {A Setwise Approach for Effective and Highly Efficient Zero-shot Ranking with Large Language Models},
  author    = {Zhuang, Shengyao and Zhuang, Honglei and Koopman, Bevan and Zuccon, Guido},
  booktitle = {Proceedings of the 47th International ACM SIGIR Conference on Research and Development in Information Retrieval},
  series    = {SIGIR '24},
  pages     = {1974--1983},
  year      = {2024},
  address   = {Washington, DC, USA},
  publisher = {ACM},
  doi       = {10.1145/3626772.3657813},
  url       = {https://doi.org/10.1145/3626772.3657813}
}

@article{landau1953dominance,
  title   = {On dominance relations and the structure of animal societies. III. The condition for a score structure},
  author  = {Landau, Hyman G},
  journal = {Bull. Math. Biophys},
  volume  = {15},
  number  = {2},
  pages   = {143--148},
  year    = {1953}
}

@book{moon2015topics,
  title     = {Topics on tournaments in graph theory},
  author    = {Moon, John W},
  year      = {2015},
  publisher = {Courier Dover Publications}
}

@misc{contextualai2025rerank,
  title        = {Open-Sourcing Reranker v2},
  author       = {{Contextual AI}},
  howpublished = {\url{https://contextual.ai/blog/rerank-v2}},
  year         = {2025},
  note         = {Accessed: 2025-01-13}
}

@article{liu2024lost,
  title     = {Lost in the Middle: How Language Models Use Long Contexts},
  author    = {Liu, Nelson F. and Lin, Kevin and Hewitt, John and Paranjape, Ashwin and Bevilacqua, Michele and Petroni, Fabio and Liang, Percy},
  journal   = {Transactions of the Association for Computational Linguistics},
  volume    = {12},
  pages     = {157--173},
  year      = {2024},
  publisher = {MIT Press},
  doi       = {10.1162/tacl_a_00638},
  url       = {https://aclanthology.org/2024.tacl-1.9/}
}

@book{zhou2008practical,
  title     = {A Practical Guide to Quantitative Finance Interviews},
  author    = {Zhou, X. and Jiu, B.},
  isbn      = {9781435715752},
  url       = {https://books.google.com/books?id=RosxmAYFFosC},
  year      = {2008},
  publisher = {Lulu.com}
}

@article{yu2025elspr,
  title   = {Elspr: Evaluator llm training data self-purification on non-transitive preferences via tournament graph reconstruction},
  author  = {Yu, Yan and Liu, Yilun and He, Minggui and Tao, Shimin and Meng, Weibin and Yang, Xinhua and Zhang, Li and Ma, Hongxia and Li, Dengye and Wei, Daimeng and others},
  journal = {arXiv preprint arXiv:2505.17691},
  year    = {2025}
}

@inproceedings{qin2024large,
  title     = {Large language models are effective text rankers with pairwise ranking prompting},
  author    = {Qin, Zhen and Jagerman, Rolf and Hui, Kai and Zhuang, Honglei and Wu, Junru and Yan, Le and Shen, Jiaming and Liu, Tianqi and Liu, Jialu and Metzler, Donald and others},
  booktitle = {Findings of the Association for Computational Linguistics: NAACL 2024},
  pages     = {1504--1518},
  year      = {2024}
}

@article{pradeep2023rankvicuna,
  title   = {Rankvicuna: Zero-shot listwise document reranking with open-source large language models},
  author  = {Pradeep, Ronak and Sharifymoghaddam, Sahel and Lin, Jimmy},
  journal = {arXiv preprint arXiv:2309.15088},
  year    = {2023}
}

@article{pradeep2023rankzephyr,
  title   = {RankZephyr: Effective and Robust Zero-Shot Listwise Reranking is a Breeze!},
  author  = {Pradeep, Ronak and Sharifymoghaddam, Sahel and Lin, Jimmy},
  journal = {arXiv preprint arXiv:2312.02724},
  year    = {2023}
}

@article{zeng2024llm,
  title   = {LLM-RankFusion: Mitigating Intrinsic Inconsistency in LLM-based Ranking},
  author  = {Zeng, Yifan and Tendolkar, Ojas and Baartmans, Raymond and Wu, Qingyun and Chen, Lizhong and Wang, Huazheng},
  journal = {arXiv preprint arXiv:2406.00231},
  year    = {2024}
}

\newpage
\appendix
\onecolumn


\part{Additional Experimental Details}\label{sec:additional-experiments}

\begin{table*}[!ht]
\centering
\caption{Reranking quality (nDCG@10) and efficiency (input tokens per query in thousands).}
\label{tab:main-results}
\scriptsize
\begin{tabular}{lcccccccccccccccc}
\toprule
& \multicolumn{6}{c}{TREC-DL} & \multicolumn{8}{c}{BEIR} & \multicolumn{2}{c}{Avg.} \\
\cmidrule(lr){2-7} \cmidrule(lr){8-15} \cmidrule(lr){16-17}
Method & DL19 & DL20 & DL21 & DL22 & DL23 & DLHard & COVID & NFC & Signal & News & R04 & Touche & DBP & Scif & nDCG & Tok \\
\midrule
\multicolumn{17}{c}{\textit{No Reranking}} \\
\midrule
BM25 & 50.6 & 48.0 & 44.6 & 26.9 & 26.2 & 28.5 & 59.5 & 33.7 & 33.0 & 39.5 & 40.7 & 44.2 & 31.8 & 67.9 & 41.1 & 0 \\
\midrule
\multicolumn{17}{c}{\textit{GPT-4.1}} \\
\midrule
Pairwise & 74.8 & 72.3 & 73.0 & 52.6 & 49.3 & 40.0 & 83.1 & 39.8 & 33.3 & 51.4 & 67.1 & 35.4 & 45.0 & 80.0 & \textbf{56.9} & 315k \\
Setwise & 73.5 & 70.4 & 72.0 & 51.9 & 48.2 & 39.6 & 83.3 & 39.9 & 34.0 & 49.6 & 66.6 & 35.6 & 44.4 & 78.4 & 56.3 & 111k \\
TourRank & 72.6 & 69.4 & 71.2 & 51.8 & 49.3 & 37.6 & 84.3 & 40.1 & 31.5 & 50.3 & 66.3 & 32.2 & 44.5 & 77.6 & 55.6 & 57k \\
TourRank-R2 & 74.5 & 70.9 & 70.3 & 51.8 & 50.3 & 38.6 & 84.5 & 40.4 & 33.2 & 50.0 & 66.3 & 32.4 & 45.5 & 78.7 & 56.2 & 114k \\
SW & 74.0 & 70.8 & 70.5 & 51.4 & 49.5 & 37.4 & 82.7 & 40.9 & 34.3 & 51.5 & 66.7 & 36.9 & 45.9 & 79.0 & 56.5 & 54k \\
SW-R2 & 73.4 & 72.0 & 70.5 & 52.0 & 49.6 & 39.4 & 82.6 & 40.5 & 34.0 & 51.9 & 66.0 & 36.4 & 46.4 & 79.1 & \underline{56.7} & 109k \\
AcuRank & 74.2 & 70.5 & 70.5 & 52.1 & 49.9 & 38.2 & 83.5 & 40.4 & 32.3 & 50.0 & 66.8 & 33.0 & 45.6 & 79.4 & 56.2 & 67k \\
AcuRank-H & 73.7 & 70.8 & 71.1 & 51.5 & 51.0 & 37.9 & 83.5 & 40.5 & 32.6 & 51.6 & 67.0 & 32.9 & 45.9 & 80.5 & 56.5 & 116k \\
{\algshortname{}-k20} & 74.6 & 70.7 & 70.4 & 51.4 & 48.9 & 37.3 & 82.4 & 39.9 & 33.7 & 50.1 & 66.4 & 36.5 & 45.3 & 79.2 & 56.2 & 40k \\
{\algshortname{}-k10} & 73.6 & 72.4 & 71.3 & 52.0 & 50.2 & 37.7 & 83.8 & 40.3 & 33.0 & 50.3 & 66.8 & 37.4 & 45.6 & 79.4 & \underline{56.7} & 42k \\
\midrule
\multicolumn{17}{c}{\textit{Gemini-3-Flash}} \\
\midrule
Pairwise & 74.9 & 72.4 & 72.7 & 53.3 & 50.2 & 38.9 & 82.7 & 40.8 & 34.2 & 48.8 & 67.9 & 36.3 & 45.5 & 80.4 & \textbf{57.1} & 334k \\
Setwise & 73.9 & 73.1 & 72.6 & 52.8 & 49.5 & 38.9 & 82.5 & 40.7 & 33.0 & 47.6 & 67.4 & 37.9 & 44.7 & 81.6 & 56.9 & 119k \\
TourRank & 74.9 & 72.4 & 71.9 & 51.0 & 49.8 & 36.6 & 82.7 & 40.6 & 33.8 & 49.7 & 66.4 & 35.1 & 45.6 & 77.9 & 56.3 & 57k \\
TourRank-R2 & 73.0 & 72.1 & 72.5 & 51.4 & 50.5 & 38.5 & 82.7 & 40.5 & 33.5 & 49.6 & 66.7 & 35.5 & 46.0 & 78.3 & 56.5 & 113k \\
SW & 74.0 & 73.1 & 73.0 & 52.1 & 51.1 & 38.9 & 82.6 & 40.8 & 32.6 & 49.4 & 66.5 & 36.1 & 45.7 & 80.6 & 56.9 & 54k \\
SW-R2 & 73.3 & 71.6 & 72.7 & 50.6 & 51.5 & 38.1 & 81.0 & 40.1 & 32.3 & 47.0 & 65.9 & 37.7 & 46.3 & 80.1 & 56.3 & 109k \\
AcuRank & 73.8 & 71.6 & 72.6 & 51.3 & 51.0 & 37.5 & 81.9 & 40.8 & 32.3 & 48.3 & 67.2 & 35.0 & 45.6 & 80.9 & 56.4 & 71k \\
AcuRank-H & 73.9 & 71.7 & 71.9 & 51.9 & 51.6 & 38.6 & 82.5 & 41.3 & 34.1 & 49.0 & 66.9 & 35.1 & 46.0 & 81.3 & 56.8 & 138k \\
{\algshortname{}-k20} & 74.9 & 72.1 & 72.4 & 51.4 & 49.7 & 38.1 & 82.3 & 41.1 & 33.0 & 48.2 & 65.8 & 37.9 & 45.3 & 80.4 & 56.6 & 40k \\
{\algshortname{}-k10} & 74.8 & 72.0 & 73.0 & 50.8 & 50.9 & 39.4 & 82.3 & 41.0 & 33.9 & 47.4 & 66.6 & 39.6 & 45.9 & 79.9 & \underline{57.0} & 42k \\
\midrule
\multicolumn{17}{c}{\textit{GLM-4.7}} \\
\midrule
TourRank & 74.1 & 71.2 & 69.8 & 51.0 & 49.4 & 38.4 & 83.4 & 40.0 & 30.4 & 49.1 & 66.0 & 31.8 & 45.9 & 77.5 & 55.6 & 57k \\
TourRank-R2 & 74.8 & 71.0 & 71.4 & 52.3 & 49.2 & 38.7 & 83.3 & 40.4 & 31.9 & 49.0 & 66.4 & 32.6 & 46.1 & 78.2 & 56.1 & 113k \\
SW & 74.2 & 71.6 & 71.5 & 51.7 & 48.3 & 38.5 & 82.9 & 39.8 & 34.1 & 50.4 & 66.2 & 37.6 & 47.6 & 80.0 & \textbf{56.7} & 54k \\
SW-R2 & 74.5 & 71.9 & 71.5 & 52.3 & 49.4 & 37.1 & 82.9 & 40.0 & 33.6 & 50.4 & 65.9 & 34.9 & 47.2 & 81.0 & \underline{56.6} & 109k \\
AcuRank & 74.0 & 70.2 & 70.8 & 51.8 & 48.3 & 37.5 & 82.5 & 40.1 & 31.7 & 49.8 & 66.0 & 33.0 & 46.5 & 80.4 & 55.9 & 64k \\
AcuRank-H & 75.1 & 71.4 & 71.6 & 52.4 & 48.4 & 39.3 & 83.4 & 40.5 & 32.1 & 50.3 & 66.7 & 31.7 & 46.5 & 80.9 & 56.4 & 108k \\
{\algshortname{}-k20} & 75.7 & 71.2 & 70.4 & 51.0 & 47.7 & 38.5 & 83.3 & 39.9 & 33.3 & 49.9 & 65.1 & 35.8 & 46.5 & 79.9 & 56.3 & 40k \\
{\algshortname{}-k10} & 72.8 & 72.4 & 71.3 & 52.4 & 49.2 & 38.6 & 81.8 & 39.9 & 33.9 & 50.0 & 66.5 & 38.4 & 46.5 & 80.3 & \textbf{56.7} & 41k \\
\midrule
\multicolumn{17}{c}{\textit{Qwen3-235B-A22B-Instruct}} \\
\midrule
TourRank & 71.3 & 69.0 & 69.0 & 49.7 & 47.3 & 36.6 & 84.8 & 38.7 & 28.3 & 50.3 & 64.0 & 34.5 & 42.1 & 73.8 & 54.2 & 59k \\
TourRank-R2 & 72.8 & 71.2 & 69.0 & 51.5 & 48.3 & 37.4 & 83.6 & 39.8 & 30.7 & 49.5 & 65.6 & 33.0 & 43.0 & 75.3 & 55.1 & 119k \\
SW & 73.5 & 69.7 & 70.6 & 49.9 & 47.8 & 39.2 & 83.8 & 39.6 & 32.5 & 50.5 & 64.5 & 32.5 & 45.0 & 77.3 & \underline{55.5} & 57k \\
SW-R2 & 73.5 & 71.1 & 70.2 & 51.5 & 48.4 & 37.8 & 83.6 & 39.9 & 31.6 & 47.9 & 65.3 & 33.0 & 45.1 & 76.3 & 55.4 & 114k \\
AcuRank & 73.9 & 69.4 & 69.8 & 50.7 & 48.9 & 37.3 & 82.1 & 39.2 & 29.0 & 50.3 & 65.4 & 32.2 & 43.2 & 76.3 & 54.8 & 71k \\
AcuRank-H & 74.2 & 69.6 & 70.2 & 51.3 & 49.5 & 39.0 & 82.7 & 39.6 & 30.5 & 50.3 & 65.9 & 32.1 & 44.2 & 76.5 & 55.4 & 128k \\
{\algshortname{}-k20} & 71.9 & 70.6 & 69.2 & 50.5 & 48.6 & 39.1 & 83.2 & 38.9 & 30.4 & 49.2 & 63.8 & 32.7 & 43.3 & 74.5 & 54.7 & 41k \\
{\algshortname{}-k10} & 74.1 & 71.1 & 70.3 & 49.0 & 47.3 & 39.7 & 83.1 & 39.8 & 31.9 & 50.9 & 65.3 & 35.9 & 43.5 & 75.9 & \textbf{55.6} & 43k \\
\midrule
\multicolumn{17}{c}{\textit{DeepSeek-V3.2}} \\
\midrule
TourRank & 72.7 & 71.2 & 69.5 & 49.6 & 47.8 & 38.8 & 83.3 & 40.0 & 31.4 & 50.0 & 65.2 & 33.5 & 45.0 & 79.6 & 55.5 & 56k \\
TourRank-R2 & 74.5 & 70.9 & 71.0 & 51.5 & 48.6 & 38.2 & 83.9 & 40.9 & 32.8 & 52.7 & 66.1 & 32.7 & 46.1 & 79.6 & 56.4 & 112k \\
SW & 74.6 & 71.4 & 71.2 & 52.5 & 47.5 & 40.8 & 84.3 & 40.7 & 34.1 & 51.4 & 65.6 & 40.0 & 47.2 & 80.7 & \textbf{57.3} & 54k \\
SW-R2 & 74.9 & 72.0 & 70.5 & 52.2 & 48.1 & 38.5 & 84.5 & 40.3 & 33.7 & 52.4 & 65.4 & 36.0 & 47.2 & 80.5 & \underline{56.9} & 108k \\
AcuRank & 73.6 & 71.2 & 70.9 & 52.1 & 48.5 & 37.2 & 83.5 & 40.5 & 31.0 & 53.1 & 66.3 & 31.3 & 45.6 & 79.3 & 56.0 & 69k \\
AcuRank-H & 74.5 & 71.4 & 71.3 & 52.4 & 48.8 & 39.4 & 84.5 & 41.2 & 32.2 & 52.2 & 66.8 & 32.6 & 46.4 & 80.9 & 56.8 & 131k \\
{\algshortname{}-k20} & 74.2 & 71.6 & 69.6 & 52.2 & 47.4 & 39.1 & 82.4 & 40.2 & 32.9 & 52.4 & 64.3 & 35.1 & 45.4 & 80.0 & 56.2 & 39k \\
{\algshortname{}-k10} & 74.0 & 71.6 & 70.9 & 52.3 & 47.1 & 40.6 & 84.6 & 40.6 & 32.1 & 50.3 & 65.4 & 37.0 & 46.2 & 80.1 & 56.6 & 41k \\
\bottomrule
\end{tabular}
\end{table*}

\section{Baseline Comparison Rationale}
\label{sec:baseline-details}

Each baseline tests a specific aspect of our tournament graph framework.

\paragraph{Sliding Windows \citep{sun2023chatgpt}} 
This baseline directly tests our central hypothesis: sliding windows process overlapping document sets, yet discard comparison information from previous windows rather than accumulating it. Our framework captures these relationships via transitive closure, potentially achieving equivalent quality with fewer queries.

\paragraph{TourRank \citep{chen2025tourrank}}
While TourRank uses a static tournament structure to schedule matches, it does not employ tournament graphs to model comparison outcomes or exploit graph-theoretic properties such as transitivity. This comparison isolates the value of maintaining a tournament graph and propagating information via transitive closure versus using tournaments purely as a scheduling mechanism.

\paragraph{AcuRank \citep{yoon2025acurank}.}
Both AcuRank and \tournamentsort{} are adaptive -- terminating when sufficiently confident about the top-$m$ -- but employ different certification criteria: AcuRank uses Bayesian score distributions, while our framework certifies via graph-theoretic finalization through transitive closure.

\paragraph{Setwise \citep{zhuang2024setwise}.}
This baseline provides the most direct contrast to our approach: both methods issue $k$-wise comparison queries, but Setwise extracts only the winner -- a single comparison -- while our framework captures the complete tournament of $\binom{k}{2}$ pairwise relationships.

\paragraph{Pairwise \citep{qin2024large}.}
Pairwise reranking serves as an upper bound on comparison granularity: each query yields exactly one pairwise relationship. This makes it highly accurate -- reducing the cognitive load on LLMs to simple binary judgments -- but computationally expensive, requiring 7--8$\times$ more tokens than our approach.

\section{Efficiency Comparison Details}
\label{sec:efficiency-details}

\begin{figure*}[t]
    \centering
    \begin{subfigure}[b]{0.48\textwidth}
        \centering
        \includegraphics[width=\textwidth]{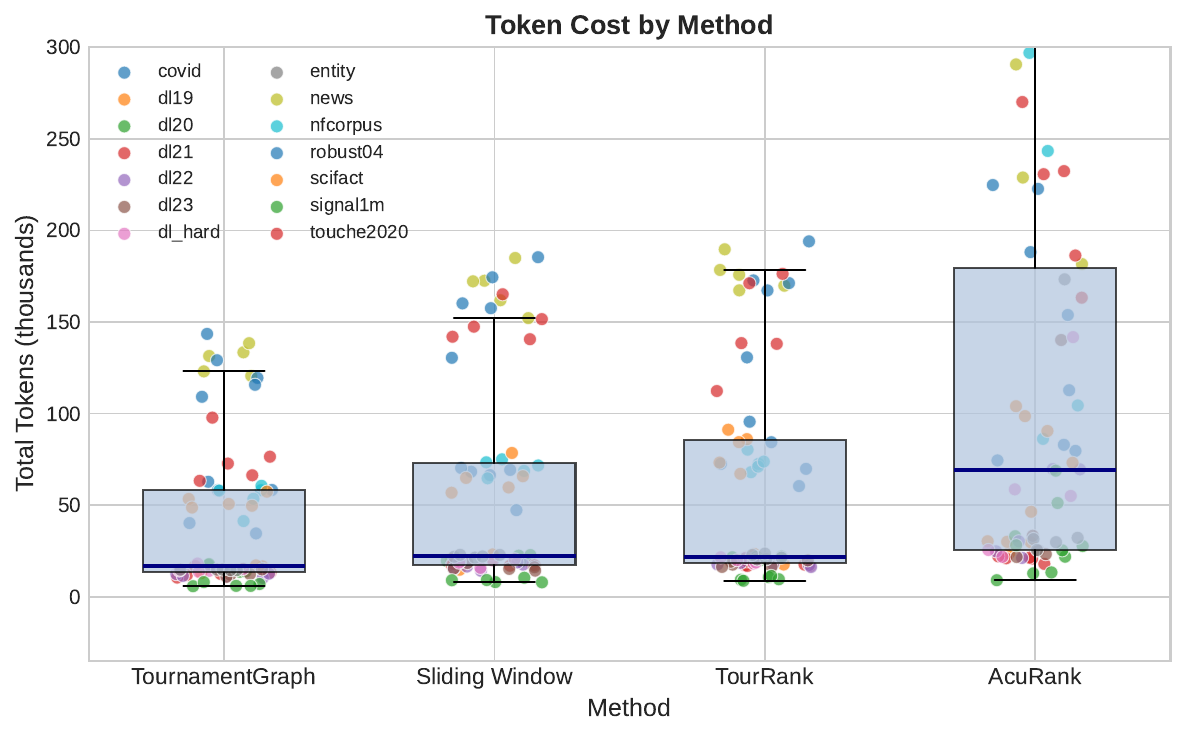}
        \caption{Token cost per query}
        \label{fig:method-cost-tokens}
    \end{subfigure}
    \hfill
    \begin{subfigure}[b]{0.48\textwidth}
        \centering
        \includegraphics[width=\textwidth]{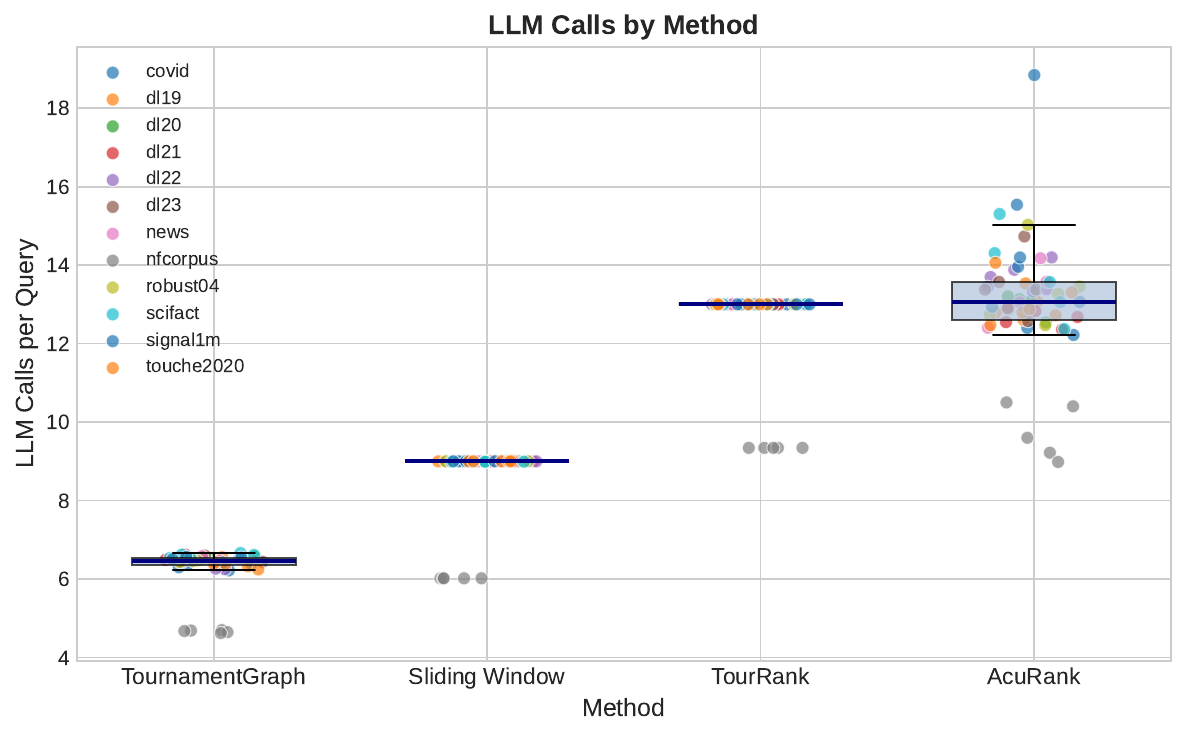}
        \caption{LLM calls per query}
        \label{fig:method-cost-calls}
    \end{subfigure}
    \caption{Efficiency comparison across listwise reranking methods on 14 datasets (GPT-4.1). Each point represents one dataset; box plots show the distribution. \textbf{(a)} Token consumption varies across datasets due to differing document lengths, but \tournamentsort{} consistently achieves the lowest median cost. AcuRank exhibits the highest variance due to its adaptive computation. \textbf{(b)} LLM call counts isolate algorithmic efficiency from document length effects. \tournamentsort{} requires $\sim$6.5 calls on average -- fewer than Sliding Window (9), TourRank (13), and AcuRank (13). Outliers at lower call counts correspond to NFCorpus, where many queries have fewer than 100 retrieved documents (see Figure~\ref{fig:nfcorpus-dist}).}
    \label{fig:method-cost}
\end{figure*}

\begin{figure}[t]
    \centering
    \includegraphics[width=0.5\columnwidth]{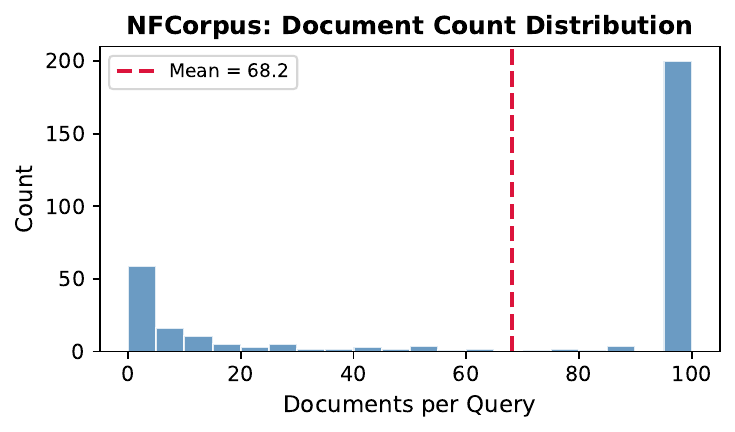}
    \caption{Distribution of retrieved documents per query in NFCorpus. Unlike other datasets where BM25 retrieves 100 documents per query, NFCorpus exhibits a bimodal distribution with many queries having fewer than 20 candidates (mean: 68.2). This explains the lower LLM call counts for NFCorpus in Table~\ref{tab:method-llm-calls}.}
    \label{fig:nfcorpus-dist}
\end{figure}

\begin{table*}[t]
\centering
\caption{LLM calls per query across reranking methods. Statistics shown for all 14 datasets and separately excluding NFCorpus (13 datasets). \tournamentsort makes the fewest calls ($\sim$6.5) while Sliding Window and TourRank are deterministic (9 and 13 calls respectively). AcuRank averages $\sim$13 calls with slight variation due to adaptive refinement.}
\label{tab:method-llm-calls}
\small
\begin{tabular}{l|rrrr|cc|rrrr|cc}
\toprule
& \multicolumn{6}{c|}{All Datasets} & \multicolumn{6}{c}{Excl. NFCorpus} \\
Method & Mean & Std & Min & Max & Std\% & Rng\% & Mean & Std & Min & Max & Std\% & Rng\% \\
\midrule
\tournamentsort & 6.3 & 0.6 & 5 & 7 & 9.0 & 32.0 & 6.5 & 0.1 & 6 & 7 & 2.0 & 7.0 \\
Sliding Window & 8.8 & 0.8 & 6 & 9 & 9.0 & 34.0 & 9.0 & 0.0 & 9 & 9 & 0.0 & 0.0 \\
TourRank & 12.7 & 1.0 & 9 & 13 & 8.0 & 29.0 & 13.0 & 0.0 & 13 & 13 & 0.0 & 0.0 \\
AcuRank & 13.1 & 1.4 & 9 & 19 & 11.0 & 75.0 & 13.4 & 1.1 & 12 & 19 & 8.0 & 49.0 \\
\bottomrule
\end{tabular}
\end{table*}

Setwise and Pairwise consume substantially more tokens than \tournamentsort{} due to their small comparison windows ($k{=}3$ and $k{=}2$, respectively), requiring many more queries to establish order.
More interesting is the comparison against methods that share our oracle structure -- AcuRank, TourRank, and Sliding Window all prompt the LLM to rank $k{=}20$ documents per call.
Despite this structural similarity, \tournamentsort{} achieves lower cost through principled information extraction.

Figure~\ref{fig:method-cost} visualizes token consumption and LLM call counts across the 14 datasets.
Token costs vary across datasets due to differing document lengths, but \tournamentsort consistently achieves the lowest median (Figure~\ref{fig:method-cost-tokens}).
To isolate algorithmic efficiency from document length effects, we also compare LLM call counts (Figure~\ref{fig:method-cost-calls}).
\tournamentsort requires only $\sim$6.5 calls per query on average, compared to 9 for Sliding Window, 13 for TourRank, and 13 for AcuRank (Table~\ref{tab:method-llm-calls}).

The outliers at lower call counts in Figure~\ref{fig:method-cost-calls} correspond to NFCorpus, where many queries have fewer than 100 retrieved documents.
Figure~\ref{fig:nfcorpus-dist} shows that NFCorpus exhibits a bimodal distribution with many queries having fewer than 20 candidates (mean: 68.2), explaining why all methods require fewer calls on this dataset.
We report statistics both with and without NFCorpus in Table~\ref{tab:method-llm-calls} to isolate this effect.

\paragraph{Comparison with AcuRank.}
AcuRank exhibits both higher cost and less predictable cost than \tournamentsort{}.
While \tournamentsort has a standard deviation of just 2\% in call counts (excluding NFCorpus), AcuRank's standard deviation reaches 8\%, with a range spanning 49\% of its mean (Table~\ref{tab:method-llm-calls}).
This variance stems from AcuRank's design: it models each document with an independent score distribution and adaptively allocates comparisons based on uncertainty.
In contrast, our framework maximizes information extraction by capturing complete tournaments from each query and propagating relationships via transitive closure, yielding deterministic convergence behavior.

AcuRank's adaptivity does offer a potential advantage: spending more computation on genuinely difficult queries.
However, modern LLMs produce largely consistent judgments, limiting the practical benefit of this flexibility.
Whether combining adaptive allocation with tournament graph structure can yield further gains remains an open question.

\section{Convergence Analysis}
\label{sec:convergence-details}

\begin{table}[t]
\centering
\small
\caption{Number of rounds as a function of window size $k$, aggregated over 70 configurations (14 datasets $\times$ 5 LLMs). The algorithm exhibits deterministic convergence: for fixed $k$, round count remains nearly constant across datasets and oracles, enabling reliable cost estimation.}
\label{tab:rounds}
\begin{tabular}{c|cccc|c}
\toprule
$k$ & Min & Max & Mean & Std & Avg.\ Tokens \\
\midrule
10 & 12 & 15 & 13.59 & 0.58 & 44,814 \\
20 & 6 & 7 & 6.73 & 0.45 & 42,139 \\
\bottomrule
\end{tabular}
\end{table}

\begin{figure*}[t]
    \centering
    \includegraphics[width=0.9\textwidth]{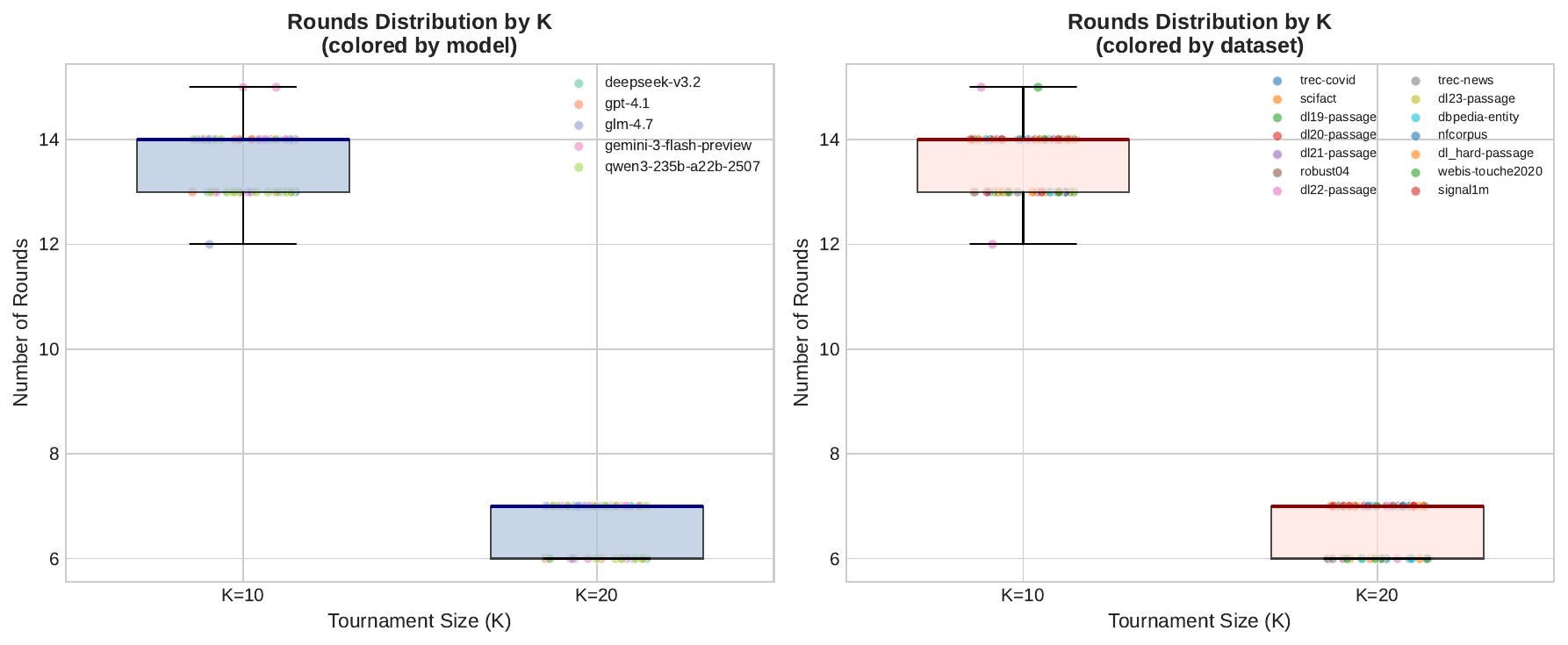}
    \caption{Distribution of rounds until convergence for $k{=}10$ and $k{=}20$, aggregated across 70 configurations (14 datasets $\times$ 5 LLMs). Left panel colors by model, right panel by dataset. The algorithm exhibits highly predictable convergence: $k{=}10$ requires 12--15 rounds (mean 13.6), while $k{=}20$ requires only 6--7 rounds (mean 6.7). Variance across models and datasets is minimal, indicating that convergence depends primarily on problem structure rather than oracle-specific factors.}
    \label{fig:rounds-by-k}
\end{figure*}

Table~\ref{tab:rounds} and Figure~\ref{fig:rounds-by-k} characterize the algorithm's convergence behavior across window sizes.
For $k{=}10$, convergence requires 12--15 rounds (mean 13.6, std 0.58); for $k{=}20$, only 6--7 rounds (mean 6.7, std 0.45).

This predictability is notable: despite variation in dataset characteristics, query difficulty, and oracle capabilities, round count remains nearly constant for fixed $k$.
The deterministic convergence stems from the progress guarantee (Lemma~\ref{lem:tied-unknown-edge-general}): each round reveals at least one new edge, and the algorithm terminates once the top-$m$ candidates are finalized.

Token consumption also remains stable across oracles despite differences in their ranking capabilities, indicating that the algorithm's query complexity depends primarily on the problem structure ($n$, $m$, $k$) rather than oracle-specific factors.
This contrasts sharply with adaptive methods like AcuRank, which exhibit 8\% variance in call counts -- the deterministic nature of our framework enables reliable cost estimation before execution.

\section{SCC Analysis}

\begin{table*}[ht]
\centering
\caption{Average SCC size by rank position on DL19 using GPT-4.1. SCCs concentrate in mid-ranks where the algorithm focuses queries, peaking around rank 15 for $k{=}20$ (avg 8.86) and rank 5--10 for $k{=}10$ (avg $\sim$4). Beyond rank 50, SCC size drops to 1.0 as lower-ranked documents are excluded via transitive closure.}
\label{tab:scc-size-by-rank}
\scriptsize


\begin{tabular}{rccccccccc}
\toprule
Rank & $1$ & $2$ & $3$ & $5$ & $10$ & $15$ & $20$ & $25$ & $50$ \\
\midrule
$K=10$ & $1.00$ & $2.09$ & $2.44$ & $2.93$ & $3.88$ & $2.74$ & $1.23$ & $1.02$ & $1.00$ \\
$K=20$ & $1.35$ & $1.77$ & $2.98$ & $3.60$ & $6.02$ & $8.86$ & $8.30$ & $5.98$ & $1.00$ \\
\bottomrule
\end{tabular}
\end{table*}

\textbf{SCCs concentrate in mid-ranks, not at the tail.}
One might expect SCCs to cluster among low-ranked documents, where items are uniformly irrelevant and difficult to distinguish.
Instead, Figure~\ref{fig:scc-size-by-rank} shows a wave pattern: SCC size peaks around ranks 10--20 for $k{=}20$ and ranks 5--10 for $k{=}10$, then drops to 1.0 beyond rank 50.
This pattern reflects the algorithm's greedy scheduling (\S\ref{sec:algorithm}): queries target non-finalized SCCs with smallest in-reach, concentrating comparisons near the decision boundary.
Lower-ranked documents are often excluded via transitive closure after a single comparison -- losing to a document that loses to a top-$m$ candidate finalizes their position without further queries.

\section{Scheduling Ablation}
\label{sec:scheduling-ablation}

\begin{table}[ht]
\centering
\caption{Scheduling ablation on synthetic transitive instances ($m=n$). Oracle calls (mean $\pm$ std over 20 seeds). Ratio = Random / Greedy.}
\label{tab:scheduling-ablation}
\small
\begin{tabular}{lrcrcr}
\toprule
Config & Greedy & Random (cands.) & Ratio & Random (all) & Ratio \\
\midrule
$n{=}100,\; k{=}10$ & $37.7 \pm 1.2$ & $132.1 \pm 9.6$ & $3.5\times$ & $541.5 \pm 102.4$ & $14.4\times$ \\
$n{=}100,\; k{=}20$ & $14.3 \pm 0.6$ & $35.9 \pm 4.3$ & $2.5\times$ & $123.5 \pm 26.3$ & $8.6\times$ \\
$n{=}200,\; k{=}10$ & $89.2 \pm 1.6$ & $520.9 \pm 32.1$ & $5.8\times$ & $2554.6 \pm 432.4$ & $28.6\times$ \\
$n{=}200,\; k{=}20$ & $33.8 \pm 0.7$ & $129.4 \pm 8.9$ & $3.8\times$ & $653.5 \pm 184.5$ & $19.3\times$ \\
\bottomrule
\end{tabular}
\end{table}

To isolate the contribution of greedy SCC scheduling (\S\ref{sec:algorithm}), we compare three scheduling strategies on synthetic transitive instances where oracle calls can be counted exactly.
We run \tournamentsort{} with the default greedy schedule and two random baselines: (1)~\emph{Random (cands.)}, which selects a random $k$-subset from the current candidate set, and (2)~\emph{Random (all)}, which selects a random $k$-subset from all available nodes.
Values in Table~\ref{tab:scheduling-ablation} are mean $\pm$ std of total oracle calls over 20 seeds.
The greedy schedule is $2.5$--$5.8\times$ more efficient than random scheduling from the candidate set, and $8.6$--$28.6\times$ more efficient than random scheduling from all nodes.

\newpage
\part{Theory}\label{part:theory}

\noindent
The most direct proof of Algorithm~\ref{alg:tournament-sort-main}'s
correctness and termination can be made short, as the proof sketch
in~\S\ref{sec:guarantees} demonstrates: output correctness follows
from a single ordering lemma on resolved vertices, and termination
from the tied-SCC progress argument on the condensation DAG.
We nonetheless present the full development through the transitive
special case (Algorithm~\ref{alg:transitive-tournament-sort}), the
general non-transitive case
(Algorithm~\ref{alg:non-transitive-tournament-graph-sort}), and
finally the resolution-based variant
(Algorithm~\ref{alg:tournament-sort-main}) for two reasons.
First, the progression builds intuition for why the algorithm works
and how it was designed: the transitive case isolates the finalization
engine on DAGs, and the non-transitive case shows how condensation
lifts it to arbitrary tournaments, before the resolution criterion
simplifies the stopping rule.
Second, the intermediate machinery -- particularly the finalization
threshold, the rank spectrum, the condensation projection, and the
SCC refinement framework -- provides the essential tools for two
concrete open problems identified in~\S\ref{sec:future-work}: proving query complexity
bounds for general~$m$ (Conjecture~\ref{conj:query-complexity},
which would likely begin with the transitive case and lift via
condensation) and extending the framework to noisy oracles (which
requires replacing the subgraph invariant with probabilistic
analogues of the finalization and resolution criteria).

\section{Problem Statement}\label{sec:Problem-Statement}

We consider the problem of identifying top-ranked vertices in an unknown
tournament graph through queries to an oracle.

\subsection{Unknown Tournaments and Oracle Model}

\paragraph{Tournament Graphs.}

A \emph{tournament} is a directed graph $G^{*}=(V,E^{*})$ in which
every pair of vertices is connected by exactly one directed edge: for
all $u\neq v\in V$, either $(u,v)\in E^{*}$ or $(v,u)\in E^{*}$,
but not both. The edge $(u,v)$ is interpreted as $u$ defeating $v$
in a comparison. Throughout, $G^{*}$ denotes a fixed but unknown tournament
that we seek to discover through oracle queries. We use the subscript
and superscript~$*$ to indicate quantities associated with the
tournament $G^{*}$.

\paragraph{Neighborhoods and Degrees.}

For a vertex $u$ in a directed graph $G=(V,E)$, the
\emph{out-neighbors} of $u$ is
$N_{G}^{+}(u)=\{v\in V:(u,v)\in E\}$ with \emph{out-degree}
$\outdeg_{G}(u)=|N_{G}^{+}(u)|$, and the \emph{in-neighbors} of $u$
is $N_{G}^{-}(u)=\{v\in V:(v,u)\in E\}$ with \emph{in-degree}
$\indeg_{G}(u)=|N_{G}^{-}(u)|$.

\paragraph{Oracle Model.}

For a set $U$ with $|U|=n$ and a positive integer $k$, let
$\binom{U}{k}:=\{S\subseteq U:|S|=k\}$ denote the collection of all
$k$-subsets of $U$, and let
$\perm Uk:=\{(u_{1},\ldots,u_{k})\in U^{k}:u_{i}\neq u_{j}\text{
  for all }i\neq j\}$ denote the set of all ordered $k$-tuples of
distinct elements from $U$. We assume access to an oracle
$O_{G^{*}}:\binom{V}{k}\to2^{E^{*}}$ that, given any $k$-subset
$S\in\binom{V}{k}$, returns all directed edges of $G^{*}$ induced by
$S$:
\[
  O_{G^{*}}(S):=E^{*}\cap\perm S2.
\]
Intuitively, given $k$ items, the oracle ``organizes'' a
subtournament between them and reveals the result.

\paragraph{Reachability.}

We write $u\leadsto_{G}v$ to denote the existence of a directed path
from $u$ to $v$ in $G$, i.e.\ $u\leadsto_{G}v$ if there exist
vertices $v_{0},v_{1},\dots,v_{k}\in V$ with $k\geq0$, $v_{0}=u$,
$v_{k}=v$, and $(v_{i},v_{i+1})\in E$ for all
$i=0,1,\dots,k-1$. The \emph{in-reach} and \emph{out-reach} of a
vertex $v$ are defined as
\begin{align}
  \inreach_{G}(v)
    &:=\{u\in V\setminus\{v\}:u\leadsto_{G}v\},
    \label{eq:inreach-def}\\
  \outreach_{G}(v)
    &:=\{u\in V\setminus\{v\}:v\leadsto_{G}u\}.
    \label{eq:outreach-def}
\end{align}
In a tournament, $|\inreach_{G^{*}}(v)|$ measures how many items
transitively dominate $v$ -- a reachability-based loss count. When
$G^{*}$ is \emph{transitive} (i.e., $(u,v)$ and $(v, p)$ in $E^*$ implies $(u, p)$ in $E^*$), 
in-reach coincides with in-neighborhood
(Corollary~\ref{cor:inreach-equals-indeg}), so
$|\inreach_{G^{*}}(v)|=\indeg_{G^{*}}(v)$.

\subsection{Problem Statement: Ranking in Tournaments with $k$-wise
  Oracle}

We formalize the primary ranking problem of interest as follows.

\begin{problem}
\label{prob:main-inreach} Given an unlabeled vertex set $V$ of size
$n$, an oracle $O_{G^{*}}$ with query size $k\geq2$, and a target
count $m\leq n$, identify a set of top-$m$ ranked vertices
$S=\{v^{(1)},\ldots,v^{(m)}\}$ such that
\begin{align}
  & |\inreach_{G^{*}}(v^{(1)})|
    \leq|\inreach_{G^{*}}(v^{(2)})|
    \leq\cdots
    \leq|\inreach_{G^{*}}(v^{(m)})|
    \label{eq:sol-set-satisfies-internal-ordering}\\
  \text{and }\quad
  & \max_{i\leq m}|\inreach_{G^{*}}(v^{(i)})|
    \leq\min_{u\notin S}|\inreach_{G^{*}}(u)|,
    \label{eq:rank-dominance-of-sol-set}
\end{align}
using as few oracle queries as possible.
\end{problem}

\begin{rem}[In-reach ranking and the decomposition ordering]
\label{rem:inreach-ranking}
Every tournament decomposes uniquely into strongly connected components
(SCCs), and the condensation -- the DAG obtained by collapsing each SCC
to a single node -- is itself a transitive
tournament~\citep{bang2008digraphs}. Vertices in the same SCC have
identical in-reach, so the in-reach ranking coincides with the
\emph{decomposition ordering} (also called the \emph{Smith set
  stratification})~\citep{laslier1997tournament}: it ranks vertices by
the position of their SCC in the condensation, declaring vertices
within the same SCC tied. When $G^{*}$ is transitive, every SCC is a
singleton, in-reach reduces to in-degree, and the ranking is a total
order identical to the Copeland ranking. In the non-transitive case,
ties may cause the top-$m$ set (and its internal ordering) to be
non-unique. Refining the within-SCC ordering to match the full
Copeland ranking requires additionally discovering the induced
subtournament on each output SCC, at a cost of at most
$\binom{|C|}{2}$ further queries per component~$C$.
\end{rem}

\subsection{Theoretical Foundations and Related
  Work}\label{sec:theoretical-related-work}

  \paragraph{Tournament solutions and query complexity.}

  Tournaments are well-studied in computational social
  choice~\citep{brandt2016handbook,laslier1997tournament}. Beyond the
  Copeland and decomposition orderings used here, classical solutions
  include the Slater ranking (minimum feedback arc set) and Kemeny
  ranking (minimum Kendall-tau distance); both are
  NP-hard~\citep{alon2006ranking,bartholdi1989voting}.
  \citet{dey2017query} established query complexity bounds for
  tournament solutions under a pairwise oracle, proving that most
  require $\Omega(n^{2})$ queries. When the top cycle has bounded
  size~$c$, they show all solutions can be found in
  $O(nc+n\log n/\log(1-1/c))$ queries; however, this identifies only
  the first SCC (the top cycle) and does not recover the full
  condensation ordering over all components. No existing work, to our
  knowledge, studies the query complexity of recovering the complete SCC
  stratification.
  
  The decomposition ordering is a natural target for our framework for
  two reasons. First, the finalization engine
  (introduced in \S\ref{par:two-layer-architecture} and detailed in \S\ref{sec:graph-prelims}) determines SCC positions without
  resolving internal structure, so the algorithm terminates as soon as
  tier membership is established. Second, it gives the honest answer
  when preferences are cyclic. In our primary
  application -- LLM-as-a-judge reranking -- cycles arise not from noise
  but from genuine ambiguity: the oracle cannot consistently order those
  items, and reporting them as tied is more faithful than imposing an
  artificial distinction. When a total ordering is needed, ties may be
  broken by a secondary signal (e.g., the original retrieval score) or additional fine-grained queries.
  
  \paragraph{Multi-wise comparisons.}
  
  The information-theoretic advantage of $k$-wise comparisons is
  established in \citet{ren2021sample}, who prove that full-ranking
  feedback improves sample complexity by a factor of $k\log k$ over
  winner feedback for ranking recovery. \citet{saha2019pac} show similar
  gains under the Plackett--Luce model. \citet{jang2017optimal} and
  \citet{chen2018nearly} provide near-optimal algorithms for top-$m$
  selection under parametric models. For noisy comparisons,
  \citet{feige1994computing} established $O(n\log n)$ bounds for
  sorting, tightened by \citet{gu2023optimal}. Our framework operates
  non-parametrically, using graph structure rather than score estimation,
  and addresses the query-efficient top-$m$ selection problem that
  arises when each comparison is expensive. More specifically, our
  setting differs from prior work in three respects: (i)~we assume a
  deterministic tournament rather than a stochastic comparison model,
  (ii)~we use a $k$-wise oracle ($k \geq 2$) rather than restricting to
  pairwise queries, and (iii)~we target the SCC decomposition ordering
  rather than a total ranking or a single tournament solution set. To
  our knowledge, no existing algorithm addresses this combination.

\subsection{Technical Overview and Roadmap}

The appendix provides full proofs of correctness and termination for
the algorithms introduced in the main text. The central question
throughout is: \emph{given only the edges revealed so far, when can
we certify that the current top-$m$ vertices are correct with respect
to the unknown tournament $G^{*}$?} We answer this first for
transitive tournaments, then lift the answer to arbitrary tournaments.

\paragraph{Two-layer architecture.}\label{par:two-layer-architecture}

The proofs separate cleanly into two layers, and keeping this
separation in mind is a useful guide for reading the
appendix.

The first layer is a \emph{finalization engine} that works on any
directed graph and is completely agnostic to the existence of an
underlying tournament. It takes a partially revealed graph, sorts
vertices by in-reach into a rank spectrum, and identifies a prefix
of vertices whose relative positions are already forced by
reachability alone. Informally, this layer answers: \emph{given what
we have seen so far, which vertices can we already certify as
top-ranked?} The most interesting pieces apply to DAGs (subgraphs of transitive tournaments).  
In the general case, we will utilize condensation machinery (c.f. \S\ref{subsec:scc-and-condensation}) to reduce subgraphs of non-transitive 
tournaments to DAGs to lift the finalization engine to general directed graphs.

The second layer is a \emph{tournament bridge} that connects the
finalization engine back to the ground truth. It uses properties
specific to the underlying tournament to show that whenever the
engine certifies a prefix in the revealed graph, the same prefix is
correct in the full tournament. Informally: \emph{why does
certification in the subgraph imply correctness in $G^{*}$?}

\paragraph{Transitive versus non-transitive case.}

In the transitive case, the revealed graph $G_{t}$ at a round $t$  is always a DAG
(every subgraph of a transitive tournament is acyclic), so the
finalization engine is applied directly to $G_{t}$ and the bridge
bounds discovered in-reach by true in-degree.

In the general case, $G_{t}$ may contain cycles, so the engine
cannot be applied to $G_{t}$ directly. Instead, we move to the
condensation $\condense{G_{t}}$, which collapses cycles into tied
tiers and is always a DAG. The finalization engine is reused without
modification on $\condense{G_{t}}$, but the bridge must be rebuilt:
a projection map induced by SCC refinement relates
$\condense{G_{t}}$ to $\condense{G^{*}}$ and transfers rank
certificates from the discovered graph to the true tournament.

\paragraph{Roadmap.}

The appendix proceeds as follows.
\begin{itemize}
\item \textbf{Section~\ref{sec:graph-prelims}} develops the reusable
  graph-theoretic machinery: the rank spectrum, known relationships,
  condensation, and the finalization engine. Nothing here assumes a
  tournament.
\item \textbf{Section~\ref{sec:transitive-analysis}} builds the
  tournament bridge for transitive tournaments, derives the greedy
  query schedule, and proves correctness and termination of the
  transitive algorithm.
\item \textbf{Section~\ref{sec:Non-Transitive-Tournaments}} lifts the
  same engine to condensation graphs, rebuilds the bridge via SCC
  refinement and projection, and proves correctness and termination
  for arbitrary tournaments.
\item \textbf{Section~\ref{sec:stronger-finalization-general}}
  introduces the resolution-based stopping rule used by the practical
  algorithm (\tournamentsort{} -- Algorithm~\ref{alg:tournament-sort-main}) and proves its correctness and
  termination.
\item \textbf{Section~\ref{sec:query-complexity}} discusses query
  complexity bounds.
\end{itemize}
Readers interested primarily in the algorithm may begin with
Section~\ref{sec:transitive-analysis} and return to the
non-transitive extension afterward; Section~\ref{sec:graph-prelims}
serves as a reference for both.
\section{Finalization Framework}\label{sec:graph-prelims}
This section builds the graph-theoretic objects that constitute the
finalization engine described in the roadmap
(\S\ref{par:two-layer-architecture}). Nearly everything here is defined for arbitrary directed graphs; the
few results that specialize to tournaments are marked. The
section is organized in four parts:
\begin{itemize}
\item \textbf{Reachability framework}
  (\S\ref{subsec:rank-spectrum}): transitive closure, in-reach, and
  the rank spectrum.
\item \textbf{Condensation}
  (\S\ref{subsec:scc-and-condensation}): strongly connected
  components, the condensation graph, and their interaction with
  in-reach.
\item \textbf{Known relationships}
  (\S\ref{subsec:Known-Relationships}): the resolved-vertex notion
  used by the practical algorithm's stopping rule.
\item \textbf{Finalization engine}
  (\S\ref{subsec:structural-properties-finalization}): the
  finalization threshold, finalized set, and their structural
  properties.
\end{itemize}

\subsection{Reachability Framework}

\paragraph{Transitive closure.}

Given a graph $G=(V,E)$, its transitive closure $G^{+}=(V,E^{+})$
is constructed by adding an edge $(u,v)$ to $E^{+}$ if $u\leadsto_{G}v$
i.e., 
\[
E^{+}:=\{(u,v)\in V\times V\mid u\neq v\text{ and }u\leadsto_{G}v\}.
\]

Some observations: 
\begin{enumerate}
\item Under transitive closure, in-reach and out-reach reduce to in-neighborhood
and out-neighborhood, respectively. 
\item If $G$ is a DAG, then $\inreach_{G}(v)\cap\outreach_{G}(v)=\emptyset$.
Otherwise, the intersection is not necessarily empty if $G$ is not
a DAG. 
\end{enumerate}
\begin{lem}[Downward Closure of In-Reach]
\label{lem:inreach-downward-closure} Let $G=(V,E)$ be a directed
graph. If $x\in\inreach_{G}(y)$, then $\inreach_{G}(x)\cup\left\{ x\right\} \subseteq\inreach_{G}(y)$. 
\end{lem}

\begin{proof}
Let $z\in\inreach_{G}(x)$. By definition, $z\leadsto_{G}x$. Since
$x\in\inreach_{G}(y)$, we have $x\leadsto_{G}y$. Concatenating these
paths gives $z\leadsto_{G}y$, so $z\in\inreach_{G}(y)$. 
\end{proof}
The in-reach has useful properties when the graph $G$ under consideration
is a directed acyclic graph (DAG).
\begin{lem}[Strict Decrease of In-Reach Along Edges]
\label{lem:inreach-strict-decrease} Let $G=(V,E)$ be a DAG. If
$x\in\inreach_{G}(y)$, then $|\inreach_{G}(x)|<|\inreach_{G}(y)|$. 
\end{lem}

\begin{proof}
By Lemma~\ref{lem:inreach-downward-closure}, $\inreach_{G}(x)\subseteq\inreach_{G}(y)$.
Since $x\in\inreach_{G}(y)$ but $x\notin\inreach_{G}(x)$ and $y\not\in\inreach_{G}(x)$
(due to otherwise being cyclic), the inclusion is strict. 
\end{proof}

\subsubsection{Rank Spectrum }\label{subsec:rank-spectrum}

We now define the central objects of our framework.
By analogy with the singular value decomposition, where the singular
values $\sigma_{1}(A)\geq\sigma_{2}(A)\geq\cdots$ are uniquely determined
while their associated singular vectors may admit rotational freedom
within tied singular subspaces, we pair in-reach values with their
corresponding vertices. The rank spectrum $R(G)$ is uniquely determined
regardless of tie-breaking, while the rank basis $V(G)$ is unique
only up to permutation of vertices sharing the same in-reach.
\begin{defn}[Rank Spectrum]
\label{def:rank-spectrum} Let $G=(V,E)$ be a directed graph on
$n$ vertices. An \emph{in-reach ordering} of $G$ is a sequence $v^{(1)}(G),v^{(2)}(G),\ldots,v^{(n)}(G)$
of all vertices in $V$ satisfying 
\[
r^{(1)}(G)\leq r^{(2)}(G)\leq\cdots\leq r^{(n)}(G),\quad\text{where }r^{(i)}(G):=|\inreach_{G}(v^{(i)}(G))|.
\]
We call the sequence $R(G):=(r^{(1)}(G),\ldots,r^{(n)}(G))$ the \emph{rank
spectrum} of $G$. Similarly, we call the sequence $V(G)=(v^{(1)}(G),\ldots,v^{(n)}(G))$
the \emph{rank basis} of $G$. When ties exist, the ordering among
tied vertices may be broken arbitrarily; the rank spectrum $R(G)$
is uniquely determined regardless of tie-breaking. 
\end{defn}

The rank spectrum forms crucial connections from subgraphs of tournaments
to the true underlying tournament.
\begin{lem}[Rank Spectrum Upper Bound for DAGs]
\label{lem:dag-rank-upperbound} Let $G=(V,E)$ be a DAG on $n$
vertices. Then 
\[
r^{(i)}(G)\leq i-1,\quad\text{for all }i\in[n].
\]
\end{lem}

\begin{proof}
Let $v=v^{(i)}(G)$ with in-reach $r:=r^{(i)}(G)=|\inreach_{G}(v)|$.
By Lemma~\ref{lem:inreach-strict-decrease}, every vertex in $\inreach_{G}(v)$
has in-reach strictly less than~$r$. Since $|\inreach_{G}(v)|=r$,
at least $r$ vertices have in-reach less than~$r$, so $v$ occupies
position $i\geq r+1$ in the non-decreasing in-reach ordering. Hence
$r^{(i)}(G)=r\leq i-1$. 
\end{proof}

\subsection{Condensation Graphs}\label{subsec:scc-and-condensation}

For non-transitive tournaments, condensation plays a crucial role as
in Section~\ref{sec:Non-Transitive-Tournaments}, where it allows reduction
from the general directed graph case to the DAG/transitive case by
collapsing cycles into equivalence classes. 

\paragraph{Strongly Connected Components (SCC).}

Let $G=(V,E)$ be a directed graph. A \emph{strongly connected component}
(SCC) of $G$ is a maximal subset $C\subseteq V$ such that for all
$u,v\in C$, there exist directed paths $u\path_{G}v$ and $v\path_{G}u$.
We define the equivalence relation\footnote{It is well-known that SCC form an equivalence relation. See for example
\href{https://en.wikipedia.org/wiki/Strongly_connected_component}{Wikipedia's SCC page}.} $\sim_{G}$ on $V$ by 
\[
u\sim_{G}v\quad\Longleftrightarrow\quad u\text{ and }v\text{ are in the same SCC of }G.
\]
For $u\in V$, we denote by $[u]_{G}$ the equivalence class of $u$
under $\sim_{G}$.

This equivalence is important in the context of Problem \ref{prob:main-inreach}
where equivalence implies equal in-reach.
\begin{lem}[Equal In-Reach Within SCCs]
\label{lem:scc-equal-inreach} Let $G=(V,E)$ be a directed graph.
If $u\sim_{G}v$, then $|\inreach_{G}(u)|=|\inreach_{G}(v)|$. 
\end{lem}

\begin{proof}
Since $u\sim_{G}v$, we have $u\leadsto_{G}v$ and $v\leadsto_{G}u$.
For any $w\in V\setminus\{u,v\}$: 
\[
w\in\inreach_{G}(u)\;\Longleftrightarrow\;w\leadsto_{G}u\;\Longleftrightarrow\;w\leadsto_{G}v\;\Longleftrightarrow\;w\in\inreach_{G}(v),
\]
where the middle equivalence follows by concatenating with the path
$u\leadsto_{G}v$ or $v\leadsto_{G}u$, respectively. Hence $\inreach_{G}(u)\setminus\{v\}=\inreach_{G}(v)\setminus\{u\}$,
which gives $|\inreach_{G}(u)|=|\inreach_{G}(v)|$. 
\end{proof}
\begin{defn}[Condensation Graph]
\label{def:condensation-graph}The \emph{condensation graph} of a
directed graph $G$, denoted $\condense G$, is the directed graph
with vertex set $\condense V_{G}:=V/{\sim_{G}}$ and edge set 
\[
\condense E_{G}:=\left\{ ([u]_{G},[v]_{G})\mid\exists\,x\in[u]_{G},\,y\in[v]_{G}\text{ such that }(x,y)\in E\text{ and }[u]_{G}\neq[v]_{G}\right\} .
\]
We will also omit the subscript $G$ and use $\condense V,\condense E$,
and $\condense u\in\condense V$ when it is clear from the context.
\end{defn}

\begin{lem}[Reachability Respects Condensation]
\label{lem:condensation-reachability} Let $G=(V,E)$ be a directed
graph. Let $C,D\in\condense V_{G}$ be distinct vertices in $\condense G$.
Then 
\[
C\leadsto_{\condense G}D\iff u\leadsto_{G}v,\quad\text{for all }u\in C\text{ and }v\in D.
\]
\end{lem}

\begin{proof}
$(\Leftarrow)$ Fix any $u\in C$ and $v\in D$ with $u\leadsto_{G}v$.
The path from $u$ to $v$ passes through a sequence of SCCs, inducing
a path from $C$ to $D$ in $\condense G$ by definition of the condensation
edges.

$(\Rightarrow)$ Let $C=C_{0}\to C_{1}\to\cdots\to C_{\ell}=D$ be
a path in $\condense G$. For each edge $(C_{i},C_{i+1})\in\condense E_{G}$,
there exist $x_{i}\in C_{i}$ and $y_{i}\in C_{i+1}$ with $(x_{i},y_{i})\in E$.
Since $C_{i+1}$ is a strongly connected component, $y_{i}$ and $x_{i+1}$
are mutually reachable within $C_{i+1}$, so $y_{i}\leadsto_{G}x_{i+1}$.
Now let $u\in C$ and $v\in D$ be arbitrary. Since $u,x_{0}\in C_{0}$
and $x_{\ell},v\in C_{\ell}$, mutual reachability within these SCCs
gives $u\leadsto_{G}x_{0}$ and $x_{\ell}\leadsto_{G}v$. Concatenating:
\[
u\leadsto_{G}x_{0}\to y_{0}\leadsto_{G}x_{1}\to y_{1}\leadsto_{G}\cdots\leadsto_{G}x_{\ell}\to y_{\ell}\leadsto_{G}v.\qedhere
\]
\end{proof}
The condensation of a directed graph is always a DAG
(Lemma~\ref{lem:condensation-acyclic}). When the graph is a
tournament, the condensation is moreover a transitive tournament
(Proposition~\ref{prop:condensation-transitive}), which pins down the
in-reach of every vertex (Lemma~\ref{lem:inreach-within-scc}). These
two tournament-specific results are included here because they
complete the condensation picture and are referenced by multiple later
sections.
\begin{lem}[Condensation is a DAG]
\label{lem:condensation-acyclic} For any directed graph $G$, the
condensation $\condense G$ is a DAG. 
\end{lem}

\begin{proof}
Suppose $\condense G$ contains a cycle of distinct nodes $[u_{1}]\to[u_{2}]\to\cdots\to[u_{\ell}]\to[u_{1}]$
with $\ell\geq2$. Then for any $x\in[u_{1}]$ and $y\in[u_{2}]$,
we have paths $x\path y$ and $y\path x$ (via the cycle), implying
$x\sim_{G}y$. This contradicts $[u_{1}]\neq[u_{2}]$. 
\end{proof}

\begin{prop}[Condensation of Tournament is Transitive Tournament]
\label{prop:condensation-transitive} If $G^{*}$ is a tournament,
then $\condense{G^{*}}$ is a transitive tournament. 
\end{prop}

\begin{proof}
Let $C_{1},C_{2}$ be distinct SCCs of $G^{*}$. Since $G^{*}$ is
a tournament, for any $u\in C_{1}$ and $v\in C_{2}$, either $(u,v)\in E^{*}$
or $(v,u)\in E^{*}$. We claim all edges between $C_{1}$ and $C_{2}$
go in the same direction. Suppose not: let $(u_{1},v_{1}),(v_{2},u_{2})\in E^{*}$
with $u_{1},u_{2}\in C_{1}$ and $v_{1},v_{2}\in C_{2}$. Since $u_{1}\sim_{G^{*}}u_{2}$
and $v_{1}\sim_{G^{*}}v_{2}$, we have paths $u_{2}\path_{G^{*}}u_{1}$
and $v_{1}\path_{G^{*}}v_{2}$. Then 
\[
u_{1}\to v_{1}\path_{G^{*}}v_{2}\to u_{2}\path_{G^{*}}u_{1}
\]
forms a cycle containing both $u_{1}$ and $v_{1}$, contradicting
$C_{1}\neq C_{2}$. Thus $\condense{G^{*}}$ is a tournament. Furthermore,
since every acyclic tournament is transitive, $\condense{G^{*}}$
is a transitive tournament. 
\end{proof}
\begin{lem}[In-Reach Within an SCC]
\label{lem:inreach-within-scc} Let $G^{*}=(V,E^{*})$ be a tournament
with condensation SCCs $C_{1},\ldots,C_{L}$ in condensation order.
Then for any $v\in C_{j}$, 
\[
|\inreach_{G^{*}}(v)|=\sum_{i=1}^{j}|C_{i}|-1.
\]
In particular, all vertices in the same SCC share the same in-reach,
and the common in-reach is strictly increasing in $j$. 
\end{lem}

\begin{proof}
Since $C_{j}$ is strongly connected, every vertex in $C_{j}\setminus\{v\}$
reaches $v$, contributing $|C_{j}|-1$. For $i<j$, the condensation
path $C_{i}\leadsto_{\condense{G^{*}}}C_{j}$ and Lemma~\ref{lem:condensation-reachability}
imply every vertex in $C_{i}$ reaches $v$, contributing $|C_{i}|$.
For $i>j$, the acyclicity of $\condense{G^{*}}$ implies no vertex
in $C_{i}$ reaches $v$. Summing gives the result. Strict monotonicity
in $j$ follows from $|C_{j}|\geq1$. 
\end{proof}

\subsection{Known Relationships}\label{subsec:Known-Relationships}

Here, we introduce the \textit{known relationship} set, which provides
a simpler finalization criterion used in Algorithm~\ref{alg:tournament-sort-main}.
\begin{defn}[Known relationships]
The \emph{known relationship set} of $v$ in $G$ is 
\begin{equation}
K_{G}(v):=\inreach_{G}(v)\cup\outreach_{G}(v),\label{eq:known-rel-def}
\end{equation}
and we write 
\[
\kappa_{G}(v):=|K_{G}(v)|
\]
 for its cardinality. 

A vertex $v$ is \emph{resolved} in $G$ if $\kappa_{G}(v)=n-1$,
i.e., every other vertex is reachable from or to $v$. 
\end{defn}

The known relationship has an interesting connection to the condensation
graph.
\begin{lem}[Known Relationships Lift Between Graph and Condensation]
\label{lem:known-rel-condensation-iff} Let $G=(V,E)$ be a directed
graph with $n$ vertices. Let $\condense G$ be its condensation with
$n':=|\condense V_{G}|$ vertices. For any SCC $C\in\condense V_{G}$,
the following are equivalent: 
\begin{enumerate}
\item $\kappa_{\condense G}(C)=n'-1$, 
\item $\kappa_{G}(u)=n-1$ for all $u\in C$, 
\item $\kappa_{G}(u)=n-1$ for some $u\in C$. 
\end{enumerate}
\end{lem}

\begin{proof}
\textbf{(i)$\Rightarrow$(ii).} Let $u\in C$ and $v\in V\setminus\{u\}$.
If $v\in C$, then $u\sim_{G}v$, so $v\in\inreach_{G}(u)\cap\outreach_{G}(u)$.
If $v\in D$ for some SCC $D\neq C$, then $\kappa_{\condense G}(C)=n'-1$
implies $D\in\inreach_{\condense G}(C)\cup\outreach_{\condense G}(C)$.
Lemma~\ref{lem:condensation-reachability} gives $v\leadsto_{G}u$
or $u\leadsto_{G}v$, so $v\in\inreach_{G}(u)\cup\outreach_{G}(u)$.
Since $v$ was arbitrary, $\kappa_{G}(u)=n-1$.

\textbf{(ii)$\Rightarrow$(iii).} Immediate (since $C\neq\emptyset$).

\textbf{(iii)$\Rightarrow$(i).} Let $u\in C$ satisfy $\kappa_{G}(u)=n-1$.
Take any SCC $D\neq C$ and pick any $v\in D$. Since $v\in\inreach_{G}(u)\cup\outreach_{G}(u)$,
either $v\leadsto_{G}u$ or $u\leadsto_{G}v$. By Lemma~\ref{lem:condensation-reachability},
either $D\leadsto_{\condense G}C$ or $C\leadsto_{\condense G}D$,
so $D\in\inreach_{\condense G}(C)\cup\outreach_{\condense G}(C)$.
Since $D$ was arbitrary, $\kappa_{\condense G}(C)=n'-1$. 
\end{proof}
\begin{rem}
The equivalence (ii)$\Leftrightarrow$(iii) says that resolution is
an \emph{all-or-nothing} property within an SCC: if any single vertex
in $C$ has all its relationships determined, then every vertex in
$C$ does. This is immediate from the proof, but can also be seen
directly: for $u,w\in C$, mutual reachability gives $\inreach_{G}(u)\cup\outreach_{G}(u)\supseteq\left(\inreach_{G}(w)\cup\outreach_{G}(w)\right)\setminus\{u\}$,
since any vertex reaching $w$ also reaches $u$ (via $w\leadsto u$),
and any vertex reachable from $w$ is reachable from $u$ (via $u\leadsto w$). 
\end{rem}

\subsection{Finalization Engine}

We now introduce machinery for measuring how much of the ranking
has been determined in a directed graph. The definitions below apply
to any directed graph; they will later be instantiated with the
revealed subgraphs arising from oracle queries.
\begin{defn}[Finalization Threshold and Finalized Set]
\label{def:finalization-general-graph} Let $G=(V,E)$ be a directed
graph on $n$ vertices with rank spectrum $R(G)=(r^{(1)}(G),\ldots,r^{(n)}(G))$
and corresponding rank basis $(v^{(1)}(G),\ldots,v^{(n)}(G))$. 
\begin{enumerate}
\item The \emph{finalization threshold} $m(G)$ is the largest $j\geq0$
such that: 
\begin{equation}
\begin{aligned}\text{(i)} & \quad r^{(i)}(G)=i-1\text{ for all }i=1,\ldots,j,\\
\text{(ii)} & \quad j\in\{0,n\}\text{ or }r^{(j)}(G)<r^{(j+1)}(G).
\end{aligned}
\label{eq:finalization-threshold}
\end{equation}
\item The \emph{finalized set} contains the finalized vertices: 
\begin{equation}
\mathrm{TOP}(G):=\{v^{(1)}(G),\ldots,v^{(m(G))}(G)\}.\label{eq:top-set}
\end{equation}
\item The \emph{candidate set} is the complement of the finalized set: 
\begin{equation}
\mathrm{CAND}(G):=V\setminus\mathrm{TOP}(G)=\{v^{(m(G)+1)}(G),\ldots,v^{(n)}(G)\}.\label{eq:cand-set}
\end{equation}
\end{enumerate}
\end{defn}

Informally, $m(G)$ measures how much of the rank spectrum has ``crystallized''
into the pattern $0,1,2,\ldots$ that characterizes a fully discovered
transitive tournament. The finalized set $\mathrm{TOP}(G)$ contains
the vertices whose positions in this prefix are fully determined.

\subsubsection{General Properties of $\mathrm{TOP}(G)$}

Note that $v^{(j)}(G)\in\mathrm{TOP}(G)$ if and only if $j\le m(G)$.
We establish some properties about nodes in $\mathrm{TOP}(G)$ for
general directed graphs. 
\begin{lem}[Exact In-Reach of Top Vertices]
\label{lem:exact-inreach-finalized} Let $G=(V,E)$ be a directed
graph. If $m(G)\ge1$, we have $\inreach_{G}(v^{(1)}(G))=\emptyset$
and for $1<j\le m(G)$, we have 
\[
\inreach_{G}(v^{(j)}(G))=\{v^{(1)}(G),\ldots,v^{(j-1)}(G)\}.
\]
\end{lem}

\begin{proof}
Let $j\leq m(G)$ and write $v^{(j)}=v^{(j)}(G)$ for brevity. The
$j=1$ case is trivial. Consider $j\ge2$. By the finalization threshold,
$|\inreach_{G}(v^{(j)})|=r^{(j)}(G)=j-1$.

Let $u\in\inreach_{G}(v^{(j)})$. By Lemma~\ref{lem:inreach-downward-closure},
$\inreach_{G}(u)\subseteq\inreach_{G}(v^{(j)})$, so $|\inreach_{G}(u)|\leq j-1$.
Since the finalization threshold requires the top $j$ vertices to
have distinct in-reach sizes and $v^{(j)}$ already has in-reach $j-1$,
we must have strict inequality. Hence, $|\inreach_{G}(u)|\leq j-2$,
which places $u\in\{v^{(1)},\ldots,v^{(j-1)}\}$. This means that
$\inreach_{G}(v^{(j)})\subseteq\{v^{(1)},\ldots,v^{(j-1)}\}$. Both
sets have cardinality $j-1$, so they are equal. 
\end{proof}
\begin{cor}[Reachability of Top Vertices]
\label{cor:mutual-reach-finalized} Let $G=(V,E)$ be a directed
graph. If $1\le i<j\le m(G)$, then we have 
\[
v^{(i)}(G)\leadsto_{G}v^{(j)}(G).
\]
\end{cor}

\begin{proof}
Write $v^{(j)}=v^{(j)}(G)$ for brevity. By Lemma \ref{lem:exact-inreach-finalized},
$\inreach_{G}(v^{(j)})=\{v^{(1)},\ldots,v^{(j-1)}\}$. Since $i<j$,
we have $v^{(i)}\in\{v^{(1)},\ldots,v^{(j-1)}\}=\inreach_{G}(v^{(j)})$,
which by definition means $v^{(i)}\leadsto_{G}v^{(j)}$. 
\end{proof}
\begin{lem}[Vertices Outside of TOP Have Large In-Reach]
\label{lem:candidate-large-inreach} Let $G=(V,E)$ be a directed
graph. If $u\notin\mathrm{TOP}(G)$, then $|\inreach_{G}(u)|\geq m(G)$. 
\end{lem}

\begin{proof}
Note that this trivially holds if $m(G)=0$. Consider $m(G)>0$. Let
$m:=m(G)$ for short and let $u:=v^{(k)}(G)$ for some $k>m$, so
that $u\notin\mathrm{TOP}(G)$. From the finalization criterion (Definition
\ref{def:finalization-general-graph}), we have $r^{(m)}(G)=m-1$
and $r^{(m)}(G)<r^{(m+1)}(G)$. Since the rank spectrum is non-decreasing
and $k>m$: 
\[
|\inreach_{G}(u)|=r^{(k)}(G)\geq r^{(m+1)}(G)>r^{(m)}(G)=m-1.
\]
Since in-reach sizes are integers, $|\inreach_{G}(u)|\geq m=m(G)$. 
\end{proof}
The next Lemma requires $G$ to be a DAG.
\begin{lem}[Top Vertices Reach All Non-Top Vertices]
\label{lem:finalized-reach-candidates} Let $G=(V,E)$ be a DAG.
For any $w\in\mathrm{TOP}(G)$ and any $u\in\mathrm{CAND}(G)$, we
have $w\leadsto_{G}u$. 
\end{lem}

\begin{proof}
Let $u\in\mathrm{CAND}(G)\implies u\not\in\mathrm{TOP}(G)$. We construct
a sequence $(x_{i})_{i\geq0}$ iteratively from $u$ to a node to
$\mathrm{TOP}(G)$: 
\begin{itemize}
\item Set $x_{0}=u$. 
\item While $\inreach_{G}(x_{i})\cap\mathrm{CAND}(G)\neq\emptyset$, set
$x_{i+1}\in\inreach_{G}(x_{i})\cap\mathrm{CAND}(G)$. Note that $x_{i+1}\leadsto x_{i}\leadsto x_{i-1}\leadsto\dots\leadsto x_{0}=u$. 
\item We terminate the sequence when $\inreach_{G}(x_{i})\cap\mathrm{CAND}(G)=\emptyset$. 
\end{itemize}
Each $x_{i}$ lies in $\mathrm{CAND}(G)=V\setminus\mathrm{TOP}(G)$,
so by Lemma \ref{lem:candidate-large-inreach}: 
\[
|\inreach_{G}(x_{i})|\geq m(G),\quad\text{for all }i.
\]
By Lemma \ref{lem:inreach-strict-decrease}, since $G$ is a DAG by
assumption, whenever $x_{i+1}$ exists: 
\[
|\inreach_{G}(x_{i+1})|<|\inreach_{G}(x_{i})|.
\]
The in-reach sizes form a strictly decreasing sequence of non-negative
integers bounded below by $m(G)$. Therefore, the sequence must terminate
at some finite index $k$ and each element $x_{i}$ for $i=0,\dots,k$
is unique. We show that at termination, the in-reach of the termination
node $x_{k}$ equals $\mathrm{TOP}(G)$ i.e., $\inreach_{G}(x_{k})=\mathrm{TOP}(G)$.

At termination time $k$, we have $\inreach_{G}(x_{k})\cap\mathrm{CAND}(G)=\emptyset$.
That implies 
\[
\inreach_{G}(x_{k})\subseteq\mathrm{TOP}(G).
\]
Since $x_{k}\in\mathrm{CAND}(G)\implies x_{k}\not\in\mathrm{TOP}(G)$,
Lemma \ref{lem:candidate-large-inreach} gives $|\inreach_{G}(x_{k})|\geq m(G)$.
Combined with $|\mathrm{TOP}(G)|=m(G)$ and the inclusion above implies:
\[
\inreach_{G}(x_{k})=\mathrm{TOP}(G).
\]
Since $x_{i+1}\in\inreach_{G}(x_{i})$ for all $i=0,1,\ldots,k-1$,
by Lemma \ref{lem:inreach-downward-closure}, we have $\inreach_{G}(x_{i+1})\subseteq\inreach_{G}(x_{i})$.
Applying this repeatedly: 
\[
\mathrm{TOP}(G)=\inreach_{G}(x_{k})\subseteq\inreach_{G}(x_{k-1})\subseteq\cdots\subseteq\inreach_{G}(x_{0})=\inreach_{G}(u).
\]
Hence, for any $w\in\mathrm{TOP}(G)$, we have $w\in\inreach_{G}(u)$,
which by definition means $w\leadsto_{G}u$. 
\end{proof}
\begin{rem}
\label{rem:dag-necessary-for-reach} Unlike most results in this section,
the DAG assumption in Lemma~\ref{lem:finalized-reach-candidates}
is essential. Cycles can inflate the in-reach of vertices outside
$\mathrm{TOP}(G)$ without any connection to the finalized vertices.
For instance, consider $V=\{a,b,c\}$ with only edges $b\to c$ and
$c\to b$. Then $\mathrm{TOP}(G)=\{a\}$, yet $a$ reaches neither
$b$ nor $c$: the cycle between $b$ and $c$ alone is responsible
for their in-reach meeting the threshold $m(G)=1$. 
\end{rem}

\subsubsection{Structural Properties of the Finalization Threshold }\label{subsec:structural-properties-finalization}

The following lemmas establish structural properties of the finalization
threshold. 
\begin{lem}[Tied Candidates After Finalization Threshold]
\label{lem:tied-candidates-after-finalization} Let $G=(V,E)$ be
a DAG on $n$ vertices. Let $m':=m(G)$ be the finalization threshold
of~$G$. If $0<m'\leq n-2$, then 
\[
r^{(m'+1)}(G)=r^{(m'+2)}(G).
\]
\end{lem}

\begin{proof}
Suppose for contradiction that $r^{(m'+1)}(G)<r^{(m'+2)}(G)$. Since
$0<m'\leq n-2$, the finalization conditions give $r^{(m')}(G)=m'-1$
and $r^{(m')}(G)<r^{(m'+1)}(G)$. By Lemma~\ref{lem:dag-rank-upperbound},
$r^{(m'+1)}(G)\leq m'$. Hence 
\[
m'-1=r^{(m')}(G)<r^{(m'+1)}(G)\leq m',
\]
forcing $r^{(m'+1)}(G)=m'$. Thus $r^{(i)}(G)=i-1$ holds for all
$i\leq m'+1$, and the strict inequality $r^{(m'+1)}(G)<r^{(m'+2)}(G)$
satisfies condition~(\ref{eq:finalization-threshold})(ii) for $j=m'+1$,
contradicting the maximality of~$m'$. 
\end{proof}
Tied candidates have a useful structural property: no edge can exist
between two vertices with the same in-reach in a DAG.
\begin{lem}[Tied Candidates Have Unknown Edge]
\label{lem:tied-candidates-unknown-edge} Let $G=(V,E)$ be a DAG.
If $u,v\in V$ are distinct nodes satisfying $|\inreach_{G}(u)|=|\inreach_{G}(v)|$,
then neither $\left(u,v\right)$ nor $\left(v,u\right)$ belong to
$E$.
\end{lem}

\begin{proof}
WLOG, suppose for contradiction that $(u,v)\in E$. Thus, $u\in\inreach_{G}(v)$
which implies $\left|\inreach_{G}(u)\right|<\left|\inreach_{G}(v)\right|$
by Lemma \ref{lem:inreach-strict-decrease}. This is a contradiction
to $\left|\inreach_{G}(u)\right|=\left|\inreach_{G}(v)\right|$.
\end{proof}
\begin{lem}[No Penultimate Finalization]
\label{lem:no-penultimate} Let $G=(V,E)$ be a directed graph on
$n\geq2$ vertices. Let $m':=m(G)$ be the finalization threshold
of $G$. Then 
\[
m'\neq n-1.
\]
Consequently, $m'\in\{0,1,\ldots,n-2,n\}$. 
\end{lem}

\begin{proof}
Suppose for contradiction that $m'=n-1$. Since $n\geq2$, we have
$n-1\notin\{0,n\}$, so Definition~\ref{def:finalization-general-graph}
requires: 
\begin{enumerate}
\item $r^{(i)}(G)=i-1$ for all $i\in[n-1]$, and 
\item $r^{(n-1)}(G)<r^{(n)}(G)$. 
\end{enumerate}
From (i), $r^{(n-1)}(G)=n-2$, so (ii) gives $r^{(n)}(G)\geq n-1$.
Since $r^{(n)}(G)=|\inreach_{G}(v^{(n)}(G))|\leq|V\setminus\{v^{(n)}(G)\}|=n-1$,
we obtain $r^{(n)}(G)=n-1$. Hence $r^{(i)}(G)=i-1$ for all $i\in[n]$,
and $j=n$ satisfies the finalization conditions in~(\ref{eq:finalization-threshold})
(since $n\in\{0,n\}$), contradicting the maximality of $m'=n-1$. 
\end{proof}

\section{Transitive Tournaments }\label{sec:transitive-analysis}
We analyze Problem \ref{prob:main-inreach} for the special case of
\textit{transitive tournaments}: A graph $G^*=(V,E^*)$ is a transitive
tournament if it is a tournament and it is transitive i.e., if $(u,v)\in E^{*}$
and $(v,w)\in E^{*}$, then $(u,w)\in E^{*}$. We provide an algorithm
for this special case with correctness guarantees.

\subsection{Algorithm Framework }\label{subsec:algorithmic-framework}

We adopt an iterative approach: at each round, we query a carefully
chosen subset of vertices, update our knowledge of the graph, and
terminate once we have sufficient information to identify the top-$m$
vertices.

\textbf{Notation.} At each round $t=1,2,\ldots,T$: 
\begin{itemize}
\item $Q_{t}\in\binom{V}{k}$ denotes the queried $k$-subset. 
\item $\hat{E}_{t}:=O_{G^{*}}(Q_{t})$ denotes the edges revealed by the
oracle. 
\item $E_{t}:=E_{t-1}\cup\hat{E}_{t}$ denotes the cumulative set of revealed
edges, with $E_{0}:=\emptyset$. 
\item $G_{t}:=(V,E_{t})$ denotes the \emph{revealed graph} at round $t$. 
\end{itemize}
By construction, $\emptyset=E_{0}\subseteq E_{1}\subseteq\cdots\subseteq E_{T}\subseteq E^{*}$.
A good algorithm should ideally produce a strictly increasing chain.

The rank spectrum (Definition \ref{def:rank-spectrum}) of the revealed
graph at round $t$ is the main object of interest.

\textbf{Discovered Rankings.} At round $t$, we order the vertices
by non-decreasing in-reach in $G_{t}$, breaking ties arbitrarily:
\[
v_{t}^{(1)},v_{t}^{(2)},\ldots,v_{t}^{(n)}\quad\text{with}\quad r_{t}^{(1)}\leq r_{t}^{(2)}\leq\cdots\leq r_{t}^{(n)},
\]
where $r_{t}^{(i)}:=|\inreach_{G_{t}}(v_{t}^{(i)})|$. The sequence
$R_{t}:=(r_{t}^{(1)},\ldots,r_{t}^{(n)})$ is the \emph{discovered
rank sequence} at round $t$. We use these as notational shorthands for
Definition \ref{def:finalization-general-graph}, where 
\[
v_{t}^{(i)}:=v^{(i)}(G_{t})\quad\text{and}\quad r_{t}^{(i)}:=r^{(i)}(G_{t}).
\]

Initially, $R_{0}=(0,0,\ldots,0)$ since no edges have been revealed. 

\subsection{Basic Facts about Transitive Tournaments  }\label{subsec:basic-facts-transitive}

We review basic facts about transitive tournaments here. 
\begin{rem}[Transitive Tournaments are Acyclic]
\label{rem:transitive-acyclic} A transitive tournament contains
no directed cycles. Indeed, if $v_{1}\to v_{2}\to\cdots\to v_{\ell}\to v_{1}$
were a cycle with $\ell\geq2$, transitivity would give $(v_{1},v_{\ell})\in E^{*}$,
while the cycle requires $(v_{\ell},v_{1})\in E^{*}$, contradicting
the tournament property. In particular, every transitive tournament
is a DAG. 
\end{rem}

\begin{lem}[Transitive Tournaments are their own Transitive Closure]
\label{lem:transitive-self-closure} Let $G^{*}$ be a transitive
tournament. Then $(G^{*})^{+}=G^{*}$; that is, $u\leadsto_{G^{*}}v$
if and only if $(u,v)\in E^{*}$. 
\end{lem}

\begin{proof}
The backward direction is immediate. For the forward direction, suppose
$u\leadsto_{G^{*}}v$ via a path $u=w_{0}\to w_{1}\to\cdots\to w_{\ell}=v$.
Repeated application of transitivity ($(w_{0},w_{1})\in E^{*}$ and
$(w_{1},w_{2})\in E^{*}$ give $(w_{0},w_{2})\in E^{*}$; then $(w_{0},w_{2})$
and $(w_{2},w_{3})$ give $(w_{0},w_{3})$; and so on) yields $(u,v)\in E^{*}$. 
\end{proof}
\begin{cor}[In-Reach Equals In-Neighborhood]
\label{cor:inreach-equals-indeg} In a transitive tournament $G^{*}$,
\[
\inreach_{G^{*}}(v)=N_{G^{*}}^{-}(v):=\{u\in V\setminus\{v\}:(u,v)\in E^{*}\}
\]
for every vertex $v$. In particular, $|\inreach_{G^{*}}(v)|=\indeg_{G^{*}}(v)$. 
\end{cor}

\begin{proof}
Immediate from Lemma~\ref{lem:transitive-self-closure}: $u\in\inreach_{G^{*}}(v)$
iff $u\leadsto_{G^{*}}v$ iff $(u,v)\in E^{*}$ iff $u\in N_{G^{*}}^{-}(v)$. 
\end{proof}
\begin{lem}[Strict In-Degree Separation in Transitive Tournaments]
\label{lem:strict-indeg-separation} Let $G^{*}$ be a transitive
tournament. If $(u,v)\in E^{*}$, then $\indeg_{G^{*}}(v)\geq\indeg_{G^{*}}(u)+1$.
In particular, no two vertices share the same in-degree. 
\end{lem}

\begin{proof}
Suppose $(u,v)\in E^{*}$. For any in-neighbor $w$ of $u$, we have
$(w,u)\in E^{*}$ and $(u,v)\in E^{*}$, so transitivity gives $(w,v)\in E^{*}$.
Hence $N_{G^{*}}^{-}(u)\subseteq N_{G^{*}}^{-}(v)$. Since $u\in N_{G^{*}}^{-}(v)\setminus N_{G^{*}}^{-}(u)$,
the inclusion is strict, giving $\indeg_{G^{*}}(v)\geq\indeg_{G^{*}}(u)+1$.

For uniqueness: given any two distinct vertices $u\neq v$, the tournament
property forces $(u,v)\in E^{*}$ or $(v,u)\in E^{*}$, so the above
gives $\indeg_{G^{*}}(u)\neq\indeg_{G^{*}}(v)$. 
\end{proof}
The following is a classical result; see, e.g., Chapter~1 of \cite{moon2015topics}.

\begin{cor}[In-Degree Sequence of Transitive Tournaments]
\label{cor:indeg-sequence} A transitive tournament on $n$ vertices
has in-degree multiset $\{0,1,\ldots,n-1\}$. 
\end{cor}

\begin{proof}
By Lemma~\ref{lem:strict-indeg-separation}, the $n$ in-degrees
are pairwise distinct. Since each in-degree lies in $\{0,1,\ldots,n-1\}$
by the tournament property, the in-degrees must be exactly $\{0,1,\ldots,n-1\}$. 
\end{proof}

\subsection{Rank Spectrum of Transitive Tournaments}\label{subsec:rank-spectrum-transitive} 
The rank spectrum (Definition
\ref{def:rank-spectrum}) of a transitive tournament admits some special
properties.
\begin{prop}[Rank Spectrum of Transitive Tournaments]
\label{prop:rank-spectrum-of-transitive-tournaments} Let $G^{*}$
be a transitive tournament of $n$ nodes. We have: 
\[
r^{(i)}(G^{*})=i-1,\quad\text{for }i=1,2,\dots,n.
\]
\end{prop}

\begin{proof}
By Corollary~\ref{cor:inreach-equals-indeg}, $|\inreach_{G^{*}}(v)|=\indeg_{G^{*}}(v)$
for every $v$. By Corollary~\ref{cor:indeg-sequence}, the in-degrees
are exactly $\{0,1,\ldots,n-1\}$. Sorting vertices by non-decreasing
in-reach (Definition~\ref{def:rank-spectrum}) therefore gives $r^{(i)}(G^{*})=i-1$. 
\end{proof}
\textbf{Notation.} Since the transitive tournament $G^{*}$ is fixed
throughout this section, we write 
\begin{align}
r_{*}^{(i)}:=r^{(i)}(G^{*})=i-1\quad\text{and}\quad v_{*}^{(i)}:=v^{(i)}(G^{*})\label{eq:rank-spectrum-of-transitive-notation}
\end{align}
for the rank spectrum and rank basis of $G^{*}$, respectively. 

The rank spectrum (Definition \ref{def:rank-spectrum}) allows us
to make important connections between subgraphs (e.g. our discovered
graph) of tournament graphs to the potential underlying unknown tournament
graph. 
\begin{lem}
\label{lem:subgraph-inreach-upperbound-by-transitive-tournament-degree}
Let $G^{*}=(V,E^{*})$ be a transitive tournament. Let $G'=(V,E')$
be a subgraph of $G^{*}$ where $E'\subseteq E^{*}$. Then for all
$v\in V$: 
\[
|\inreach_{G'}(v)|\leq\indeg_{G^{*}}(v).
\]
\end{lem}

\begin{proof}
Since $E'\subseteq E^{*}$, we have $\inreach_{G'}(v)\subseteq\inreach_{G^{*}}(v)$
for all $v\in V$. Since $G^{*}$ is a transitive tournament, in-reach
coincides with in-neighborhood by Corollary \ref{cor:inreach-equals-indeg},
hence $|\inreach_{G^{*}}(v)|=\indeg_{G^{*}}(v)$. Combining these
inequalities yields the result: 
\[
\left|\inreach_{G'}(v)\right|\leq\left|\inreach_{G^{*}}(v)\right|=\indeg_{G^{*}}(v).\qedhere
\]
\end{proof}
Another connection we can make is by applying Lemma~\ref{lem:dag-rank-upperbound}
to the rank spectrum of subgraphs of the transitive tournament. 
\begin{cor}[Rank Upper Bound for Subgraphs of Transitive Tournaments]
\label{cor:transitive-tournament-rank-upperbounds-subgraph-rank}
Let $G^{*}=(V,E^{*})$ be a transitive tournament on $n$ vertices.
Let $G'=(V,E')$ be a subgraph of $G^{*}$ where $E'\subseteq E^{*}$.
Then 
\[
r^{(i)}(G')\leq r_{*}^{(i)}=i-1,\quad\text{for all }i\in[n].
\]
\end{cor}

\begin{proof}
By Remark~\ref{rem:transitive-acyclic}, $G^{*}$ is acyclic. Since
$E'\subseteq E^{*}$, the subgraph $G'$ is also acyclic. The bound
$r^{(i)}(G')\leq i-1$ then follows from Lemma~\ref{lem:dag-rank-upperbound},
and $r_{*}^{(i)}=i-1$ is Proposition~\ref{prop:rank-spectrum-of-transitive-tournaments}. 
\end{proof}

\subsection{Eliminations and Finalization }\label{subsec:transitive-eliminations-and-finalization}
\begin{rem}[A Filtering Criterion]
\label{rem:elimination-transitive}Let $G^{*}=(V,E^{*})$ be a transitive
tournament. Let $G'=(V,E')$ be a subgraph of $G^{*}$ where $E'\subseteq E^{*}$.
If a vertex $v\in V$ satisfies $|\inreach_{G'}(v)|\geq m$, then
$v$ is \textit{not} in the final top $m$ vertices of $G^{*}$ i.e.,
$v\notin\{v_{*}^{(1)},\ldots,v_{*}^{(m)}\}$.
\end{rem}

\begin{proof}
In a transitive tournament on $n$ vertices, the true in-degrees are
exactly $\{0,1,\ldots,n-1\}$, so $r_{*}^{(m)}=m-1$ (Proposition
\ref{prop:rank-spectrum-of-transitive-tournaments}). By Corollary
\ref{cor:transitive-tournament-rank-upperbounds-subgraph-rank}, $|\inreach_{G'}(v)|\leq\indeg_{G^{*}}(v)$.
Thus $|\inreach_{G'}(v)|\geq m$ implies $\indeg_{G^{*}}(v)\geq m>r_{*}^{(m)}$,
so $v$ cannot be among the top $m$ vertices. 
\end{proof}
We say that a vertex $v$ in a subgraph $G'$ of an underlying tournament
$G^{*}$ is \emph{finalized} when we have sufficient information to
confirm $v$'s membership in the top-$m$ set of $G^{*}$. The following
proposition characterizes when this occurs.
\begin{prop}[Top-$j$ Finalization Criterion in Transitive Tournaments]
\label{prop:top-j-finalization-criterion-in-transitive-tournaments}
Let $G^{*}=(V,E^{*})$ be a transitive tournament. Let $G'=(V,E')$
be a subgraph of $G^{*}$ where $E'\subseteq E^{*}$. Suppose there
exists $j\in[n-1]$ such that $r^{(1)}(G')<r^{(2)}(G')<\cdots<r^{(j)}(G')<r^{(j+1)}(G')$.
Then the top $j$ rank spectrum of $G'$ aligns with $G^{*}$ i.e., 
\begin{enumerate}
\item $r^{(i)}(G')=i-1$ for all $i\leq j$, and 
\item $\left(v^{(1)}(G'),v^{(2)}(G'),\ldots,v^{(j)}(G')\right)=\left(v_{*}^{(1)},v_{*}^{(2)},\ldots,v_{*}^{(j)}\right)$. 
\end{enumerate}
\end{prop}

\begin{proof}
We prove (a) by induction. For the base case, by Corollary \ref{cor:transitive-tournament-rank-upperbounds-subgraph-rank},
$r^{(1)}(G')\leq r_{*}^{(1)}=0$. Furthermore, $r^{(1)}(G')=|\inreach_{G'}(v^{(1)}(G'))|\geq0$
implies that $r^{(1)}(G')=0$. For the inductive step, assuming $r^{(i-1)}(G')=i-2$,
the strict inequality $r^{(i)}(G')>r^{(i-1)}(G')=i-2$ combined with
$r^{(i)}(G')\leq r_{*}^{(i)}=i-1$ from Corollary \ref{cor:transitive-tournament-rank-upperbounds-subgraph-rank}
forces $r^{(i)}(G')=i-1$.

For (b), since $r^{(j+1)}(G')>r^{(j)}(G')=j-1$ from (a), we have
$r^{(j+1)}(G')\geq j$. All vertices $v^{(j+1)}(G'),\ldots,v^{(n)}(G')$
have discovered in-reach at least $r^{(j+1)}(G')\geq j$. Hence, vertices
$v^{(j+1)}(G'),\ldots,v^{(n)}(G')$ cannot be among the top $j$ in
$G^{*}$ following Proposition \ref{prop:rank-spectrum-of-transitive-tournaments}.
This establishes the set equality. The ordering within the set follows
by matching ranks. 
\end{proof}
Proposition \ref{prop:top-j-finalization-criterion-in-transitive-tournaments}
provides the key bridge for transitive tournaments: whenever a
strictly increasing prefix appears in the discovered rank spectrum,
that prefix is correct with respect to the true tournament. Combined
with the finalization threshold (Definition~\ref{def:finalization-general-graph}),
which captures the maximal extent of this crystallization, this
drives both the algorithm design and its termination analysis.

\subsection{A Greedy Algorithm for Ranking Transitive Tournaments }\label{subsec:greedy-schedule}

\begin{algorithm}[tb]
\caption{Greedy Tournament Sort (Transitive Case)}
\label{alg:transitive-tournament-sort}
\begin{algorithmic}[1]
\REQUIRE Vertex set $V$ with $|V|=n$, oracle $O_{G^{*}}$ with query size $k$, target count $m\leq n$
\ENSURE Top-$m$ vertices $(v_{*}^{(1)},\ldots,v_{*}^{(m)})$ in ascending in-reach order
\STATE $E_{0}\leftarrow\emptyset$,\; $G_{0}\leftarrow(V,E_{0})$,\; $t\leftarrow 0$
\WHILE{true}
    \STATE Compute rank spectrum $R_t = (r_{t}^{(1)},\ldots,r_{t}^{(n)})$ of $G_t$ per Definition~\ref{def:rank-spectrum}
    \STATE Compute rank basis $(v_{t}^{(1)},\ldots,v_{t}^{(n)})$ of $G_t$ per Definition~\ref{def:rank-spectrum}
    \STATE Compute finalization threshold $m_{t}:=m(G_t)$ per Definition~\ref{def:finalization-general-graph}
    \IF{$m_t \ge m$} 
        \STATE \COMMENT{termination condition}
        \STATE \textbf{break}
    \ENDIF

    \STATE $k'\leftarrow\min(k,\;n-m_{t})$ \COMMENT{$n - m_t \geq 2$ by Lemma~\ref{lem:no-penultimate}}
    \STATE $S_{t}\leftarrow\{v_{t}^{(m_{t}+1)},\ldots,v_{t}^{(m_{t}+k')}\}$ \COMMENT{$k'$ candidates with smallest in-reach}
    \IF{$k' < k$}
        \STATE Pad: $Q_t \leftarrow S_t \cup P_t$, where $P_t \subseteq \mathrm{TOP}(G_t)$ with $|P_t| = k - k'$
    \ELSE
        \STATE $Q_t \leftarrow S_t$
    \ENDIF
    \STATE Query oracle: $\hat{E}_{t}\leftarrow O_{G^{*}}(Q_{t})$
    \STATE Update: $E_{t+1}\leftarrow E_{t}\cup\hat{E}_{t}$,\; $G_{t+1}\leftarrow(V,E_{t+1})$
    \STATE $t \leftarrow t + 1$
\ENDWHILE
\STATE \textbf{return} $(v_{t}^{(1)},\ldots,v_{t}^{(m)})$
\end{algorithmic}
\end{algorithm}

\begin{rem}
\label{rem:algorithm-details} When $|\mathrm{CAND}(G_{t})|<k$, the
padding vertices in $P_{t}\subseteq\mathrm{TOP}(G_{t})$ are already
finalized; edges revealed involving $P_{t}$ are redundant by Corollary~\ref{cor:mutual-reach-finalized}
and Lemma~\ref{lem:finalized-reach-candidates}, so padding preserves
correctness while satisfying the oracle's requirement that $|Q_{t}|=k$. 
\end{rem}

We now present Algorithm \ref{alg:transitive-tournament-sort} that
iteratively queries carefully and adaptively chosen sets of $k$ vertices
until the top-$m$ vertices are finalized. The key insight is to always
query candidates with the smallest discovered in-reach to guarantee
that we always add an edge to make progress, where we terminate when
there is enough information to finalize all top candidates. The correctness
proof is presented in Theorem \ref{thm:correctness-transitive}.

To establish correctness, we show that Algorithm \ref{alg:transitive-tournament-sort}
adds at least one new edge every round (for at most $\binom{n}{2}$
rounds), guaranteeing that we do not get stuck and that we terminate
eventually when we have enough information.
\begin{thm}[Termination of Algorithm \ref{alg:transitive-tournament-sort}]
\label{thm:termination-transitive} Algorithm \ref{alg:transitive-tournament-sort}
terminates in at most $\binom{n}{2}$ rounds.
\end{thm}

\begin{proof}
At a non-terminal round $t$, we have $m_{t}<m\leq n$. By Lemma~\ref{lem:no-penultimate},
$m_{t}\neq n-1$, so $m_{t}\leq n-2$ and thus $|\mathrm{CAND}(G_{t})|=n-m_{t}\geq2$.
We claim $r_{t}^{(m_{t}+1)}=r_{t}^{(m_{t}+2)}$. If $m_{t}>0$, Lemma
\ref{lem:tied-candidates-after-finalization} implies that $r_{t}^{(m_{t}+1)}=r_{t}^{(m_{t}+2)}$.
If $m_{t}=0$, then the finalization condition $r_{t}^{(1)}<r_{t}^{(2)}$
fails, so $r_{t}^{(1)}=r_{t}^{(2)}$ (since ranks are non-decreasing).
In any case, $r_{t}^{(m_{t}+1)}=r_{t}^{(m_{t}+2)}$ implies that $v_{t}^{(m_{t}+1)}$
and $v_{t}^{(m_{t}+2)}$ are tied candidates i.e., $|\inreach_{G_{t}}(v_{t}^{(m_{t}+1)})|=|\inreach_{G_{t}}(v_{t}^{(m_{t}+2)})|$.
By Lemma~\ref{lem:tied-candidates-unknown-edge}, the edge between
$v_{t}^{(m_{t}+1)}$ and $v_{t}^{(m_{t}+2)}$ is not in $E_{t}$.
Since both vertices are in $Q_{t}$, querying the oracle reveals this
edge, ensuring $|E_{t+1}|>|E_{t}|$.
\end{proof}
The algorithm terminates and outputs something. We now prove that its
output correctly solves Problem \ref{prob:main-inreach}.
\begin{thm}[Correctness of Algorithm \ref{alg:transitive-tournament-sort}]
\label{thm:correctness-transitive} Algorithm \ref{alg:transitive-tournament-sort}
correctly identifies the top-$m$ vertices of a transitive tournament
graph $G^{*}$. 
\end{thm}

\begin{proof}
At termination round $T$, $m_{T}\geq m$. We verify that the output
$(v_{T}^{(1)},\ldots,v_{T}^{(m)})$ satisfies the two conditions of
Problem~\ref{prob:main-inreach}. Since $m_{T}\geq m\geq1$, the
finalization threshold (Definition~\ref{def:finalization-general-graph})
provides: (i) $r_{T}^{(i)}=i-1$ for all $i\leq m_{T}$, and (ii)
$m_{T}=n$ or $r_{T}^{(m_{T})}<r_{T}^{(m_{T}+1)}$. Applying Proposition~\ref{prop:top-j-finalization-criterion-in-transitive-tournaments}
with $j=m_{T}$ (using (i), (ii), and the fact that $G_{T}$ is a
subgraph of the transitive tournament $G^{*}$) yields: 
\begin{equation}
v_{T}^{(i)}=v_{*}^{(i)}\quad\text{for all }i\leq m_{T}.\label{eq:alignment}
\end{equation}
In particular, the first $m\leq m_{T}$ output vertices satisfy $v_{T}^{(i)}=v_{*}^{(i)}$
for $i=1,\ldots,m$. By Corollary~\ref{cor:inreach-equals-indeg},
in-reach coincides with in-degree in a transitive tournament, and
by Proposition~\ref{prop:rank-spectrum-of-transitive-tournaments},
$|\inreach_{G^{*}}(v_{*}^{(i)})|=i-1$. Therefore:

\textbf{\textit{Internal ordering}}\textit{~(\ref{eq:sol-set-satisfies-internal-ordering}).}
For $i=1,\ldots,m$: 
\[
|\inreach_{G^{*}}(v_{T}^{(i)})|=|\inreach_{G^{*}}(v_{*}^{(i)})|=i-1,
\]
which is strictly increasing in $i$.

\noindent\textbf{\textit{Rank dominance}}\textit{~(\ref{eq:rank-dominance-of-sol-set}).}
Let $S=\{v_{T}^{(1)},\ldots,v_{T}^{(m)}\}$. The maximum in-reach
within $S$ is $|\inreach_{G^{*}}(v_{T}^{(m)})|=m-1$. For any $u\notin S$,
we have $u=v_{*}^{(j)}$ for some $j>m$ (by (\ref{eq:alignment})
and the uniqueness of rank basis in transitive tournaments), so $|\inreach_{G^{*}}(u)|=j-1\geq m>m-1$.
Hence 
\[
\max_{i\leq m}|\inreach_{G^{*}}(v_{T}^{(i)})|=m-1<m\leq\min_{u\notin S}|\inreach_{G^{*}}(u)|.\qedhere
\]
\end{proof}

\section{Non-Transitive Tournaments}\label{sec:Non-Transitive-Tournaments}

We now generalize Algorithm~\ref{alg:transitive-tournament-sort}
to arbitrary tournaments. When $G^{*}$ is not transitive, two difficulties
arise.
\begin{enumerate}
\item \textbf{Finalization breaks down.} The finalization threshold (Definition~\ref{def:finalization-general-graph})
requires the rank spectrum to form a strictly increasing prefix $0,1,2,\ldots$
-- the signature of a transitive tournament (Proposition~\ref{prop:rank-spectrum-of-transitive-tournaments}).
In a non-transitive tournament, vertices sharing an SCC have equal
in-reach (Lemma~ \ref{lem:scc-equal-inreach} and \ref{lem:inreach-within-scc}),
so ties are intrinsic and this prefix pattern need not appear in~$G_{t}$
itself.
\item \textbf{Scheduling breaks down.} Algorithm~\ref{alg:transitive-tournament-sort}
guarantees progress by querying tied candidates, which must have an
unknown edge between them (Lemma~\ref{lem:tied-candidates-unknown-edge}).
In the general case, tied vertices may already share a known edge
via a cycle, so the same scheduling rule can stall. 
\end{enumerate}
Both difficulties are resolved by a single observation: the condensation
$\condense G$ of any directed graph is a DAG (Lemma~\ref{lem:condensation-acyclic}),
and the condensation of any tournament is a transitive tournament
(Proposition~\ref{prop:condensation-transitive}). This lets us reuse
the finalization engine from Section~\ref{sec:graph-prelims}
without modification by applying it to~$\condense{G_{t}}$ rather
than to~$G_{t}$ directly.

The tournament bridge, however, must be rebuilt. In-reach in a subgraph
$G_{t}$ no longer directly bounds in-degree in~$G^{*}$, so we instead
relate the two condensations $\condense{G_{t}}$ and $\condense{G^{*}}$
through the natural projection map induced by SCC refinement (Lemma~\ref{lem:scc-refinement}).
This projection preserves reachability (Lemma~\ref{lem:projection-path-preservation})
and yields a new bridge lemma (Lemma~\ref{lem:discovered-respects-true-rank})
that transfers rank information from the revealed subgraph to the
true tournament via the condensation. Adding edges can only merge
SCCs and never split them, so the projection is monotone and the algorithm
makes steady progress.

\begin{algorithm}[tb]
\caption{Tournament Sort General}
\label{alg:non-transitive-tournament-graph-sort}
\begin{algorithmic}[1]
    \REQUIRE Vertex set $V$ with $|V|=n$, oracle $O_{G^{*}}$ with
    query size $k$, target count $m\leq n$ \ENSURE A sequence of top-$m$
    vertices $(v_{*}^{(1)},\ldots,v_{*}^{(m)})$ in non-decreasing in-reach order.
    \STATE $E_{0}\leftarrow\emptyset$,\; $G_{0}\leftarrow(V,E_{0})$,\; $t\leftarrow0$
    \WHILE{true}
        \STATE Compute the condensation graph $\condense{G_{t}}=\left(\condense{V_{t}},\condense{E_{t}}\right)$
        with $n_{t} := \left|\condense{V_{t}}\right|$ per Definition~\ref{def:condensation-graph}.
        \STATE Compute rank spectrum $\widetilde{R_{t}} := (r^{(1)}\left(\condense{G_{t}}\right),\ldots,r^{(n_{t})}\left(\condense{G_{t}}\right))$
        of $\condense{G_{t}}$ per Definition~\ref{def:rank-spectrum}
        \STATE Compute the corresponding rank basis $(C_{t}^{(1)},\ldots,C_{t}^{(n_{t})})$
        of $\condense{G_{t}}$ given $\widetilde{R_{t}}$ per Definition~\ref{def:rank-spectrum}
        where $r^{(i)}\left(\condense{G_{t}}\right):=\inreach_{\condense{G_{t}}}\left(C_{t}^{(i)}\right)$
        for $C_{t}^{(i)}\in\condense{V_{t}}$.
        \STATE Compute finalization threshold $\widetilde{m_{t}}:=m(\condense{G_{t}})$
        per (\ref{eq:finalization-threshold}).
        \STATE Compute finalized SCCs $\widetilde{\mathrm{TOP}_{t}}\leftarrow\left\{ C_{t}^{(1)},\dots,C_{t}^{(\widetilde{m_{t}})}\right\} =\mathrm{TOP}\left(\condense{G_{t}}\right)$
        per (\ref{eq:top-set}).
        \STATE $\mathrm{TOP}_{t}\leftarrow\bigcup_{i=1}^{\widetilde{m_{t}}}C_{t}^{(i)}$
        \COMMENT{lift finalized nodes to graph}
        \IF{$\left|\mathrm{TOP}_{t}\right|\ge m$}
            \STATE \textbf{break}
            \COMMENT{we found enough top nodes}
        \ENDIF
        \STATE $k'\leftarrow\min(k,\;n_{t}-\widetilde{m_{t}})$
        \STATE $\widetilde{\mathrm{CAND}_{t}}\leftarrow\{C_{t}^{(\widetilde{m_{t}}+1)},\ldots,C_{t}^{(\widetilde{m_{t}}+k')}\}$
        \COMMENT{$k'$ candidates with smallest in-reach}
        \IF{$k'<k$}
            \STATE Pad: $\widetilde{Q_{t}}\leftarrow\widetilde{\mathrm{CAND}_{t}}\cup P_{t}$,
            where $P_{t}\subseteq\mathrm{TOP}(\condense{G_{t}})$ with $|P_{t}|=k-k'$
        \ELSE
            \STATE $\widetilde{Q_{t}}\leftarrow\widetilde{\mathrm{CAND}_{t}}$
        \ENDIF
        \STATE Pick representative: $Q_{t}\leftarrow\left\{ \mathrm{rep}(C):C\in\widetilde{Q_{t}}\right\}$
        where $\mathrm{rep}$ picks an element from a set.
        \STATE Query oracle: $\hat{E}_{t}\leftarrow O_{G^{*}}(Q_{t})$
        \STATE Update: $E_{t+1}\leftarrow E_{t}\cup\hat{E}_{t}$,\; $G_{t+1}\leftarrow(V,E_{t+1})$
        \STATE $t\leftarrow t+1$
    \ENDWHILE
    \STATE Let $j^{*}\leftarrow\min\{j\leq\widetilde{m_{t}}:\sum_{i=1}^{j}|C_{t}^{(i)}|\geq m\}$
    \STATE Let $S_{j^{*}}$ be any subset of $C_{t}^{(j^{*})}$ with $|S_{j^{*}}|=m-\sum_{i=1}^{j^{*}-1}|C_{t}^{(i)}|$
    \STATE \textbf{return} the solution sequence $S:=(v^{(1)},\ldots,v^{(m)})$
    obtained by concatenating $C_{t}^{(1)},C_{t}^{(2)},\ldots,C_{t}^{(j^{*}-1)},S_{j^{*}}$
    in this order.
\end{algorithmic}
\end{algorithm}

Algorithm~\ref{alg:non-transitive-tournament-graph-sort} implements
this strategy. The remainder of the section develops the necessary
machinery: Section~\ref{subsec:scc-refinement-and-projection} establishes
SCC refinement and the projection map, Section~\ref{subsec:rank-spectrum-under-condensation}
connects the rank spectra of~$G_{t}$ and~$\condense{G_{t}}$ to
that of~$G^{*}$, and Section~\ref{subsec:general-termination-and-correctness}
proves termination and correctness.

\begin{rem}[Equivalence for Transitive Case]
When the underlying tournament $G^{*}$ is transitive, Algorithm
\ref{alg:non-transitive-tournament-graph-sort} reduces to Algorithm
\ref{alg:transitive-tournament-sort}. Specifically:
\begin{enumerate}
\item All SCCs are singletons: $[u]_{G_{t}}=\{u\}$ for all $u\in V$ and
round $t$. 
\item $\condense{G_{t}}\cong G_{t}$ canonically. 
\item In-reach in $\condense{G_{t}}$ equals in-reach in $G_{t}$. 
\end{enumerate}
\end{rem}

\begin{proof}
In a transitive tournament, there are no cycles, so each SCC is a
singleton. The condensation is isomorphic to the original graph, and
all operations coincide. 
\end{proof}

\subsection{SCC Refinement and Projection }\label{subsec:scc-refinement-and-projection}
\begin{lem}[SCC Refinement]
\label{lem:scc-refinement} Let $G=(V,E)$ be a directed graph. Let
$G'=(V,E')$ be a subgraph of $G$ where $E'\subseteq E$. Then, we
have
\[
u\sim_{G'}v\implies u\sim_{G}v.
\]
Equivalently, each SCC of subgraphs of $G$ is contained in some SCC
of $G$ i.e.,
\[
G'\subseteq G\implies\condense u_{G'}\subseteq\condense u_{G},\ \text{ for all }u\in V.
\]
 
\end{lem}

\begin{proof}
If $u\sim_{G'}v$, then there exist paths $u\path_{G'}v$ and $v\path_{G'}u$.
Since $E'\subseteq E$, these paths exist in $G$, so $u\sim_{G}v$. 
\end{proof}
\begin{defn}[Projection Map]
Let $G=(V,E)$ be a directed graph. Let $G'=(V,E')$ be a subgraph
of $G$ where $E'\subseteq E$. The refinement induces a natural projection
$\phi_{G}:\condense V_{G'}\to\condense V_{G}$ defined by 
\[
\phi_{G}([u]_{G'})=[u]_{G}.
\]
This is well-defined by Lemma~\ref{lem:scc-refinement}. We will
often omit the subscript $G$ and domain definition when the domain
and codomain are clear from the context.
\end{defn}

\begin{lem}[Path Preservation Under Projection]
\label{lem:projection-path-preservation} Let $G=(V,E)$ be a directed
graph. Let $G'=(V,E')$ be a subgraph of $G$ where $E'\subseteq E$.
If $D\path_{\condense{G'}}C$ in $\condense{G'}$, then either $\phi_{G}(D)=\phi_{G}(C)$
or $\phi_{G}(D)\path_{\condense G}\phi_{G}(C)$. 
\end{lem}

\begin{proof}
Consider a path $D=D_{0}\to D_{1}\to\cdots\to D_{\ell}=C$ in $\condense{G'}$.
For each edge $(D_{i},D_{i+1})\in\condense{E'}$, there exist $x_{i}\in D_{i}$,
$y_{i}\in D_{i+1}$ with $(x_{i},y_{i})\in E'\subseteq E$.

If $\phi_{G}(D_{i})=\phi_{G}(D_{i+1})$, the edge collapses to a single
vertex.

If $\phi_{G}(D_{i})\neq\phi_{G}(D_{i+1})$, then $[x_{i}]_{G}\neq[y_{i}]_{G}$,
so $(\phi_{G}(D_{i}),\phi_{G}(D_{i+1}))\in\condense E$.

Concatenating the non-collapsed edges yields either $\phi_{G}(D)=\phi_{G}(C)$
(if every edge along the path collapses) or a walk from $\phi_{G}(D)$
to $\phi_{G}(C)$ in $\condense G$. Since $\condense G$ is a DAG
by Lemma~\ref{lem:condensation-acyclic}, every walk between distinct
vertices contains a path, so $\phi_{G}(D)\path_{\condense G}\phi_{G}(C)$. 
\end{proof}

\subsection{Rank Spectrum Under Condensation}\label{subsec:rank-spectrum-under-condensation}
\begin{lem}[Discovered Reachability Respects True Rank]
\label{lem:discovered-respects-true-rank} Let $G=(V,E)$ be a directed
graph. Let $G'=(V,E')$ be a subgraph of $G$ where $E'\subseteq E$.
Let $C,D\in\condense V_{G'}$ be distinct SCCs of $G'$. Then 
\[
C\leadsto_{\condense{G'}}D\implies\left|\inreach_{G}(u)\right|\le\left|\inreach_{G}(v)\right|,\quad\text{for all }u\in C\text{ and }v\in D,
\]
with equality if and only if $\phi_{G}(C)=\phi_{G}(D)$ (i.e.\ $u$
and $v$ belong to the same SCC of $G$). 
\end{lem}

\begin{proof}
Let $C,D\in\condense V_{G'}$ be distinct SCCs of $G'$ with $C\leadsto_{\condense{G'}}D$.
Let $u\in C$ and $v\in D$. Lemma~\ref{lem:projection-path-preservation}
implies that either $\phi_{G}(C)=\phi_{G}(D)$ or $\phi_{G}(C)\leadsto_{\condense G}\phi_{G}(D)$. 

In the first case, $[u]_{G}=[v]_{G}$. This means that they are in
the same SCC i.e., $u\sim_{G}v$. Hence, Lemma (\ref{lem:scc-equal-inreach})
implies that $\left|\inreach_{G}(u)\right|=\left|\inreach_{G}(v)\right|$.

In the second case, we have $\phi_{G}(C)\neq\phi_{G}(D)$ and $\phi_{G}(C)\leadsto_{\condense G}\phi_{G}(D)$.
Since $\condense G$ is a DAG (Lemma~\ref{lem:condensation-acyclic}),
there is no path from $[v]_{G}$ back to $[u]_{G}$, so $v\not\leadsto_{G}u$
and hence $v\notin\inreach_{G}(u)$. Since $C\leadsto_{\condense{G'}}D$,
there exists a path $u\leadsto_{G'}v$ via Lemma \ref{lem:condensation-reachability}
when applied to $G'$. Since $u\leadsto_{G'}v$ and $E'\subseteq E$,
we have $u\leadsto_{G}v$. Then by Lemma~\ref{lem:inreach-downward-closure},
we have $\inreach_{G}(u)\subseteq\inreach_{G}(v)$. Finally, $u\in\inreach_{G}(v)$
but $v\not\in\inreach_{G}(u)$ as shown and $u\not\in\inreach_{G}(u)$
by definition implies $\left|\inreach_{G}(u)\right|<\left|\inreach_{G}(v)\right|$.
The equality if and only if follows from this case. 
\end{proof}
\begin{rem}
Lemma~\ref{lem:discovered-respects-true-rank} is the key ``bridge''
between the discovered subgraphs ($G'=G_{t}$) in the algorithm and
the true tournament ($G=G^{*}$). 
\end{rem}

Lemma \ref{lem:general-finalized-reach-candidates} provides an analogous
result to Lemma~\ref{lem:finalized-reach-candidates} for the condensation
graph. 
\begin{lem}[Top Vertices Reach All Non-Top Vertices in Condensation Graphs]
\label{lem:general-finalized-reach-candidates} Let $G=(V,E)$ be
a directed graph. Let $G'=(V,E')$ be a subgraph of $G$ where $E'\subseteq E$.
Then, for any \textup{$C\in\mathrm{TOP}\left(\condense{G'}\right)\text{ and }D\notin\mathrm{TOP}\left(\condense{G'}\right)$,
we have}
\begin{equation}
\left|\inreach_{G}(u)\right|\le\left|\inreach_{G}(v)\right|,\quad\text{for all }u\in C\text{ and }v\in D.\label{eq:finalized-dominance}
\end{equation}
\end{lem}

\begin{proof}
Since $\condense{G'}$ is a DAG, $C\leadsto_{\condense{G'}}D$ follows
from \ref{lem:finalized-reach-candidates}. Equation (\ref{eq:finalized-dominance})
follows immediately from Lemma \ref{lem:discovered-respects-true-rank}.
\end{proof}
\begin{lem}[Finalized Order is Correct]
\label{lem:finalized-order-correct} Let $G=(V,E)$ be a directed
graph. Let $G'=(V,E')$ be a subgraph of $G$ where $E'\subseteq E$.
If $1\le i<j\le m(\condense{G'})$, then 
\begin{equation}
\left|\inreach_{G}(u)\right|\le\left|\inreach_{G}(v)\right|,\quad\text{for all }u\in v^{(i)}\left(\condense{G'}\right)\text{ and }v\in v^{(j)}\left(\condense{G'}\right).\label{eq:finalized-order-correct}
\end{equation}
\end{lem}

\begin{proof}
Since $\condense{G'}$ is a DAG, $v^{(i)}\left(\condense{G'}\right)\leadsto_{\condense{G'}}v^{(j)}\left(\condense{G'}\right)$
follows from Corollary~\ref{cor:mutual-reach-finalized}. Then equation
(\ref{eq:finalized-order-correct}) follows immediately from Lemma~\ref{lem:discovered-respects-true-rank}. 
\end{proof}

\subsection{Termination and Correctness of Algorithm \ref{alg:non-transitive-tournament-graph-sort}}\label{subsec:general-termination-and-correctness}

Now we provide a generalization of Lemma \ref{lem:tied-candidates-unknown-edge}
to the non-transitive case with a minor adaptation to handle SCCs
rather than individual vertices. 
\begin{lem}[Equal In-Reach Implies Unknown Edge]
\label{lem:tied-unknown-edge-general}Let $G=(V,E)$ be a directed
graph. If $C,D\in\condense V$ are distinct SCCs with $\left|\inreach_{\condense G}(C)\right|=\left|\inreach_{\condense G}(D)\right|$,
then the edge between $C$ and $D$ is not in $\condense G$. Consequently,
for any $u\in C$, $v\in D$, the edge between $u$ and $v$ is not
in $E$. 
\end{lem}

\begin{proof}
Since $\condense G$ is a DAG, Lemma~\ref{lem:tied-candidates-unknown-edge}
implies that there is no edge between $C$ and $D$ in $\condense G$.
This in turn implies that there cannot be any edge between any
$u\in C$ and $v\in D$, because otherwise there would be an edge
between $C$ and $D$ by definition.
\end{proof}
\begin{thm}[Termination of Algorithm~\ref{alg:non-transitive-tournament-graph-sort}]
\label{thm:termination-general-theory}Suppose that the underlying
tournament $G^{*}=\left(V,E^{*}\right)$ has at least 2 SCCs i.e.,
$\left|\condense V_{G^{*}}\right|\ge2$. Then Algorithm~\ref{alg:non-transitive-tournament-graph-sort}
terminates in at most $\binom{n}{2}$ rounds.
\end{thm}

\begin{proof}
We use the notation from Algorithm~\ref{alg:non-transitive-tournament-graph-sort}
throughout: $n_{t}$, $\widetilde{m_{t}}$, $\mathrm{TOP}_{t}$,
$\widetilde{\mathrm{CAND}_{t}}$, and $C_{t}^{(i)}$ refer to the
condensation quantities defined there.

If the algorithm has not terminated, $|\mathrm{TOP}_{t}|<m$. Note
that $n_{t}\ge\left|\condense V_{G^{*}}\right|\ge2$ for all $t$.
We claim $\widetilde{m_{t}}<n_{t}$. If $\widetilde{m_{t}}=n_{t}$,
then $|\mathrm{TOP}_{t}|=n\geq m$, which is a contradiction. By Lemma~\ref{lem:no-penultimate}
on $\condense{G_{t}}$ and $\widetilde{m_{t}}$, we must have $\widetilde{m_{t}}\neq n_{t}-1$.
Hence, $\widetilde{m_{t}}\leq n_{t}-2$ and thus $|\widetilde{\mathrm{CAND}_{t}}|\geq2$.
We now show that there are at least 2 tied SCCs in $\widetilde{\mathrm{CAND}_{t}}$,
specifically $\left|\inreach_{\condense{G_{t}}}\left(C_{t}^{(\widetilde{m_{t}}+1)}\right)\right|=\left|\inreach_{\condense{G_{t}}}\left(C_{t}^{(\widetilde{m_{t}}+2)}\right)\right|$,
which would allow us to make progress with Lemma \ref{lem:tied-unknown-edge-general}
by querying representatives to discover at least 1 new edge. Indeed,
since $\condense{G_{t}}$ is a DAG with $n_{t}$ vertices and $\widetilde{m_{t}}:=m\left(\condense{G_{t}}\right)$
satisfies $\widetilde{m_{t}}\le n_{t}-2$ as shown, Lemma \ref{lem:tied-candidates-after-finalization}
implies that $r^{(\widetilde{m_{t}}+1)}(\condense{G_{t}})=r^{(\widetilde{m_{t}}+2)}(\condense{G_{t}})$
which means $\left|\inreach_{\condense{G_{t}}}\left(C_{t}^{(\widetilde{m_{t}}+1)}\right)\right|=\left|\inreach_{\condense{G_{t}}}\left(C_{t}^{(\widetilde{m_{t}}+2)}\right)\right|$.
Lemma \ref{lem:tied-unknown-edge-general} implies that there is no
edge between any node in $C_{t}^{(\widetilde{m_{t}}+1)}$ and $C_{t}^{(\widetilde{m_{t}}+2)}$.
Hence, since Algorithm \ref{alg:non-transitive-tournament-graph-sort}
queries representatives from $C_{t}^{(\widetilde{m_{t}}+1)}$ and
$C_{t}^{(\widetilde{m_{t}}+2)}$, the oracle reveals at least one
new edge at round $t$ i.e., $|E_{t+1}|>|E_{t}|$. Since $|E_{t}|\leq\binom{n}{2}$,
the algorithm terminates in at most $\binom{n}{2}$ rounds.
\end{proof}

\begin{thm}[Correctness of Algorithm~\ref{alg:non-transitive-tournament-graph-sort}]
\label{thm:correctness-general} If Algorithm~\ref{alg:non-transitive-tournament-graph-sort}
terminates, then it outputs a valid top-$m$ solution to Problem~\ref{prob:main-inreach}. 
\end{thm}

\begin{rem}[Why the termination theorem assumes at least two SCCs]
The assumption $\left|\condense V_{G^{*}}\right|\ge2$ in Theorem~\ref{thm:termination-general-theory}
excludes the degenerate case where the underlying tournament $G^{*}$
consists of a single SCC. In that case every vertex reaches every
other vertex, so $|\inreach_{G^{*}}(v)|=n-1$ for all $v\in V$ by
Lemma~\ref{lem:inreach-within-scc}. Hence every subset of $m$ vertices,
in any order, is a valid solution to Problem~\ref{prob:main-inreach}.
The interesting case is therefore when the condensation has multiple
components, so that the SCC decomposition induces a nontrivial tiered
ranking.
\end{rem}

\begin{proof}
We want to show that once Algorithm \ref{alg:non-transitive-tournament-graph-sort}
terminates at some time $T$, the sequence $(v^{(1)},\ldots,v^{(m)})$
obtained by concatenating $C_{T}^{(1)},C_{T}^{(2)},\ldots,C_{T}^{(j^{*}-1)},S_{j^{*}}$
where $S_{j^{*}}$ be any subset of $C_{T}^{(j^{*})}$ with $|S_{j^{*}}|=m-\sum_{i=1}^{j^{*}-1}|C_{T}^{(i)}|$
with $j^{*}\leftarrow\min\{j\leq\widetilde{m_{T}}:\sum_{i=1}^{j}|C_{T}^{(i)}|\geq m\}$
satisfies a solution of Problem \ref{prob:main-inreach}. We verify
the two conditions of Problem \ref{prob:main-inreach} separately.
Internal ordering follows from the SCC structure of finalized components.
Rank dominance requires splitting into vertices inside vs. outside
the boundary SCC $C_{T}^{(j^{*})}$.

Let us first show internal ordering $\left|\inreach_{G^{*}}\left(v^{(i)}\right)\right|\le\left|\inreach_{G^{*}}\left(v^{(j)}\right)\right|$
for all $1\le i\le j\le m$ as in (\ref{eq:sol-set-satisfies-internal-ordering}).
Let $v^{(i)}\in C_{T}^{(i')}$ and $v^{(j)}\in C_{T}^{(j')}$ for
some $i',j'\in[j^{*}]$. Note that $i'\le j'\le j^{*}\le\widetilde{m_{T}}$
by the concatenation construction and definition of $j^{*}$. 

If $i'=j'$ then $v^{(i)}\sim_{G_{T}}v^{(j)}$ and thus $v^{(i)}\sim_{G^{*}}v^{(j)}$
by Lemma \ref{lem:scc-refinement}, which implies $\left|\inreach_{G^{*}}\left(v^{(i)}\right)\right|=\left|\inreach_{G^{*}}\left(v^{(j)}\right)\right|$
by Lemma \ref{lem:scc-equal-inreach}. 

Otherwise if $i'<j'$, then since $G_{T}$ is a subgraph of $G^{*}$
and $i'<j'\le m\left(\condense{G_{T}}\right)$, Lemma \ref{lem:finalized-order-correct}
implies that $\left|\inreach_{G^{*}}(u)\right|\le\left|\inreach_{G^{*}}(v)\right|$
for all $u\in C_{T}^{(i')}$ and $v\in C_{T}^{(j')}$. This means
that $\left|\inreach_{G^{*}}\left(v^{(i)}\right)\right|\le\left|\inreach_{G^{*}}\left(v^{(j)}\right)\right|$.
In either case, we have $\left|\inreach_{G^{*}}\left(v^{(i)}\right)\right|\le\left|\inreach_{G^{*}}\left(v^{(j)}\right)\right|$.

We will now show that the solution \emph{set} $S:=\left\{ v^{(1)},\ldots,v^{(m)}\right\} $
contains all the top ranks i.e., $\max\left\{ \inreach_{G^{*}}\left(v^{(i)}\right)\mid i=1,2,\dots,m\right\} \le\min_{v\not\in S}\inreach_{G^{*}}\left(v\right)$
as in (\ref{eq:rank-dominance-of-sol-set}). We want to show that
$\inreach_{G^{*}}\left(u\right)\le\inreach_{G^{*}}\left(v\right)$
for all $u\in S$ and $v\not\in S$. Note that $u$ is in one of $C_{T}^{(1)},C_{T}^{(2)},\ldots,C_{T}^{(j^{*}-1)},S_{j^{*}}$.
We will split the proof into 2 cases: (1) $v$ is in $C_{T}^{\left(j^{*}\right)}\setminus S_{j^{*}}$,
or (2) $v$is in $C_{T}^{(j)}$ for some $j>j^{*}$. 

\textbf{Case 1}: Suppose $v\in C_{T}^{\left(j^{*}\right)}\setminus S_{j^{*}}$. 

If $u\in S_{j^{*}}\subseteq C_{T}^{\left(j^{*}\right)}$, then $u$
and $v$ are in the same SCC i.e., $u\sim_{G_{T}}v$. Lemma \ref{lem:scc-refinement}
then implies that $u\sim_{G^{*}}v$ since $G_{T}$ is a subgraph of
$G^{*}$. Thus, Lemma \ref{lem:scc-equal-inreach} implies that $\left|\inreach_{G^{*}}(u)\right|=\left|\inreach_{G^{*}}(v)\right|$. 

Otherwise if $u\not\in S_{j^{*}}$, then $u\in C_{T}^{(i)}$ for some
$i<j^{*}$. Then Lemma \ref{lem:finalized-order-correct} implies
that $\left|\inreach_{G^{*}}(u)\right|\le\left|\inreach_{G^{*}}(v)\right|$
since $v\in C_{T}^{(j^{*})}$. Either cases, we have $\left|\inreach_{G^{*}}(u)\right|\le\left|\inreach_{G^{*}}(v)\right|$
as desired.

\textbf{Case 2}: Suppose $v\in C_{T}^{(j)}$ for some $j>j^{*}$.
Note that $u\in C_{T}^{(i)}$ for some $i\le j^{*}\le\widetilde{m_{T}}$.
Since $j^{*}<j$, we have $i<j$. 

If $j\le\widetilde{m_{T}}=m\left(\condense{G_{T}}\right)$, then since
$i<j\le m\left(\condense{G_{T}}\right)$ and $G_{T}$ is a subgraph
of $G^{*}$, Lemma \ref{lem:finalized-order-correct} implies that
$\left|\inreach_{G^{*}}(u)\right|\le\left|\inreach_{G^{*}}(v)\right|$. 

For the other subcase, if $j>m\left(\condense{G_{T}}\right)$, we
have $C_{T}^{(j)}\not\in\mathrm{TOP}\left(\condense{G_{T}}\right)$.
Since $C_{T}^{(i)}\in\mathrm{TOP}\left(\condense{G_{T}}\right)$ and
$G_{T}$ is a subgraph of $G^{*}$, Lemma \ref{lem:general-finalized-reach-candidates}
implies that $\left|\inreach_{G^{*}}(u)\right|\le\left|\inreach_{G^{*}}(v)\right|$.
Either cases, we have $\left|\inreach_{G^{*}}(u)\right|\le\left|\inreach_{G^{*}}(v)\right|$
as desired. 

Combining case 1 and 2, we have $\left|\inreach_{G^{*}}(u)\right|\le\left|\inreach_{G^{*}}(v)\right|$
for all $u\in S$ and $v\not\in S$. Thus (\ref{eq:rank-dominance-of-sol-set})
is satisfied. 

Thus, since the sequence $(v^{(1)},\ldots,v^{(m)})$ output by Algorithm~\ref{alg:non-transitive-tournament-graph-sort}
satisfies both (\ref{eq:sol-set-satisfies-internal-ordering}) and
(\ref{eq:rank-dominance-of-sol-set}) of Problem \ref{prob:main-inreach},
the algorithm is correct. 
\end{proof}

\section{Resolution-Based Finalization and \tournamentsort{}}\label{sec:stronger-finalization-general}
Algorithms~\ref{alg:transitive-tournament-sort}
and~\ref{alg:non-transitive-tournament-graph-sort}
terminate when the finalization threshold of the condensation reaches~$m$,
requiring the rank spectrum to crystallize into the pattern
$0,1,2,\ldots$ up to position~$m$.
While this criterion is theoretically clean, it is more elaborate than
what is needed for a practical implementation.
In this section, we develop the vertex-level criterion used by
Algorithm~\ref{alg:tournament-sort-main}, based on the
\emph{known-relationship count}
\[
\kappa_G(v):=|K_G(v)|,
\qquad
K_G(v):=\inreach_G(v)\cup\outreach_G(v).
\]
A vertex is \emph{resolved} when $\kappa_G(v)=n-1$, meaning that its
relationship to every other vertex is already determined by the revealed
graph.

The key point is that resolution is sufficient for the practical stopping
rule: \tournamentsort{} terminates once the current top-$m$ vertices by
discovered in-reach are all resolved.
This criterion is simpler to check than crystallization of the full rank
spectrum, while still yielding a correct output.
We first relate resolution to the earlier finalization framework, then
prove output correctness from a basic ordering property of resolved
vertices, and finally use the condensation only in the progress and
termination argument.

\subsection{Properties of Known Relationship }
\begin{prop}[Known Relationships As a Necessary Condition for Finalization]
\label{prop:known-relationship-finalization} Let $G=(V,E)$ be a
DAG with $m(G)\ge1$. Then we have 
\[
v\in\mathrm{TOP}(G)\implies\left|\inreach_{G}(v)\cup\outreach_{G}(v)\right|=n-1.
\]
\end{prop}

\begin{proof}
We write $v^{(j)}:=v^{(j)}(G)$ for brevity. Let $v=v^{(j)}$ for
some $j\leq m(G)$ i.e., $v\in\mathrm{TOP}(G)$. We show that every
vertex in $V\setminus\{v\}$ belongs to either $\inreach_{G}(v)$
or $\outreach_{G}(v)$. First, by Lemma \ref{lem:exact-inreach-finalized},
we know exactly the in-reaches of $v$: 
\[
\inreach_{G}(v)=\begin{cases}
\emptyset, & \text{if \ensuremath{j=1}},\\
\{v^{(1)},\ldots,v^{(j-1)}\}, & \text{if \ensuremath{j>1}}.
\end{cases}
\]
Now, for any $v^{(i)}\in\mathrm{TOP}(G)$ with $i>j$, Corollary \ref{cor:mutual-reach-finalized}
gives $v^{(j)}\leadsto_{G}v^{(i)}$, so: 
\[
\{v^{(j+1)},\ldots,v^{(m(G))}\}\subseteq\outreach_{G}(v).
\]
Finally, by Lemma \ref{lem:finalized-reach-candidates}, for any $u\in\mathrm{CAND}(G)=V\setminus\mathrm{TOP}(G)$,
we have $v\leadsto_{G}u$, so: 
\[
\mathrm{CAND}(G)\subseteq\outreach_{G}(v).
\]
The sets $\inreach_{G}(v)$ and $\outreach_{G}(v)$ are disjoint (since
$G$ is a DAG and any vertex in their intersection would create a
cycle through $v$). Combining the results above: 
\begin{align*}
\inreach_{G}(v) & =\{v^{(1)},\ldots,v^{(j-1)}\},\\
\outreach_{G}(v) & \supseteq\{v^{(j+1)},\ldots,v^{(m(G))}\}\cup\mathrm{CAND}(G).
\end{align*}
These sets partition $V\setminus\{v\}$: 
\[
V\setminus\{v\}=\{v^{(1)},\ldots,v^{(j-1)}\}\cup\{v^{(j+1)},\ldots,v^{(m(G))}\}\cup\mathrm{CAND}(G).
\]
Therefore: 
\[
\inreach_{G}(v)\cup\outreach_{G}(v)\supseteq V\setminus\{v\}.
\]
Trivially, we have $\inreach_{G}(v)\cup\outreach_{G}(v)\subseteq V\setminus\{v\}$.
We conclude: 
\[
\left|\inreach_{G}(v)\cup\outreach_{G}(v)\right|=n-1.\qedhere
\]
\end{proof}
\begin{rem}
The converse of Proposition~\ref{prop:known-relationship-finalization}
does not hold; see the remark following Proposition~\ref{prop:characterization-finalized}
for a counterexample and the additional condition that yields a full
characterization. 
\end{rem}

We can obtain a characterization of the known relationships by strengthening
the assumption a bit. This helps with understanding the finalization
criterion and will be useful for implementing a practical version
of the algorithm.
\begin{prop}[Characterization of Finalized Vertices]
\label{prop:characterization-finalized} Let $G=(V,E)$ be a DAG
with $m(G)\geq1$. Then 
\[
v\in\mathrm{TOP}(G)\;\iff\;\kappa_{G}(v)=n-1\;\text{ and }\;|\inreach_{G}(v)|<m(G).
\]
\end{prop}

\begin{proof}
$(\Rightarrow)$ Let $v=v^{(j)}(G)$ for some $j\leq m(G)$. By Proposition~\ref{prop:known-relationship-finalization},
$\kappa_{G}(v)=n-1$. By the finalization threshold (Definition~\ref{def:finalization-general-graph}),
$|\inreach_{G}(v)|=r^{(j)}(G)=j-1<m(G)$.

$(\Leftarrow)$ Suppose $|\inreach_{G}(v)|<m(G)$. By the contrapositive
of Lemma~\ref{lem:candidate-large-inreach}, $v\notin\mathrm{CAND}(G)$,
so $v\in\mathrm{TOP}(G)$. 
\end{proof}
\begin{rem}
The in-reach condition $|\inreach_{G}(v)|<m(G)$ alone suffices for
the backward direction; the known-relationship condition $\kappa_{G}(v)=n-1$
is redundant for that implication. We retain it in the biconditional
because it is the computationally natural quantity to track: the algorithm
accumulates known relationships incrementally, and the in-reach condition
provides a cheap secondary check once resolution is established.

Note also that the converse of Proposition~\ref{prop:known-relationship-finalization}
fails in general: $\kappa_{G}(v)=n-1$ does not imply $v\in\mathrm{TOP}(G)$.
Consider $V=\{a,b,c\}$ with $E=\{(a,c),(b,c)\}$. Here $m(G)=1$
with $\mathrm{TOP}(G)=\{a\}$, yet $\kappa_{G}(c)=2=n-1$. Proposition~\ref{prop:characterization-finalized}
correctly excludes $c$ because $|\inreach_{G}(c)|=2\geq m(G)=1$,
violating the in-reach condition. 
\end{rem}

\subsection{Correctness of Algorithm~\ref{alg:tournament-sort-main} }\label{subsec:practical-correctness}

We now establish correctness and termination of Algorithm~\ref{alg:tournament-sort-main}.
The key observation is that the output validity depends on a single
elementary property of \emph{resolved} vertices, while the condensation
machinery is needed only for proving termination. 

\subsubsection{Output Correctness}

Recall that the algorithm computes, for each vertex $u\in V$ in the
revealed graph $G=(V,E)$ with $E\subseteq E^{*}$, the loss count
$\left|\inreach_{G}(u)\right|$ and the known relationship $\kappa_{G}(u)$.
The algorithm terminates when the set $T$ of $m$ vertices with smallest
discovered in-reach is contained in
$F:=\{u\in V:\kappa_G(u)=n-1\}$.
\begin{lem}[Resolved Vertex Ordering]
\label{lem:resolved-vertex-ordering} Let $G^{*}=(V,E^{*})$ be a
tournament and $G=(V,E)$ a subgraph with $E\subseteq E^{*}$. Let
$u,v\in V$ be distinct. If $u$ is resolved (i.e., $\kappa_{G}(u)=n-1$)
and $v\notin\inreach_{G}(u)$, then 
\[
|\inreach_{G^{*}}(u)|\le|\inreach_{G^{*}}(v)|.
\]
\end{lem}

\begin{proof}
Since $\inreach_{G}(u)\cup\outreach_{G}(u)=V\setminus\{u\}$ and $v\notin\inreach_{G}(u)$,
we have $v\in\outreach_{G}(u)$, so $u\leadsto_{G}v$ and hence $u\leadsto_{G^{*}}v$
(as $E\subseteq E^{*}$). In $G^{*}$, exactly one of the following
holds: 
\begin{enumerate}
\item $v\leadsto_{G^{*}}u$ as well. Then $u\sim_{G^{*}}v$, so $u$ and
$v$ belong to the same SCC of $G^{*}$, giving $|\inreach_{G^{*}}(u)|=|\inreach_{G^{*}}(v)|$
by Lemma~\ref{lem:scc-equal-inreach}. 
\item $v\not\leadsto_{G^{*}}u$. Then $[u]_{G^{*}}$ and $[v]_{G^{*}}$
are distinct SCCs with $[u]_{G^{*}}\leadsto_{\condense{G^{*}}}[v]_{G^{*}}$.
Since the condensation of a tournament is a transitive tournament
(Proposition~\ref{prop:condensation-transitive}), $[u]_{G^{*}}$
strictly precedes $[v]_{G^{*}}$ in condensation order. Lemma~\ref{lem:inreach-within-scc}
then gives $|\inreach_{G^{*}}(u)|<|\inreach_{G^{*}}(v)|$. 
\end{enumerate}
In both cases, $|\inreach_{G^{*}}(u)|\le|\inreach_{G^{*}}(v)|$. 
\end{proof}
The following observation connects the algorithm's vertex-level quantities
to the hypothesis of Lemma~\ref{lem:resolved-vertex-ordering}.
\begin{lem}[Smaller Discovered Loss Excludes In-Reach Membership]
\label{lem:loss-excludes-inreach} Let $G=(V,E)$ be a directed graph.
Let $u,v\in V$ be distinct. If $u$ is resolved (i.e., $\kappa_{G}(u)=n-1$)
and $\left|\inreach_{G}(u)\right|\le\left|\inreach_{G}(v)\right|$,
then $v\notin\inreach_{G}(u)$. 
\end{lem}

\begin{proof}
Suppose for contradiction that $v\in\inreach_{G}(u)$. Lemma~\ref{lem:inreach-downward-closure}
gives $\inreach_{G}(v)\cup\{v\}\subseteq\inreach_{G}(u)$, so $|\inreach_{G}(v)|+1\le|\inreach_{G}(u)|$,
i.e., $\left|\inreach_{G}(v)\right|<\left|\inreach_{G}(u)\right|$,
contradicting $\left|\inreach_{G}(u)\right|\le\left|\inreach_{G}(v)\right|$. 
\end{proof}
Combining these 2 lemmas provide an important bridge from the resolved
vertices' in-reaches of the discovered graph to the in-reaches of
the underlying tournament. 
\begin{cor}[Resolved In-Reaches are True In-Reaches]
\label{cor:resolved-inreaches-are-true}Let $G^{*}=(V,E^{*})$ be
a tournament and $G=(V,E)$ a subgraph with $E\subseteq E^{*}$. Let
$u,v\in V$ be distinct. We have, 
\[
\kappa_{G}(u)=n-1\text{ and }\left|\inreach_{G}(u)\right|\le\left|\inreach_{G}(v)\right|\implies|\inreach_{G^{*}}(u)|\le|\inreach_{G^{*}}(v)|.
\]
\end{cor}

\begin{proof}
The proof immediately follows from combining Lemma \ref{lem:resolved-vertex-ordering}
and Lemma \ref{lem:loss-excludes-inreach}.
\end{proof}
\begin{thm}[Output Correctness of Algorithm~\ref{alg:tournament-sort-main}]
\label{thm:practical-output-correctness} If Algorithm~\ref{alg:tournament-sort-main}
terminates, its output is a valid solution to Problem~\ref{prob:main-inreach}. 
\end{thm}

\begin{proof}
At termination with discovered graph $G\subseteq G^{*}$, $T\subseteq F$,
meaning every $u\in T$ satisfies $\kappa_{G}(u)=n-1$, and $T$ consists
of the $m$ vertices with smallest in-reach.

\paragraph{Internal ordering~(\ref{eq:sol-set-satisfies-internal-ordering}).}

Let $v^{(1)},\ldots,v^{(m)}$ be the vertices of $T$ sorted by ascending
in-reach. For any $1\le i<j\le m$, we have $\left|\inreach_{G}(v^{(i)})\right|\le\left|\inreach_{G}(v^{(j)})\right|$
with $\kappa_{G}(v^{(i)})=n-1$. Then Corollary \ref{cor:resolved-inreaches-are-true}
implies that $\left|\inreach_{G^{*}}(v^{(i)})\right|\le\left|\inreach_{G^{*}}(v^{(j)})\right|$.

\paragraph{Rank dominance~(\ref{eq:rank-dominance-of-sol-set}).}

Let $u\in T$ and $v\notin T$. Since $T$ contains the $m$ vertices
with smallest in-reaches, we have $\left|\inreach_{G}(u)\right|\le\left|\inreach_{G}(v)\right|$.
Since $\kappa_{G}(u)=n-1$, Corollary \ref{cor:resolved-inreaches-are-true}
implies that $|\inreach_{G^{*}}(u)|\le|\inreach_{G^{*}}(v)|$. Taking
the maximum over $u\in T$ and minimum over $v\notin T$: 
\[
\max_{u\in T}|\inreach_{G^{*}}(u)|\;\le\;\min_{v\notin T}|\inreach_{G^{*}}(v)|.\qedhere
\]
\end{proof}
\begin{rem}
Lemma~\ref{lem:resolved-vertex-ordering} is one-sided: only $u$
needs to be fully resolved ($\kappa_G(u)=n{-}1$); no assumption is made
on $v$. This is why the proof requires only $T\subseteq F$ and places
no constraint on vertices outside~$T$. 
\end{rem}

\subsubsection{Termination}

Output correctness requires no condensation machinery. Termination,
however, uses the condensation to guarantee that every non-terminal
round discovers at least one new edge.
\begin{lem}[Tied SCCs Contain Unresolved Vertices]
\label{lem:tied-sccs-unresolved} Let $G=(V,E)$ be a directed graph.
If $C,D\in\condense V$ are distinct SCCs with $|\inreach_{\condense G}(C)|=|\inreach_{\condense G}(D)|$,
then for any $u\in C$ and $v\in D$, we have $\kappa_{G}(u)<n-1$
and $\kappa_{G}(v)<n-1$. 
\end{lem}

\begin{proof}
By Lemma~\ref{lem:tied-unknown-edge-general}, there is no edge between
any vertex in $C$ and any vertex in $D$ in $E$. In particular,
for any $u\in C$ and $v\in D$: $v\notin\inreach_{G}(u)\cup\outreach_{G}(u)$,
so $\kappa_{G}(u)\le n-2<n-1$. The same argument gives $\kappa_{G}(v)<n-1$.
\end{proof}
\begin{thm}[Termination of Algorithm~\ref{alg:tournament-sort-main}]
\label{thm:practical-termination} Algorithm~\ref{alg:tournament-sort-main}
terminates in at most $\binom{n}{2}$ rounds. 
\end{thm}

\begin{proof}
At any non-terminal round, $T\not\subseteq F$: some vertex $u\in T$
has $\kappa_{G}(u)<n-1$. We show the scheduled query $Q$ discovers
at least one new edge.

Let $\widetilde{m}:=m(\condense G)$ denote the finalization threshold
of $\condense G$, and let $n':=|\condense V|$. We first establish
$\widetilde{m}\le n'-2$. If $\widetilde{m}=n'$, then every SCC is
finalized, and Proposition~\ref{prop:known-relationship-finalization}
applied to the DAG $\condense G$ gives $\kappa_{\condense G}(C)=n'-1$
for every $C\in\mathrm{TOP}(\condense G)=\condense V$. This implies
$\kappa_{G}(u)=n-1$ for all $u\in V$ by Lemma \ref{lem:known-rel-condensation-iff},
contradicting the existence of an unresolved vertex. Hence $\widetilde{m}<n'$.
Lemma~\ref{lem:no-penultimate} then gives $\widetilde{m}\ne n'-1$,
so $\widetilde{m}\le n'-2$.

By Lemma~\ref{lem:tied-candidates-after-finalization} applied to
the DAG $\condense G$, the SCCs at positions $\widetilde{m}+1$ and
$\widetilde{m}+2$ in the condensation rank spectrum are tied: 
\[
|\inreach_{\condense G}(C^{(\widetilde{m}+1)})|=|\inreach_{\condense G}(C^{(\widetilde{m}+2)})|.
\]
By Lemma~\ref{lem:tied-sccs-unresolved}, both SCCs contain vertices
with $\kappa_{G}(\cdot)<n-1$, so both appear in the candidate set
$\mathcal{C}$ of the algorithm. Since they share the same in-reach
value in $\condense G$, they are among the earliest elements of $\mathcal{C}$
(sorted by ascending in-reach). The algorithm selects representatives
from each, and by Lemma~\ref{lem:tied-unknown-edge-general}, no
edge exists between these representatives in $E$. The oracle query
therefore reveals at least one new edge: $|E_{t+1}|>|E_{t}|$.

Since $|E_{t}|\le\binom{n}{2}$ for all $t$, the algorithm terminates
in at most $\binom{n}{2}$ rounds. 
\end{proof}
\begin{rem}[Comparison with Algorithms~\ref{alg:transitive-tournament-sort}
and~\ref{alg:non-transitive-tournament-graph-sort}]
\label{rem:practical-vs-theoretical} The theoretical algorithms
terminate when the finalization threshold of the condensation reaches
$m$, requiring the rank spectrum to crystallize into the pattern
$0,1,2,\ldots$ up to position $m$. Algorithm~\ref{alg:tournament-sort-main}
terminates when the $m$ lowest-loss vertices are individually resolved
($\kappa_{G}(\cdot)=n-1$), a vertex-level condition that can be satisfied
strictly earlier than full crystallization. The condensation enters
Algorithm~\ref{alg:tournament-sort-main} only in the scheduling
step (for ordering candidate SCCs) and in the termination proof (for
guaranteeing progress). The output correctness proof is independent
of the condensation and relies solely on Lemma~\ref{lem:resolved-vertex-ordering}. 
\end{rem}

\section{Query Complexity Discussion}\label{sec:query-complexity}

\subsection{Setup and Notation}\label{subsec:qc-setup}
We measure \emph{query complexity} as the number of calls made to the
$k$-wise oracle. Formally, an algorithm $\mathcal{A}$ adaptively
selects query sets $Q_{1},Q_{2},\ldots$ with $Q_{t}\subseteq V$ and
$|Q_{t}|\le k$, observes $\mathcal{O}(Q_{t})$, and terminates once it can
certify the desired output (top-$m$ selection).

\paragraph{Instance-wise and worst-case complexity.}
For a fixed underlying tournament $G^{*}$, let $T_{\mathcal{A}}(G^{*})$
denote the (random) number of oracle calls made by $\mathcal{A}$ until
termination.\footnote{In this paper the oracle is deterministic; randomness, if any, comes from tie-breaking or implementation choices.}
We define the worst-case query complexity over a class $\mathcal{G}$ as
\[
Q_{\mathcal{A}}(\mathcal{G};n,k,m):=\sup_{G^{*}\in\mathcal{G}}T_{\mathcal{A}}(G^{*}).
\]
In particular, we will distinguish the transitive class
$\mathcal{T}_{n}^{\mathrm{tr}}$ (total orders) from the general class
$\mathcal{T}_{n}$ (all tournaments).

\paragraph{Sequential vs.\ parallel depth.}
The complexity above counts total oracle calls.
If multiple disjoint query sets are evaluated in parallel, a second
measure is the number of \emph{rounds} (parallel depth). We defer a
formal treatment of parallel depth to future work.

\subsection{Transitive Case}\label{subsec:transitive-query-complexity}

The $\binom{n}{2}$ bound from Theorem~\ref{thm:termination-transitive} is a
worst-case upper bound that does not exploit the structure of greedy
selection. We now derive tighter bounds.

\begin{prop}[Top-1 Query Complexity]
\label{thm:top-1-complexity} When $m=1$, Algorithm \ref{alg:transitive-tournament-sort}
identifies the top vertex in at most $\left\lceil\frac{n-1}{k-1}\right\rceil$ queries.
\end{prop}

\begin{proof}
In the transitive setting, the restriction of $G^{*}$ to any queried
set $Q$ is a total order, hence it has a unique best element (the one
with zero losses within $Q$). By transitivity, any other element of $Q$
loses to this best element and therefore cannot be the global best.
Thus each query can eliminate at least $|Q|-1\ge k-1$ candidates from
being the top vertex. Starting from $n$ candidates, after $q$ queries at
most $n-q(k-1)$ candidates remain; requiring $n-q(k-1)\le1$ yields the
bound.
\end{proof}

\subsection{A Conjecture for Top-$m$ Query Complexity}\label{subsec:qc-conjecture}
Beyond $m=1$, a tight query complexity analysis remains open.
The proved termination bound is $O(n^{2})$
(Theorem~\ref{thm:practical-termination}), obtained by showing that each
non-terminal round reveals at least one new edge.
This existential progress argument does not exploit the greedy structure
of the scheduler, which empirically produces query counts far below
$O(n^{2})$.
For $m=1$, Proposition~\ref{thm:top-1-complexity} gives a tight
$O(n/k)$ bound via a clean elimination argument: each query eliminates
$k-1$ candidates.
For $m>1$, the algorithm must simultaneously refine multiple rank
positions, and the interaction between candidate elimination and frontier
advancement resists a simple recurrence.
We state a conjectural bound motivated by the structure of the algorithm
and validated empirically in Section~\ref{subsec:qc-empirical-results}.

\begin{conjecture}[Top-$m$ Query Complexity]
\label{conj:query-complexity}
Let $\tournamentsort^{\dagger}$ denote Algorithm~\ref{alg:tournament-sort-main}
with candidate SCCs ordered lexicographically by ascending condensation
in-reach, then ascending condensation out-reach, then fixed ID;
and with each $\mathrm{rep}(C)$ chosen by minimum~$\kappa_G$, then fixed ID.
Then
\[
Q_{\text{\tournamentsort}^{\dagger}}(\mathcal{T}_{n};\,n,k,m)
= O\!\left(\frac{n-1}{k-1} + \frac{m-1}{k-1}\,\log_k m\right).
\]
\end{conjecture}

\begin{rem}[Decomposition]
The conjectured form decomposes into two terms:
\begin{itemize}
\item A \emph{candidate reduction} term $(n-1)/(k-1)$: the cost to winnow
$n$ candidates down to a single maximum, matching Proposition~\ref{thm:top-1-complexity}.
\item A \emph{frontier refinement} term $(m-1)/(k-1)\cdot\log_{k}m$:
the additional cost to extract elements $2,\ldots,m$ from the remaining
candidates, reflecting the logarithmic depth of repeated selection.
\end{itemize}
The $O(\cdot)$ notation hides a constant factor.
Empirically (Section~\ref{subsec:qc-empirical-results}), we observe that
this constant is at most $1.25$ across all tested configurations,
suggesting the functional form is tight up to lower-order terms.
The specific tie-breaking rules in $\tournamentsort^{\dagger}$ pin down
a single deterministic algorithm, avoiding ambiguity from arbitrary
implementation choices.
\end{rem}

\begin{rem}[Lower bounds]\label{rem:lower-bound}
Standard arguments show the conjectured form is tight up to constants.
Elimination gives $\Omega((n{-}1)/(k{-}1))$ for finding the maximum
(Proposition~\ref{thm:top-1-complexity}); a decision-tree bound gives
$\Omega(\log(m!)/\log(k!))=\Omega(m\log m/(k\log k))$ for ordering
$m$ items with $k!$-outcome queries.
When $m=n$ the latter recovers the classical
$\Omega(n\log n/(k\log k))$ sorting lower bound.
Since $\max(a,b)=\Theta(a+b)$ for non-negative $a,b$, the maximum of
the two bounds already matches the conjectured
$(n{-}1)/(k{-}1)+(m{-}1)/(k{-}1)\cdot\log_k m$ up to constant
factors; the open challenge is the upper bound.
\end{rem}

\begin{rem}[Difficulty of a formal proof]\label{rem:conjecture-difficulty}
The gap between the proved $O(n^{2})$ termination bound and
Conjecture~\ref{conj:query-complexity} reflects a fundamental limitation of
the existing analysis, which is built around \emph{correctness and eventual
progress} rather than a sharp accounting of how quickly the greedy scheduler
resolves the top-$m$ frontier.
In Algorithm~\ref{alg:tournament-sort-main}, the condensation is used only
to order candidate SCCs and to guarantee that every non-terminal iteration
reveals at least one new edge
(Theorem~\ref{thm:practical-termination}) -- far weaker than the conjectured
scaling.
Proving the stronger bound would require a new structural invariant showing
that repeated $k$-wise queries not only make progress, but concentrate
information efficiently on the unresolved top-$m$ boundary.

Two structural mismatches compound the difficulty.
First, the algorithm schedules queries at the SCC level of the condensation,
while termination is a vertex-level condition ($\kappa_G(v)=n-1$);
bridging these levels requires tracking how SCC-level queries translate into
vertex-level resolution.
Second, transitive-closure effects make the dynamics nonlocal: a single new
edge can trigger cascading updates to in-reach sets and SCC structure,
resolving vertices far from the query set.
The 25-horses example (Figure~\ref{fig:horses}) illustrates a third
difficulty: the per-query information gain is inherently non-uniform.
The execution exhibits three distinct phases:
(i)~a \emph{skeleton-building} phase (Rounds~1--5) where disjoint groups
are sorted independently and each query reveals exactly $k-1$ new edges;
(ii)~a \emph{connection} phase (Round~6) where group winners are compared,
transitive closure propagates across previously isolated components, and
vertices resolve en masse;
(iii)~a \emph{saturation} phase (Round~7) where the remaining top
candidates require targeted comparisons with diminishing transitive-closure
leverage.
Any tight analysis must likely account for these phase transitions, as no single
amortization rate can capture the varying progress.
Establishing the conjecture would likely require a phase-aware potential
function that tracks the rate at which the top-$m$ frontier
crystallizes -- a substantial extension of the present framework.
\end{rem}

\subsection{Why the Bound Should Extend to All Tournaments}\label{subsec:qc-nontransitive}

Conjecture~\ref{conj:query-complexity} states the bound over the full class $\mathcal{T}_n$.
The following remark provides structural intuition for why non-transitive tournaments should not be harder than transitive ones.

\begin{rem}[Why Cycles Help]\label{rem:why-cycles-help}
The intuition is that cycles can only make top-$m$ selection \emph{easier}.
When vertices form a cycle in $G^*$, they belong to the same SCC and thus share a common rank in the condensation DAG.
Once the algorithm discovers such a cycle, all vertices in the SCC finalize simultaneously -- they need not be distinguished from one another.
In contrast, the transitive case requires establishing pairwise orderings among all vertices: no shortcuts exist.

More precisely, for any tournament $G^*$, one can construct a ``shadow'' transitive tournament $\tilde{G}$ that respects the SCC tier structure of $G^*$: vertices in better-ranked tiers beat vertices in worse-ranked tiers, with intra-tier orderings assigned arbitrarily.
Inter-tier edges in $\tilde{G}$ are determined by the same tier structure as in $G^*$, so any query sequence reveals the same inter-tier relationships in both.
The difference lies within tiers: in $\tilde{G}$, all vertices are singletons requiring full resolution; in $G^*$, cycles merge vertices into SCCs that finalize together.
Thus the query count on $G^*$ should never exceed that on $\tilde{G}$.
A formal proof of this reduction remains open.
\end{rem}

\begin{figure*}[ht]
\centering
\includegraphics[width=\linewidth]{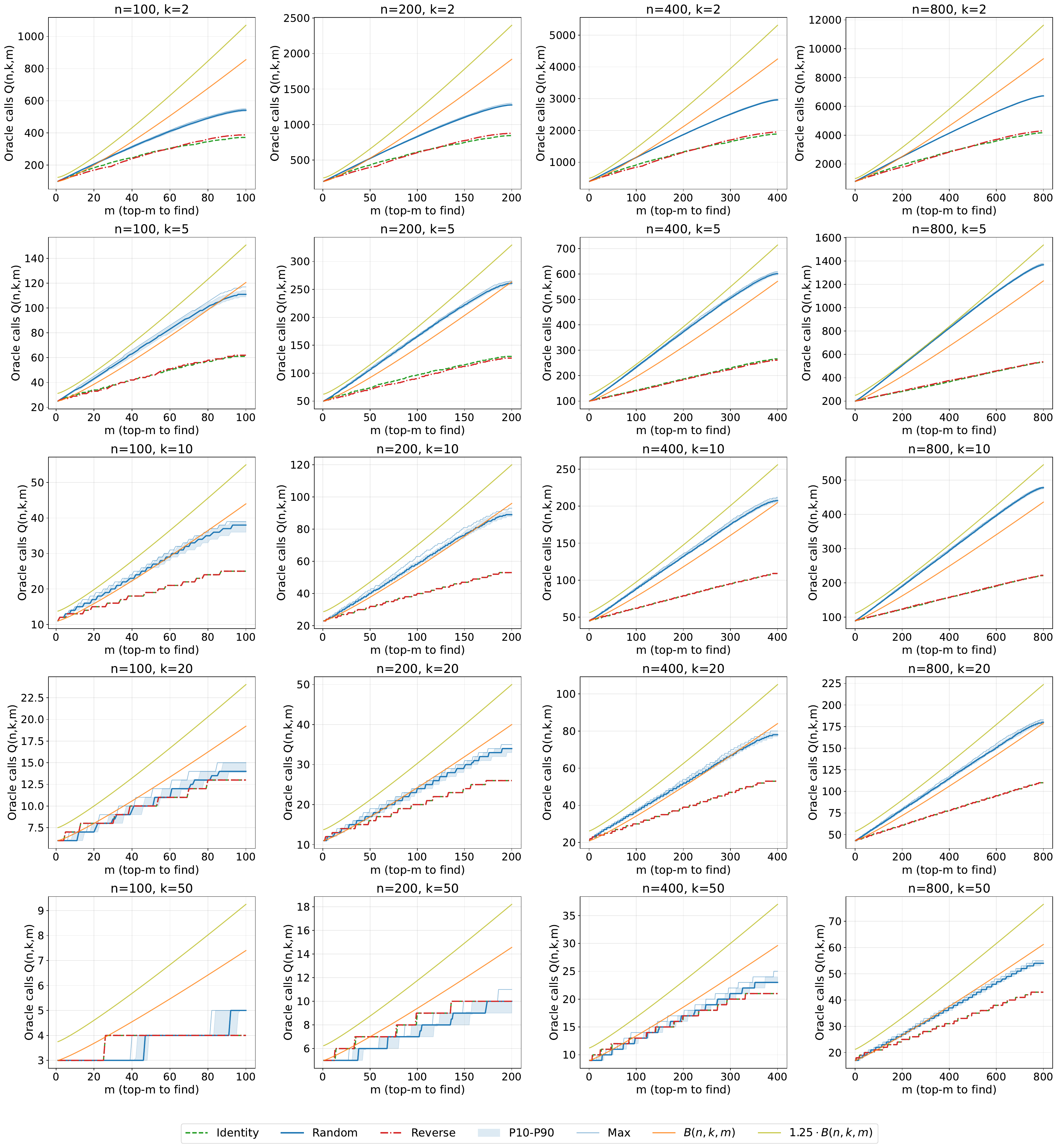}
\caption{Empirical query counts $Q(n,k,m)$ versus target count $m$ for three initial orderings: random (blue, 20 seeds with P10--P90 band and max), identity (green, sorted input), and reverse (red, reverse-sorted input). The conjectured bound $B(n,k,m)$ (orange) closely tracks the empirical complexity; $1.25 \cdot B(n,k,m)$ (yellow-green) upper-bounds all observed counts. Identity and reverse permutations lie uniformly below the random baseline.}
\label{fig:q-vs-m-grid}
\end{figure*}

\subsection{Empirical Methodology}\label{subsec:qc-empirical-methodology}

We empirically validate Conjecture~\ref{conj:query-complexity} by measuring
the query complexity of $\tournamentsort^{\dagger}$ on synthetic transitive
instances -- the conjectured worst case
(Remark~\ref{rem:why-cycles-help}).

\paragraph{Instance generation.}
A transitive tournament on $n$ vertices corresponds to a total order,
uniquely determined by a permutation $\pi:[n]\to[n]$.
The underlying tournament $G^*_\pi$ has edge $u\to v$ iff $\pi^{-1}(u)<\pi^{-1}(v)$.
The oracle $\mathcal{O}$ applied to any query set $Q\subseteq V$
returns the restriction of $G^*_\pi$ to $Q$ -- equivalently, the
total order on $Q$ induced by $\pi$.

We test three classes of initial orderings:
\begin{enumerate}[nosep]
\item \emph{Random}: $\pi$ sampled uniformly at random (20 independent seeds per $(n,k)$ pair).
\item \emph{Identity}: $\pi = \mathrm{id}$, i.e.\ items presented in sorted order.
      The best items appear first in the input, giving the algorithm early access to them.
\item \emph{Reverse}: $\pi(i) = n - i$, i.e.\ items presented in reverse sorted order.
      The best items appear last, forcing the algorithm to scan the entire input
      before encountering them.
\end{enumerate}
The identity and reverse permutations are deterministic (one run per $(n,k)$ pair).

\paragraph{Measurement protocol.}
For each parameter triple $(n,k,m)$, we execute $\tournamentsort^{\dagger}$
with target count $m$ and record $T_{\mathcal{A}}(G^*_\pi)$, the number of
oracle calls until termination.
Rather than running separate experiments for each $m\in\{1,\ldots,n\}$,
we exploit a key observation: a single execution with $m=n$ (full sort)
reveals $Q(n,k,m')$ for all $m'\leq n$ as a byproduct.
Specifically, when the $m'$-th vertex is finalized at round $t$,
we record $T_{m'}:=t$ as the query count for top-$m'$ selection.
This yields the full curve $m\mapsto Q(n,k,m)$ from one run.

\paragraph{Experimental grid.}
We test $n \in \{100, 200, 400, 800\}$ and $k \in \{2, 5, 10, 20, 50\}$, extracting $m = 1, \ldots, n$ from each full-sort run. For each $(n, k, m)$, we report the median, 10th--90th percentile range, and maximum of $T_m$ across seeds (for random permutations) or the single observed value (for identity and reverse).

\subsection{Empirical Results}\label{subsec:qc-empirical-results}

Figure~\ref{fig:q-vs-m-grid} presents the empirical query counts
alongside the concrete reference curve
\begin{equation}\label{eq:conjectured-bound}
B(n,k,m) \;=\; \left\lceil\frac{n-1}{k-1}\right\rceil
+ \frac{m-1}{k-1}\cdot\left(1+\log_k(m)\right),
\end{equation}
a non-asymptotic instantiation of the $O(\cdot)$ form in
Conjecture~\ref{conj:query-complexity} (the additive~$1$ inside the
parenthesis is absorbed by the big-$O$).

\paragraph{Main observations.}
\begin{enumerate}[nosep]
\item \emph{The conjectured form captures the scaling.}
Across all tested $(n,k)$ pairs, the empirical curves follow the
shape predicted by $B(n,k,m)$: sublinear in $k$, with a candidate-reduction
plateau for small $m$ transitioning to logarithmic growth.

\item \emph{The constant is small.}
The maximum observed query count satisfies
$\max_\pi T_{\mathcal{A}}(G^*_\pi) \leq 1.25\cdot B(n,k,m)$
for all tested configurations.
The median typically lies below $B(n,k,m)$ itself.

\item \emph{Variance is low.}
The 10th--90th percentile bands are narrow, indicating that
query complexity is stable across random permutations and
not dominated by rare worst-case instances within our sample.

\item \emph{Structured orderings are easier than random.}
The identity and reverse permutations -- representing the two extremes of
initial ordering -- produce query counts that lie uniformly below
the random baseline across all tested $(n,k,m)$ configurations
(Figure~\ref{fig:q-vs-m-grid}).
Both stay well within $B(n,k,m)$ itself, far from the $1.25$ envelope.
This suggests that random permutations, which maximally mix rank positions,
are harder for the algorithm than structured orderings.
\end{enumerate}

\begin{table}[ht]
\centering
\caption{Full-sort query counts ($m{=}n$, $k{=}10$) for three initial orderings. Random reports [min, max] over 20 seeds.}
\label{tab:adversarial-counts}
\small
\begin{tabular}{rrrcc}
\toprule
$n$ & Identity & Reverse & Random [min, max] & $1.25 \cdot B(n,k,n)$ \\
\midrule
100 & 25 & 25 & [35, 39] & 55 \\
200 & 53 & 53 & [86, 89] & 120 \\
400 & 109 & 109 & [204, 212] & 255.9 \\
800 & 221 & 222 & [470, 480] & 544.4 \\
\bottomrule
\end{tabular}
\end{table}

Table~\ref{tab:adversarial-counts} summarizes full-sort query counts ($m{=}n$) at $k{=}10$.
The gap between structured and random permutations widens as $n$ increases: at $n{=}800$, identity and reverse require 221--222 queries while random requires 470--480 -- all well within the $1.25 \cdot B(n,k,n)$ envelope.
Notably, identity and reverse permutations yield nearly identical query counts despite representing opposite extremes of the initial ordering, suggesting the algorithm's efficiency depends on the inherent structure of the permutation rather than its alignment with the target ranking.

\paragraph{Limitations.}
These experiments provide evidence for, but do not prove,
Conjecture~\ref{conj:query-complexity}. Key caveats:
\begin{itemize}[nosep]
\item We test identity, reverse, and random permutations; a worst-case
permutation engineered with knowledge of the algorithm's internal scheduling
could conceivably yield higher query counts. Transitive tournaments are the
conjectured worst case (Remark~\ref{rem:why-cycles-help}), but this reduction
 itself remains unproved.
\item The parameter range is bounded ($n\leq 800$, $k\leq 50$);
extrapolation to larger scales is uncertain.
\item The $1.25$ multiplicative gap may reflect lower-order terms
not captured by the asymptotic form, or slack in our analysis.
\end{itemize}
See Remark~\ref{rem:conjecture-difficulty} for a discussion of the
obstacles to a formal proof.

\end{document}